\documentclass[runningheads]{llncs}

\usepackage{eccv}

\usepackage{eccvabbrv}

\usepackage{graphicx}
\usepackage{booktabs}

\usepackage[accsupp]{axessibility}  

\usepackage{wrapfig}
\usepackage{multirow}

\usepackage[ruled,vlined]{algorithm2e}

\usepackage{amsmath}
\usepackage{amssymb}

\usepackage{graphicx}
\usepackage{booktabs}
\usepackage{multirow}
\usepackage{subcaption} 

\usepackage{caption}
\usepackage{subcaption}

\usepackage{arydshln}
\usepackage{dashrule}
\newcommand{\gmp}[1]{\textcolor{green}{#1}} 
\usepackage{pifont}

\newcommand{\cmark}{\ding{51}}            
\newcommand{\xmark}{\ding{55}}            
\newcommand{\tmark}{$\triangle$}

\newcommand{\Mu}{\mathrm{M}}

\usepackage[table]{xcolor}  
\usepackage{xcolor} 

\usepackage{wrapfig}
\usepackage{multirow}

\makeatletter
\@ifundefined{c@assumption}{%
  \spnewtheorem{assumption}{Assumption}{\bfseries}{\itshape}%
}{}
\makeatother

\numberwithin{proposition}{subsection}
\numberwithin{lemma}{subsection}
\numberwithin{assumption}{subsection}



\usepackage[ruled,vlined]{algorithm2e}

\usepackage{amsmath}
\usepackage{amssymb}

\usepackage{graphicx}
\usepackage{booktabs}
\usepackage{multirow}

\usepackage{caption}
\usepackage{subcaption}

\makeatletter
\let\l@title\@gobbletwo
\let\l@author\@gobbletwo 
\makeatother

\newcommand{\hlc}[2]{\begingroup\setlength{\fboxsep}{1pt}\colorbox{#1}{\strut #2}\endgroup}

\newcommand\blfootnote[1]{%
  \begingroup
  \renewcommand{\thefootnote}{}%
  \NoHyper
  \footnotetext{#1}%
  \endNoHyper
  \endgroup
}

\usepackage{hyperref}

\usepackage{orcidlink}

\begin{document}

\title{Learning Zero-Shot Subject-Driven Video Generation Using 1\% Compute}


\author{
Daneul Kim$^{1,3,\S}$\orcidlink{0000-0003-2223-8063}\hspace{0.28em} Jingxu Zhang$^{3}$\orcidlink{0009-0007-5978-5235}\hspace{0.28em} Wonjoon Jin$^{2}$\orcidlink{0000-0001-6883-3920}\hspace{0.28em} Sunghyun Cho$^2$\orcidlink{0000-0001-7627-3513}\hspace{0.28em} \\ Qi Dai$^3$\orcidlink{0000-0002-4693-2968}\hspace{0.28em} Jaesik Park$^{1,\dagger}$\orcidlink{0000-0001-5541-409X}\hspace{0.28em} Chong Luo$^{3,\dagger}$\orcidlink{0000-0003-0939-474X}
}

\authorrunning{D.~Kim et al.}

\institute{
$^1$Seoul National University, Seoul, Republic of Korea \\
$^2$POSTECH, Pohang, Republic of Korea\\
$^3$Microsoft Research Asia, Beijing, China \\
\email{\{carpedkm,jaesik.park\}@snu.ac.kr, chong.luo@microsoft.com}
}

\maketitle

\begin{abstract}
Subject-driven video generation (SDV-Gen) aims to produce videos of a specific subject by adapting a pretrained video model, enabling personalized and application-driven content creation.
To achieve this goal, per-subject tuning methods require approximately 200 A100 GPU hours to generate a customized video, whereas zero-shot methods avoid per-subject tuning but typically rely on millions of subject-video pairs for the supervision, incurring massive network fine-tuning costs (10K--200K A100 GPU hours).
We propose a \emph{data- and compute-efficient} zero-shot SDV-Gen framework that avoids test-time per-subject tuning and the use of large-scale subject-video pairs.
Our key idea decomposes SDV-Gen into (i) identity injection learned from subject-image pairs and (ii) motion-awareness preservation maintained by a small set of arbitrary videos.
We optimize the two tasks with stochastic switching, using random reference-frame sampling and image-token dropout to prevent trivial first-frame copying. 
Our gradient analysis shows that the two objectives rapidly evolve toward nearly orthogonal update subspaces, explaining the stable optimization.
Using CogVideoX-5B, we adapt a single model with 200K subject-image pairs and 4,000 arbitrary videos in 288 A100 GPU hours.
This yields about 1\% of compute compared to prior zero-shot baselines (\ie, 0.4\% of VACE and 2.8\% of Phantom) while using no subject-video pairs, yet remaining competitive in subject fidelity and motion quality.
We show that the same recipe transfers to Wan 2.1-1.3B and Wan 2.2-5B.
\end{abstract}

\blfootnote{\hspace{-2em}$^\S$Work done during internship at Microsoft Research Asia.}
\blfootnote{\hspace{-2em}$^\dagger$Corresponding authors.}

\begin{figure}
    \centering
  \includegraphics[width=\textwidth]{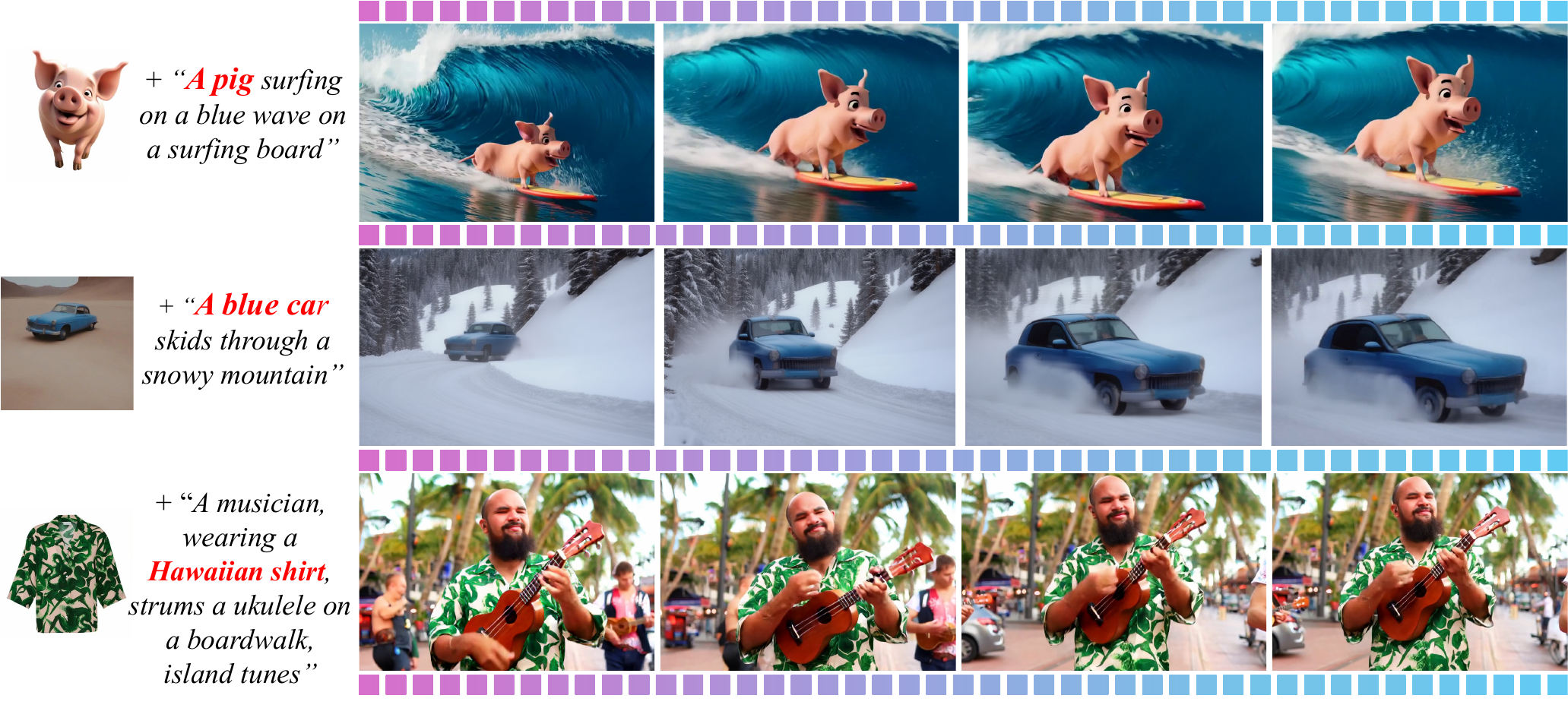}
\caption{\textbf{Our results.} We extend video generation models to enable \emph{zero-shot} subject-driven video generation (SDV-Gen). We use small datasets and compute for the supervision. \textbf{(First row)} Results of extended CogVideoX-5B~\cite{cogvideox} supervised with 200K subject-image pairs and 4,000 arbitrary videos. \textbf{(Second row)} The same model extended with only 4,000 subject-image pairs and 4,000 videos. \textbf{(Third row)} Results of extended Wan 2.2-5B~\cite{wan2025} with 200K subject-image pairs and 4,000 arbitrary videos.}
    \label{fig:teaser}
\end{figure}

\section{Introduction}

Recent advances in video diffusion models~\cite{svd,cogvideo,cogvideox,wan2025} have expanded controllable text-to-video (T2V) synthesis and video customization, enabling conditioning on signals such as keypoints, edges, and reference images~\cite{atzmon2024multi,videocontrolnet,blipdiffusion,ominicontrol,ipadapter}.
Building on this progress, \emph{subject-driven video generation}~\cite{videobooth,conceptmaster,dreamvideo,dreamvideo2,consisid,magicmirror} (also known as subject-to-video customization or video personalization, \ie, SDV-Gen) seeks to synthesize videos that preserve a target subject's identity while allowing variations in scene, motion, and context, an ability essential for customized content creation.
In practice, this is typically achieved by adapting a pretrained T2V backbone to the target subject.

Early SDV-Gen approaches fine-tune a separate model or LoRA per subject~\cite{dreambooth,dreamvideo,stillmoving,customcrafter}, limiting scalability. More recent methods train a single model that generalizes to unseen subjects~\cite{videobooth,conceptmaster,consisid,vace,phantom,hu2025hunyuancustom,videoalchemist}, through training with large-scale subject-video pairs, but this implicitly assumes that both identity and motion-awareness must be learned from subject-video pairs. 
More importantly, these zero-shot methods rely on large subject-video datasets and substantial compute (Table~\ref{tab:compute}; 10K--210K A100 GPU-hours for VACE~\cite{vace} and Phantom~\cite{phantom}), raising the barrier to entry in this field. 
Reducing dependence on subject-video pairs and lowering computational cost are thus crucial for making subject-driven, customized video generation more accessible.   

A potential alternative is to learn identity from subject-driven image generation (SDI-Gen) data (\ie, subject-image pairs) while inheriting motion priors from the pretrained T2V backbone.
In principle, subject-image pairs provide strong identity cues, whereas the backbone encodes motion priors (\ie, temporal dynamics) from large-scale pretraining.
However, na\"{\i}ve fine-tuning on subject-images pairs severely degrades inter-frame temporal coherence, often causing motion to collapse into inconsistent dynamics or near-static (frozen) movement.
This suggests that identity and motion modeling interfere with each other, indicating that treating SDV-Gen as a single objective has a fundamental issue.

\begin{table*}[t]
    \centering
    \caption{\textbf{Comparison of subject-driven video generation methods.} We compare per-subject tuning methods' required inputs and per-subject tuning time, along with zero-shot methods' dataset size and required finetuning time.}
    \resizebox{\textwidth}{!}{%
        \setlength{\tabcolsep}{4pt}
        \renewcommand{\arraystretch}{1.12}
        \begin{tabular}{l cccc}
            \toprule
            \textbf{\textit{Per-subject Tuning}} & \textbf{Zero-shot} & \textbf{Required Inputs} & \textbf{Base Model (Param.)} & \textbf{A100 hours$^*$}\\
            \midrule
            CustomCrafter~\cite{customcrafter} & \xmark &
            200 regularizing imgs per subject &
            VideoCrafter2~\cite{chen2024videocrafter2} (1.4B) &
            200$^{\ddagger}$ \\
            Still-Moving~\cite{stillmoving} & \xmark &
            Few reference images + 40 videos &
            Lumiere~\cite{lumiere} (1.2B) &
            --
            \\
            \midrule
            \textbf{\textit{Zero-shot Methods}} & \textbf{Zero-shot}  & \textbf{Dataset Size} & \textbf{Base Model (Param.)} & \textbf{A100 hours$^*$} \\
            \midrule
            VideoBooth~\cite{videobooth} & 
            \tmark &
            48,724 subject-video pairs &
            SD-based VDM~\cite{rombach2021highresolution,lvdm} (1.08B)&
            775--1,938$^\ddagger$  \\
            VACE~\cite{vace} & \cmark &
            53M source videos &
            LTX~\cite{ltx} \& Wan~\cite{wan2025} (1.3--14B) &
            70K--210K$^{\ddagger}$ \\
            Phantom~\cite{phantom} &
            \cmark & 
            1M subject-video pairs &
            Wan~\cite{wan2025} (1.3--14B) + Seed~\cite{seed} &
            10K--30K$^{\ddagger}$ \\ [2pt]
            \cdashline{1-5}
            \noalign{\vskip 3pt}
            \rowcolor{yellow!20}
            \textbf{Ours-\textit{tiny/mini}} &
            \cmark &
            \begin{tabular}[c]{@{}l@{}}
            \textbf{2,000/4,000} subject-image pairs \\
            ~~~~~+ \textbf{4,000} arbitrary videos
            \end{tabular} &
            CogVideoX~\cite{cogvideox} (5B) &
            \textbf{288} \\
            \cdashline{1-5}
            \noalign{\vskip 3pt}
            \rowcolor{yellow!20}
            \textbf{Ours} &
            \cmark &
            \begin{tabular}[c]{@{}l@{}}
            \textbf{200K} subject-image pairs \\
            + \textbf{4{,}000} arbitrary videos
            \end{tabular} &
            \begin{tabular}[c]{@{}l@{}}
            ~~~~CogVideoX~\cite{cogvideox} (5B) \\
            \& Wan~\cite{wan2025} (2.1-1.3B, 2.2-5B)
            \end{tabular} &
            \textbf{288} \\
            \bottomrule
        \end{tabular}%
    }
    \parbox{\textwidth}{\scriptsize\centering
    $^{*}$Total GPU hours. $^{\ddagger}$Estimated from reported batch sizes and wall-clock training time. Please see the supplement for details.
    }
    \label{tab:compute}
\end{table*}

We consider subject adaptation into two complementary components: identity injection from subject-image pairs and motion-awareness preservation from a small set of arbitrary videos.  
Based on this, we propose to use a stochastic task-switching schedule that alternates between these objectives, predominantly sampling subject-image pairs while using arbitrary videos only to maintain temporal coherence.  
This separation allows identity and motion to be learned from different types of data sources, avoiding the collapse observed in na\"{\i}ve SDI-Gen fine-tuning and removing the need for subject-video pairs.  

Across CogVideoX-5B~\cite{cogvideox}, Wan 2.1-1.3B and Wan 2.2-5B~\cite{wan2025}, we observe that the conflict between the two gradients from the two-objectives diminishes rapidly and becomes nearly orthogonal, leading to zero-shot SDV-Gen result as in Figure~\ref{fig:teaser}. To understand why this two-objective training remains stable under task switching, we analyze the gradients of the two objectives during fine-tuning. The gradient analysis suggests that identity and motion updates concentrate on largely disjoint parameter subspaces, reducing interference.
 
This yields a data- and compute-efficient recipe to extend a T2V backbone into a zero-shot (\ie, tuning-free at test time) SDV-Gen model without subject-video pairs.  
We fine-tune CogVideoX-5B on 200K subject-image pairs and 4,000 Pexels videos~\cite{pexels} for 4,000 steps (288 A100-hours), using \textbf{$\sim$1\%} of the compute (2.8\% of Phantom~\cite{phantom}; 0.4\% of VACE~\cite{vace}. See supplement Sec.~\textbf{F} for the detailed cost computation.) while remaining competitive.  
By showing that SDV-Gen can be learned with substantially less computation than prior baselines, we outline a practical path toward scalable video customization.  
Our approach even applies in a strong low-data regime (4,000 subject-image pairs with 4,000 arbitrary videos) and transfers to Wan 2.1-1.3B and Wan~2.2-5B.

Our contributions are summarized as follows:
\begin{itemize}
    \item We introduce a recipe to extend T2V models into a compelling zero-shot subject-driven video generation framework. We separate identity injection from motion preservation, enabling training without any subject-video pairs.
    \item We provide a gradient analysis showing that the identity and motion-aware objectives update independent parameter subspaces, explaining why the proposed formulation remains stable without subject-video pairs.
    \item We fine-tune multiple T2V backbones using only 200K subject-image pairs and 4,000 arbitrary videos, trained for 288 A100-hours, which is 1\% of compute compared to other zero-shot baselines.
\end{itemize}

\section{Related Work}
\label{sec:related}

\subsection{Subject-driven Image Generation}

Diffusion models have substantially advanced text-to-image synthesis~\cite{rombach2021highresolution,sd3,pixartalpha}, and extensive work studies how to inject novel subjects while preserving identity.
Early approaches introduce explicit spatial control (\eg, pose or edge guidance)~\cite{controlnet,t2iadapter}. 
More recent methods employ cross-attention-based visual encoders such as IP-Adapter and SSR-Encoder~\cite{ipadapter,ssrencoder}, as well as personalization techniques including DreamBooth~\cite{dreambooth}, Textual Inversion~\cite{textualinversion}, CustomDiffusion~\cite{customdiffusion}, DisenBooth~\cite{disenbooth}, and subsequent multi-subject or zero-shot variants~\cite{liu2023customizable,msdiffusion,jdeituningfree2024,Ding_2024,subjectagnostic2024,tewel2024training}.
These methods typically learn subject-specific representations from subject-image pairs, enabling subject-driven image generation (SDI-Gen).
We follow this line by using subject-image pairs (Subject-200K from OminiControl~\cite{ominicontrol}) as the primary source of identity supervision, and additionally leverage a small number of arbitrary videos to preserve temporal coherence.

\subsection{Subject-driven Video Generation}

Subject-driven video generation (SDV-Gen) extends SDI-Gen to the temporal domain. Per-subject tuning methods such as DreamVideo~\cite{dreamvideo}, MotionBooth~\cite{motionbooth}, Still-Moving~\cite{stillmoving}, and CustomCrafter~\cite{customcrafter} and DualReal~\cite{wang2025dualreal} adapt a model (or LoRA) for each identity, incurring non-trivial test-time compute per subject. In contrast, zero-shot approaches including VideoBooth~\cite{videobooth}, Concept-Master~\cite{conceptmaster}, DreamVideo-2~\cite{dreamvideo2}, MagicMirror~\cite{magicmirror}, ID-Animator~\cite{idanimator}, and Consis-ID~\cite{consisid} train a single model but often rely on subject-video pairs and may target restricted domains (\eg, faces). Large-scale DiT-based systems such as VACE~\cite{vace}, Phantom~\cite{phantom}, HunyuanCustom~\cite{hu2025hunyuancustom}, and VideoAlchemist~\cite{videoalchemist} further improve quality by training on million-scale datasets (\eg, Phantom-Data~\cite{phantomdata} and OpenS2V-Nexus~\cite{s2vnexus}), but require substantial compute budgets (Table~\ref{tab:compute}). In contrast, we study a \emph{data- and compute-efficient} alternative that avoids subject-video pairs during adaptation by decomposing identity learning and motion preservation, adapting pretrained backbones using subject-image pairs and a small set of arbitrary videos. We demonstrate efficient zero-shot SDV-Gen on modern video diffusion backbones, such as CogVideoX-5B~\cite{cogvideox}, Wan 2.1-1.3B and 2.2-5B~\cite{wan2025}.

\subsection{Multi-task Optimization}

Our approach trains a single model with two objectives: (1) identity injection on subject-image pairs and (2) motion preservation on arbitrary videos.
This relates to multi-task learning, where gradient conflicts can hurt performance. Methods such as Gradient Surgery~\cite{gradientsurgery} mitigate interference by modifying updates.
Related issues such as catastrophic forgetting~\cite{mccloskey1989catastrophic} motivate continual learning approaches that revisit past samples via memory buffers~\cite{GEM,AGEM} or replay~\cite{experience_replay,generative_replay,pseudo_rehearsal,catastrophic_forgetting}.
In contrast, since we keep a fixed, small pool of videos,
we proposed to adopt a lightweight stochastic task-switching schedule between subject-to-image and image-to-video updates.
This design is computationally more efficient than Gradient Surgery and continual learning approaches because it updates a single objective per step, avoiding the simultaneous computation of multiple task gradients, explicit gradient projections, and memory expansion.
We analyze training dynamics and show that gradient conflict diminishes rapidly.
This supports stable optimization under limited data and compute.

\section{Method}
\label{sec:method}

\subsection{Preliminaries}
\label{subsec:prelim}

\setlength{\intextsep}{-6pt}
\begin{wrapfigure}[22]{r}{0.48\columnwidth}
  \centering
  \includegraphics[width=\linewidth,
          trim={2pt 2pt 2pt 6pt}, clip]{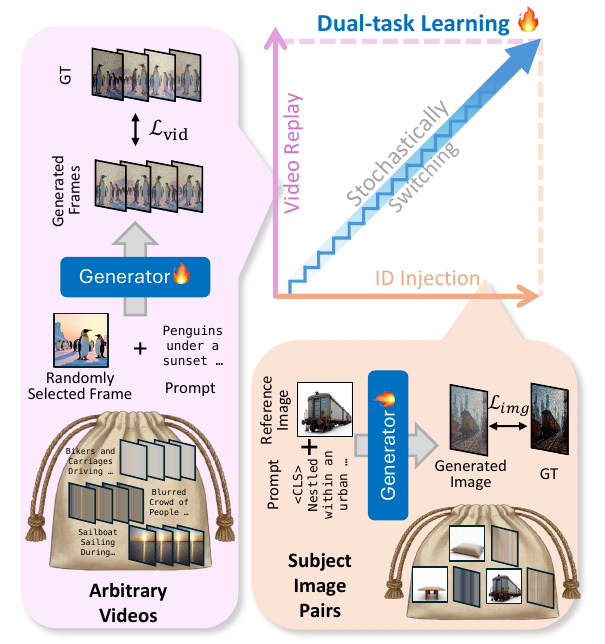}
  \caption{\textbf{Dual-task learning strategy.} We formulate SDV-Gen as a \emph{dual-task} problem: (i) identity injection (bottom) from subject-image pairs, and (ii) motion-awareness preservation (left) using arbitrary videos via stochastically switched learning.}
  \label{fig:overview_dual}
\end{wrapfigure}
We fine-tune a pretrained diffusion transformer (DiT)~\cite{dit2023} for zero-shot subject-driven video generation.
We instantiate our framework with CogVideo X-5B~\cite{cogvideox} and show that the same adaptation recipe transfers to Wan 2.1-1.3B and Wan 2.2-5B~\cite{wan2025}.

At each denoising step, the backbone processes noisy visual tokens $X \in \mathbb{R}^{N \times d}$ and text tokens $C_T \in \mathbb{R}^{M \times d}$ that share embedding dimension $d$.
Each block applies LayerNorm followed by an attention operation that updates the visual tokens $X$ under text conditioning from $C_T$, with spatial and temporal positions encoded by RoPE~\cite{rope} on the visual tokens, $X_{i,j} \leftarrow X_{i,j}\cdot R(i,j)$.
The text conditioning is realized as joint multi-modal attention over the concatenated sequence $[X; C_T]$ in CogVideoX and as cross-attention from $X$ to $C_T$ in the Wan family, and our adaptation is agnostic to this choice.

Because attention cost scales quadratically with the number of tokens, video-centric fine-tuning is substantially more expensive than image-only training.
Our goal is to adapt the backbone without any subject-video pairs by using subject-image pairs as the primary identity supervision and a small set of arbitrary videos to preserve temporal dynamics.

Let $\mathcal{D}_{\text{img}}$ denote a dataset of subject-image pairs $\{(I_{\text{ref}}, I_{\text{out}}), P\}$, where $I_{\text{ref}}$ and $I_{\text{out}}$ depict the same subject under different poses, viewpoints, or contexts, and $P$ is the text prompt.
Let $\mathcal{D}_{\text{vid}}$ denote a text-video dataset $\{(P_{\text{vid}}, V)\}$, where $P_{\text{vid}}$ is the text prompt and $V$ is the corresponding video.
We define two training objectives, $\mathcal{L}_{\text{img}}=\mathcal{L}_{\text{subject-image}}(I_{\text{ref}}, I_{\text{out}}, P)$ for identity injection and $\mathcal{L}_{\text{vid}}=\mathcal{L}_{\text{text-video}}(P_{\text{vid}}, V)$ for motion-awareness preservation.
The corresponding gradients are $g_{\text{img}}=\nabla_{\theta}\mathcal{L}_{\text{img}}$ and $g_{\text{vid}}=\nabla_{\theta}\mathcal{L}_{\text{vid}}$ with respect to the trainable parameters $\theta$.

\begin{figure*}[t]
    \centering
    \includegraphics[width=0.996\linewidth]{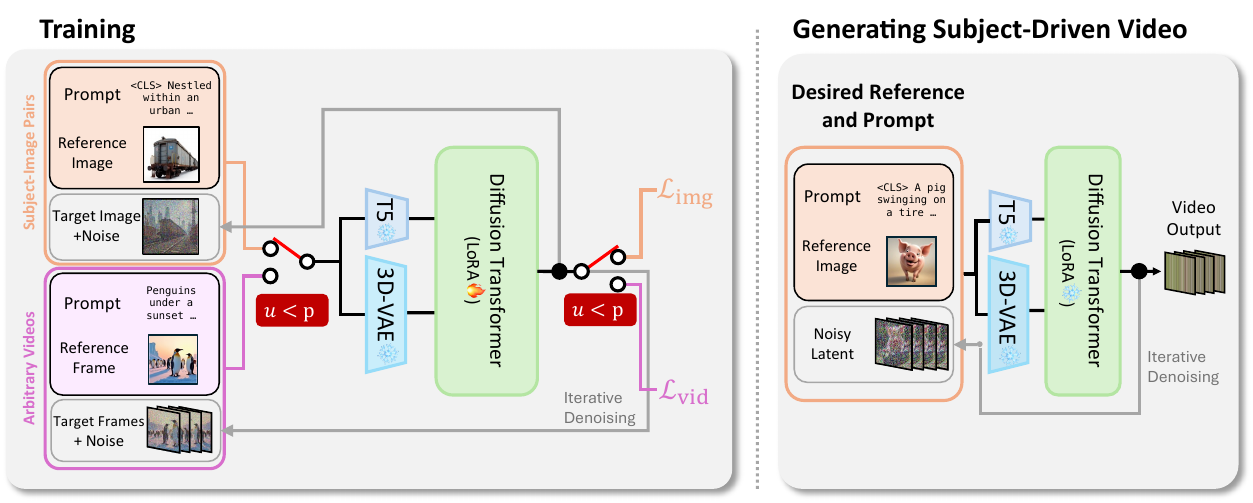}
    \caption{\textbf{Training and Inference Details.} \textbf{Left:} During training, we stochastically alternate between two objectives: identity injection using subject-image pairs and motion-awareness preservation using a small set of arbitrary videos. \textbf{Right:} At inference time, \emph{no additional per-subject tuning is required}. The model generates a video conditioned on the reference image and text prompt in a \emph{zero-shot} manner.}
    \label{fig:overview_training}
\end{figure*}

\subsection{Dual-Task Formulation}
\label{subsec:dual-task}

We formulate a dual-task problem (Figure~\ref{fig:overview_dual}): (i) identity injection from subject-image pairs to learn subject appearance, and (ii) motion-awareness preservation from arbitrary videos.  
We optimize via stochastic task switching (Figure~\ref{fig:overview_training}).  
At each iteration, we draw $u \sim \mathcal{U}(0,1)$ and select:
\begin{equation}
\label{eq:stochastic-loss}
    \mathcal{L}_{\text{total}} =
    \begin{cases}
        \mathcal{L}_{\text{vid}}, & \text{if } u < p, \\
        \mathcal{L}_{\text{img}}, & \text{otherwise}.
    \end{cases}
\end{equation}
yielding expected update $\mathbb{E}[g] = (1-p)\,g_{\text{img}} + p\,g_{\text{vid}}$.
We use $p = 0.2$, so 80\% of updates come from subject-image pairs and 20\% from arbitrary videos during fine-tuning. $p$ is chosen empirically. We demonstrate more in detail in supplementary material.

\subsection{Identity Injection with Subject-Image Pairs}
\label{subsec:identity}

Following OminiControl~\cite{ominicontrol}, given subject-image pairs $(I_{\text{ref}}, I_{\text{out}})$ of the same subject and prompt $P$, we encode $X_{\text{ref}} = \mathrm{VAE}(I_{\text{ref}})$, $X_{\text{out}} = \mathrm{VAE}(I_{\text{out}})$, $C_T = \mathrm{T5}(P)$, and train MM-DiT for $X_{\text{out}}$ conditioned on $(X_{\text{ref}}, C_T)$.
For single-frame input, we truncate RoPE embeddings to match the token count of one frame.

We apply LoRA~\cite{lora} to a subset of attention and normalization layers that interact with reference-image tokens, while keeping the rest of the backbone frozen.
To anchor identity in the conditioning stream, we prepend a special \verb|<CLS>| token to prompts (e.g., ``\texttt{An <CLS> armchair in the living room}'').
We provide ablations in the supplement showing that \verb|<CLS>| improves identity metrics (CLIP-I, DINO-I) without degrading motion quality.

\subsection{Motion Awareness with Arbitrary Videos}
\label{subsec:temporal}

Fine-tuning a pretrained video model on SDI-Gen pairs alone often erodes temporal modeling and yields near-static generations at test time.
To preserve motion, we introduce an image-to-video (I2V) objective on a small set of arbitrary videos $(P_{\text{vid}}, V)\in\mathcal{D}_{\text{vid}}$.
For each video, we sample a reference frame index $i$ and set $I_{\text{ref}} = V_i$ as conditioning, while the full video $V$ serves as the reconstruction target.
We encode
\begin{equation}
    X_{\text{ref}} = \mathrm{VAE}(I_{\text{ref}}), \quad
    X_{\text{vid}} = \mathrm{VAE}(V), \quad
    C_T^{\text{vid}} = \mathrm{T5}(P_{\text{vid}}),
\end{equation}
and train MM-DiT for $X_{\text{vid}}$ conditioned on $(X_{\text{ref}}, C_T^{\text{vid}})$.
We adopt I2V fine-tuning (instead of pure T2V) since it matches the inference-time modality (reference image$\rightarrow$video) and isolates motion preservation from identity supervision.

\noindent\textbf{Mitigating reference-frame copying.}
Conditioning on a single frame can admit a degenerate solution that simply replicates the reference appearance across time with minimal motion.
We propose to apply two simple regularizers to discourage this behavior:
(i) \emph{random reference-frame sampling}, where $i$ is sampled uniformly over frames rather than fixed to the first frame;
(ii) \emph{image-token dropout}, where we randomly drop a subset of tokens in $X_{\text{ref}}$ with probability $p_{\text{drop}}$, encouraging the model to rely on temporal priors instead of copying.
We ablate these regularizers in the supplement.

\subsection{Training Procedure}
\label{subsec:training}

Putting everything together, fine-tuning proceeds as follows.
At each training step we sample $u \sim \mathcal{U}(0,1)$ and
\begin{enumerate}
    \item if $u < p$, draw a mini-batch of arbitrary videos $(T,V)$ from $\mathcal{D}_{\text{vid}}$, sample a random reference frame and apply image-token dropping, and update $\theta$ with $\mathcal{L}_{\text{vid}}(T,V)$;
    \item otherwise, draw subject-image pairs $(I^{(1)}, I^{(2)}, P)$ from $\mathcal{D}_{\text{img}}$ and update $\theta$ with $\mathcal{L}_{\text{img}}(I^{(1)}, I^{(2)}, P)$.
\end{enumerate}
This stochastic proxy replay maintains motion-awareness through occasional video updates while injecting identity from subject-image pairs in most steps.

\begin{wrapfigure}[13]{r}{0.52\columnwidth}
  \centering
  \includegraphics[width=\linewidth]{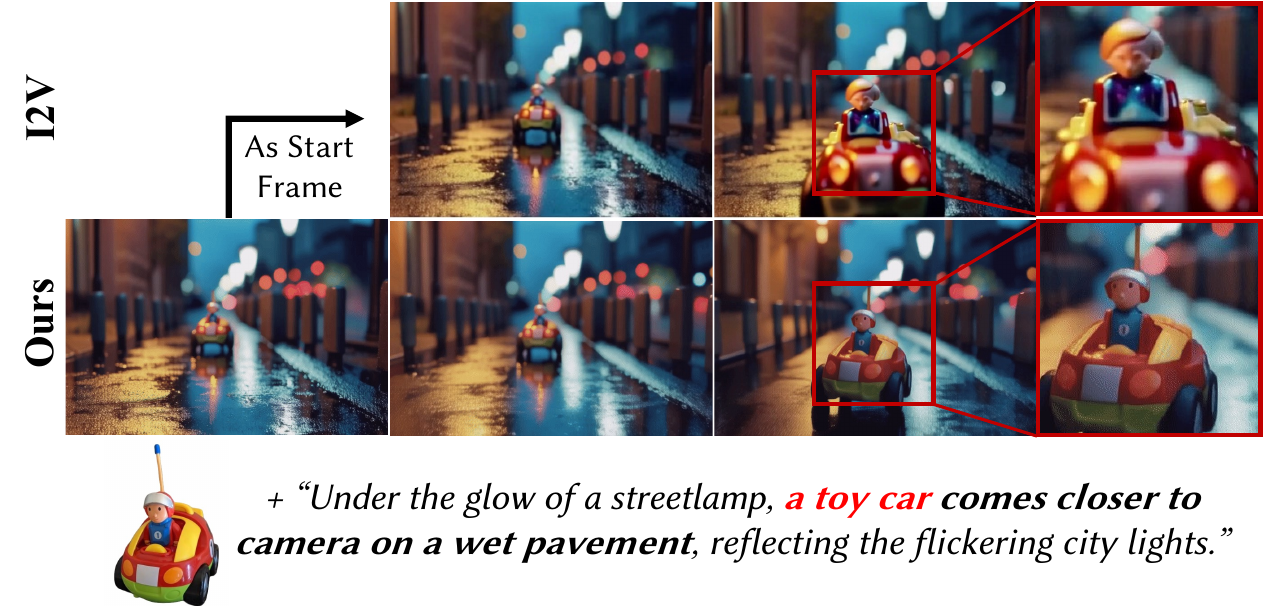}
  \caption{\textbf{Limitation of SDI-Gen$\rightarrow$I2V.} When the subject is small in the first frame, I2V fails to produce consistent results because it cannot reliably interpret low-resolution subjects.}
  \label{fig:i2v_limit}
\end{wrapfigure}

\noindent\textbf{SDI-Gen$\rightarrow$I2V.}
An alternative way to generate subject-driven video generation is to sequentially apply SDI-Gen model and I2V model. 
However, this pipeline relies on clear subject visibility in the initial frame for identity propagation, which often fails in dynamic scenes (e.g., ``a toy car comes closer to the camera on a wet pavement,'' Figure~\ref{fig:i2v_limit}).
When the subject is small or occluded, the image-to-video model may misidentify it and introduce inconsistent or fabricated details, causing a \emph{progressive loss of identity} in later frames.
We provide additional comparison with SDI-Gen$\rightarrow$I2V as baselines.
Please refer to the result in Sec.~\ref{sec:experiments}.

\noindent\textbf{Computational Efficiency.}
We quantify the computational efficiency of our stochastic task-switching scheme by analyzing the expected per-step cost.  
Let $C_{\text{img}}$ and $C_{\text{vid}}$ denote the per-step cost of SDI-Gen and I2V updates, respectively.
Under MM-DiT attention, if a video contains $T$ latent frames, the total token count grows approximately linearly with $T$, so the attention cost grows approximately as $T^2$.
For $T{=}13$, we empirically observe $C_{\text{vid}} \approx 169\,C_{\text{img}}$.
With task-switch probability $p{=}0.2$, the expected cost per step is  
\begin{align}
\mathbb{E}[C] &= (1-p)C_{\text{img}} + pC_{\text{vid}} \notag 
\approx 0.8C_{\text{img}} + 0.2 \cdot 169C_{\text{img}} \approx 34.6\,C_{\text{img}},
\end{align}
which is substantially lower than video-only fine-tuning.  
Using only 4,000 arbitrary videos, we fine-tune T2V models within 288 A100-hours while achieving competitive zero-shot SDV-Gen performance (Table~\ref{tab:compute}).

\section{Gradient Dynamics in Dual-Task Learning}
\label{sec:grad-dynamics}
In Sec.~\ref{subsec:dual-task}, we formulated SDV-Gen as a dual-task learning problem with two losses, $\mathcal{L}_{\text{img}}$ and $\mathcal{L}_{\text{vid}}$, and their corresponding gradients $g_{\text{img}}(t)$ and $g_{\text{vid}}(t)$.
To analyze gradient dynamics, we freeze the model at each checkpoint $\theta_t$ and compute $g_{\text{img}}(t)$ and $g_{\text{vid}}(t)$ on mini-batches from $\mathcal{D}_{\text{img}}$ and $\mathcal{D}_{\text{vid}}$ with respect to the trainable parameters $\theta_{\text{train}}$ (LoRA and task-specific normalization layers), matching the training setup.
We log their cosine similarity $\phi(t)$ and the gradient norms $\|g_{\text{img}}(t)\|_2$ and $\|g_{\text{vid}}(t)\|_2$.
Measurement protocols are provided in the supplement.

\subsection{Do the Gradients Become Orthogonal?}
\label{subsec:grad-results}

Figure~\ref{fig:grad_analysis} plots the evolution of $\phi(t)$ under stochastic dual-task learning.
We fine-tune pretrained CogVideoX-5B and Wan 2.2-5B using Subject-200K~\cite{ominicontrol} as subject-image supervision and a small set of arbitrary videos from Pexels~\cite{pexels}.
Starting from initialization, the cosine similarity rapidly converges to a narrow band around zero and remains there for the rest of fine-tuning regardless of the model used.
This supports the hypothesis that, after a short transient, identity injection and motion-awareness preservation update nearly orthogonal directions in the trainable parameter subspace.
(See supplement Sec.~\textbf{B.1-2} for details and comparison and Sec.~\textbf{B.3} for analysis on Wan 2.1-1.3B variant.)

To rule out the trivial explanation that $\phi(t)$ is close to zero only because both gradients vanish as training converges, we additionally track gradient norms.
As shown in Figure~\ref{fig:grad_analysis}, both $\|g_{\text{img}}(t)\|_2$ and $\|g_{\text{vid}}(t)\|_2$ remain well above the numerical noise floor throughout fine-tuning and stay within the same order of magnitude after the initial drop.
The decay of $\phi(t)$ therefore reflects a genuine change in gradient \emph{direction} rather than an artifact of vanishing magnitudes.

\begin{figure}[t]
    \centering
    \includegraphics[width=\columnwidth, trim={0.75 0 0.75 0}]{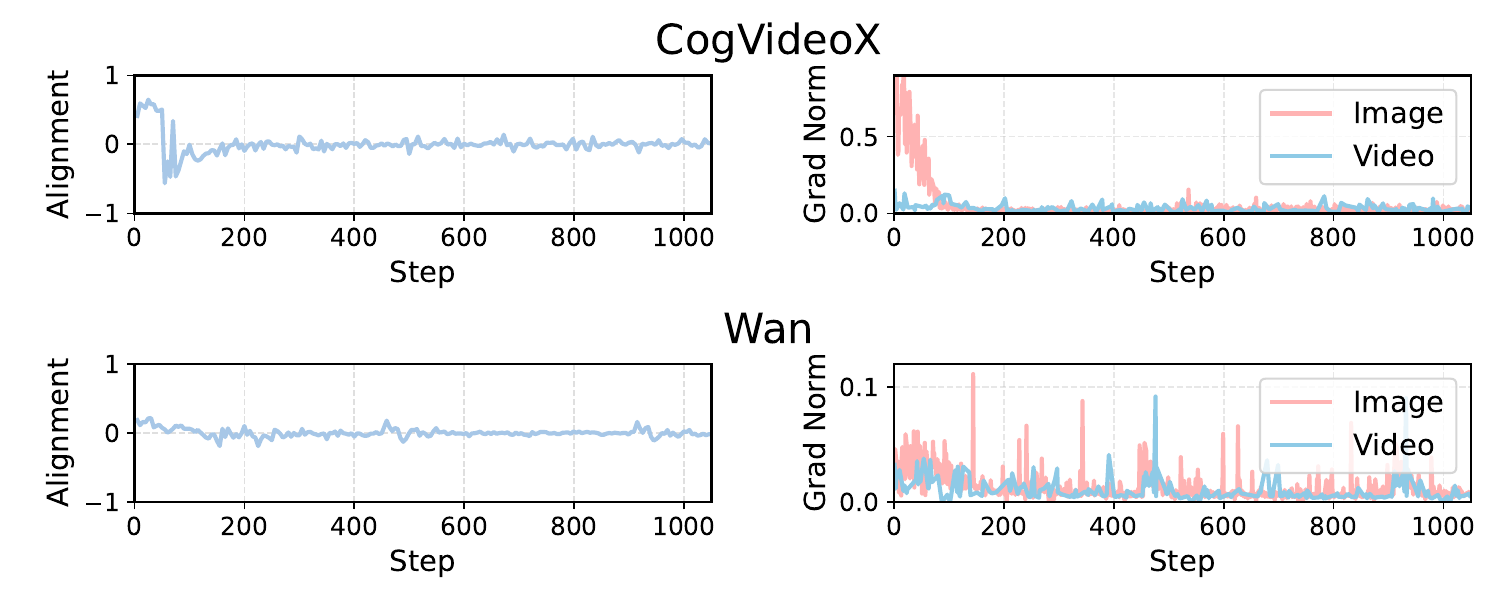}
    \caption{\textbf{Gradient analysis on alignment and norms during fine-tuning on CogVideoX-5B (Top) and Wan 2.2-5B (Bottom).} \textbf{(Left)} Cosine similarity $\phi(t)$ between $g_{\text{img}}$ and $g_{\text{vid}}$ (over trainable parameters) quickly collapses to near zero under dual-task training, indicating emergent near-orthogonality. \textbf{(Right)} $\ell_2$ norms $\|g_{\text{img}}(t)\|_2$ and $\|g_{\text{vid}}(t)\|_2$ remain non-negligible and similar in scale after 100 steps.}
    \label{fig:grad_analysis}
\end{figure}

\subsection{Why Does the Gradient Conflict Diminish?}
\label{subsec:grad-intuition}

The empirical behavior above suggests that stochastic switching between $\mathcal{L}_{\text{img}}$ and $\mathcal{L}_{\text{vid}}$ tends to reduce their gradient inner product over training. To provide intuition for this behavior, we include a simplified analysis in supplement Sec.~\textbf{C} based on a local second-order quadratic approximation. This analysis is intended only as an explanatory model, and should not be interpreted as a formal guarantee for DiT training, where the optimization landscape is highly non-convex.

\begin{proposition}[Local decay of gradient inner product]
\label{prop:orth}
Under the local second-order assumptions, let
\begin{equation}
\bar{\mathcal{L}}(\theta)=(1-p)\mathcal{L}_1(\theta)+p\mathcal{L}_2(\theta),
\end{equation}
and consider gradient descent on the mixture loss. Then there exist constants $C>0$ and $\rho\in(0,1)$ such that
\begin{equation}
\left|\left\langle \nabla \mathcal{L}_1(\theta_t),\nabla \mathcal{L}_2(\theta_t)\right\rangle\right|
\le C\rho^t.
\end{equation}
In particular,
\begin{equation}
\lim_{t\to\infty}
\left\langle \nabla \mathcal{L}_1(\theta_t),\nabla \mathcal{L}_2(\theta_t)\right\rangle = 0.
\end{equation}
\end{proposition}

\subsection{Relation to Gradient-Surgery Method}

Several works in multi-task learning explicitly manipulate gradients to resolve conflicts, such as Gradient Surgery~\cite{gradientsurgery}, which projects each task gradient to remove components that conflict with other tasks.
In our two-task setting, Gradient Surgery would replace $g_{\text{img}}$ by $g^{\text{GS}}_{\text{img}} = g_{\text{img}} - \frac{\langle g_{\text{img}}, g_{\text{vid}}\rangle}{\|g_{\text{vid}}\|_2^2} g_{\text{vid}}$ if $\langle g_{\text{img}}, g_{\text{vid}}\rangle < 0$
and analogously for $g_{\text{vid}}$.
This guarantees $\langle g^{\text{GS}}_{\text{img}}, g_{\text{vid}}\rangle \ge 0$ at each step, but requires computing \emph{both} gradients and performing additional projections at every iteration, roughly doubling the gradient-computation cost compared to sampling a single task.

Our empirical analysis shows that the simpler stochastic switching scheme drives the gradient cosine $\phi(t)$ toward zero while keeping both norms non-vanishing, achieving a similar end effect (reducing gradient conflict) without explicit projections or extra gradient evaluations.
For this reason, we use dual-task learning with stochastic switching as our default, and regard Gradient Surgery-style techniques as unnecessary refinements in this setting.
We provide qualitative and quantitative comparisons between stochastic switching and Gradient Surgery in the supplement.

\section{Experiments}
\label{sec:experiments}
\subsection{Setup}
\noindent\textbf{Implementation.}
We build our approach on CogVideoX-5B, Wan 2.1-1.3B, and Wan 2.2-5B.
CogVideoX-5B uses a 3D MM-DiT backbone~\cite{cogvideox} with a DDIM scheduler, whereas the Wan backbones use an X-DiT architecture~\cite{wan2025} trained with flow matching.
For all backbones, we optimize the two objectives with AdamW using a learning rate of $5 \times 10^{-5}$ and a cosine-with-restarts schedule.
For CogVideoX-5B, we fine-tune for 4,000 steps using BF16 mixed precision, with batch sizes of 256 for image updates and 32 for video updates.
For Wan 2.1-1.3B and Wan 2.2-5B, we fine-tune each backbone for 5,000 steps, with batch sizes of 192 for image updates and 16 for video updates.
The measured training cost is approximately 288 A100 GPU-hours for each backbone.
We apply LoRA to the designated trainable layers, using rank 128 for CogVideoX and rank 32 for Wan, with dropout 0.2.
Unless otherwise specified, we use stochastic task switching with probability $p=0.2$, chosen empirically.
We provide ablations on the LoRA rank and switching probability in the supplement with CogVideoX.

We use OminiControl’s Subject200K dataset as our subject-image pairs. To study data efficiency, we additionally evaluate 2,000 and 4,000 subsets of Subject200K.
We sample $\sim$4,000 arbitrary videos (1\% of Pexels 400K~\cite{pexels}), providing diverse real-world motion patterns.
We train two sampling-rate variants: an 8 FPS model aligned with the original CogVideoX setting, and a 16 FPS model further trained on Pexels videos to improve motion smoothness and temporal consistency at higher frame rates.
All ablations are conducted on CogVideoX, with additional implementation details deferred to the supplementary material.

\begin{table*}[t]
    \centering
    \caption{\textbf{VBench benchmark results}. We use public Phantom~\cite{phantom} and VACE~\cite{vace} results on Wan-2.1 1.3B~\cite{wan2025}.}
    \resizebox{\textwidth}{!}{%
    \setlength{\tabcolsep}{2pt}
    \begin{tabular}{l ccccc}
        \toprule
        \textbf{\textit{SDI-Gen $\rightarrow$ I2V methods}}~~~~~~~~ & ~~\textbf{Motion Smooth.}~~ & ~~\textbf{Dynamic Degree}~~ & ~~~~~\textbf{CLIP-T}~~~~~ & ~~~~~\textbf{CLIP-I}~~~~~ & ~~~~~\textbf{DINO-I}~~~ \\
        \midrule
        OminiControl~\cite{ominicontrol} & 98.40 & 47.92 & 33.06 & 72.09 & 52.61  \\
        BLIP~\cite{blipdiffusion}  & 97.81 & 48.10 & 28.83 & \cellcolor{orange!20}79.27 & 57.21 \\
        IP-Adapter~\cite{ipadapter}  & 97.74 & 51.90 & 28.94 & \cellcolor{gray!20}76.60 & 51.41 \\
        \midrule
        \textbf{\textit{SDV-Gen Methods}}~~~~~~~~  & ~~\textbf{Motion Smooth.}~~ & ~~\textbf{Dynamic Degree}~~ & ~~~~~\textbf{CLIP-T}~~~~~ & ~~~~~\textbf{CLIP-I}~~~~~ & ~~~~~\textbf{DINO-I}~~~ \\
        \midrule
        VideoBooth~\cite{videobooth} & 96.89 & 53.33 & 29.59 & 65.65 & 34.17\\
        Phantom-1.3B~\cite{phantom} & 98.70 & \cellcolor{gray!15}67.08 & \cellcolor{gray!15}33.50 & 73.00 &53.55 \\
        VACE-1.3B~\cite{vace} & \cellcolor{gray!15}98.74 & 43.75 & \cellcolor{orange!20}33.70 & 73.54 & 54.47 \\
        MAGREF-480P~\cite{deng2025magref} & \cellcolor{yellow!20}98.86 & 65.09 & 33.00 & 71.15 & 48.50\\[2pt]
     \cdashline{1-6}
        \noalign{\vskip 2pt}
        Ours (CogVideoX-5B) & 98.23 & \cellcolor{yellow!20}76.25 & \cellcolor{yellow!20}33.65 & 76.19 & \cellcolor{orange!20}61.23 \\
        Ours (Wan 2.1-1.3B) & \cellcolor{orange!20}98.97 & 53.75 & 32.51 & 74.65 & \cellcolor{yellow!20}58.77\\
        Ours (Wan 2.2-5B) & 98.09 & \cellcolor{orange!20}82.50 & 31.29 & \cellcolor{yellow!20}76.97 &\cellcolor{gray!15} 57.27 \\
        \bottomrule
        {\footnotesize Colors: \hlc{orange!20}{1st}, \hlc{yellow!20}{2nd}, \hlc{gray!15}{3rd}.}
    \end{tabular}%
    }
    \label{tab:quanti}
\end{table*}

\begin{table*}[t]
\centering
\caption{\textbf{OpenS2V-Eval v1.1 benchmark results on Single-Domain.}
Total score is the normalized weighted sum of other scores.}
\label{tab:single_domain_stv}
\resizebox{\textwidth}{!}{%
\begin{tabular}{l c c c c c c c c c}
\toprule
\textbf{Method} & \textbf{Training Cost$\downarrow$} & \textbf{Total}$\uparrow$ & \textbf{Aes}$\uparrow$ & \textbf{Smooth.}$\uparrow$ & \textbf{Amp.}$\uparrow$ & \textbf{FaceSim}$\uparrow$ & \textbf{Gme}$\uparrow$ & \textbf{Nexus}$\uparrow$ & \textbf{Natural}$\uparrow$ \\
\midrule
Vidu 2.0          & - & 52.90 & 43.32 & 91.88 & 17.52 & 36.19 & 66.96 & 44.84 & 66.11 \\
Pika 2.1          & - & 53.12 & 47.43 & 86.07 & 26.32 & 32.33 & 69.84 & 47.35 & 64.68 \\
Kling 1.6         & - & 56.67 & 45.97 & 85.76 & 47.17 & 39.27 & 65.36 & 49.30 & 73.63 \\
\midrule
VACE-P1.3B~\cite{vace}       & $\approx$70K hrs  & 49.20 & 48.93 & 95.68 & 11.91 & 18.04 & 70.78 & 36.24 & 66.85 \\
VACE-1.3B~\cite{vace}        & $\approx$70K hrs  & 51.13 & 49.41 & 95.42 & 22.51 & 22.37 & 70.87 & 38.34 & 68.33 \\
\rowcolor{orange!20}VACE-14B~\cite{vace} & $\approx$210K hrs & \textbf{61.75} & 48.94 & 93.16 & 19.69 & 64.65 & 65.86 & 50.82 & 70.56 \\
Phantom-1.3B~\cite{phantom}  & $\approx$10K hrs  & 54.50 & 49.00 & 93.70 & 16.38 & 44.03 & 69.54 & 37.72 & 66.76 \\
\rowcolor{yellow!20}Phantom-14B~\cite{phantom} & $\approx$30K hrs & \textbf{57.02} & 47.46 & 94.86 & 41.55 & 51.82 & 70.07 & 35.30 & 71.11 \\
SkyReels-A2-P14B~\cite{chen2025skyreels} & - & 55.06 & 40.85 & 85.54 & 26.41 & 54.42 & 61.81 & 48.60 & 61.85 \\
MAGREF-480P~\cite{deng2025magref} & - & 53.44 & 46.31 & 92.63 & 27.43 & 33.77 & 69.02 & 42.45 & 68.33 \\
\midrule
Ours (CogVideoX-5B)~\cite{cogvideox} & 
\textbf{288 hrs} & 50.05 & 45.40 & 93.90 & 19.38 & 18.05 & 70.53 & 41.23 & 68.52 \\
Ours (Wan 2.1-1.3B)~\cite{wan2025} & 
\textbf{288 hrs} & 51.45 &  43.80 & 93.68 & 17.55 & 13.90 & 71.12 & 45.21 & 75.65 \\
\rowcolor{gray!15}Ours (Wan 2.2-5B)~\cite{wan2025} & \textbf{288 hrs} & \textbf{56.84} & 44.70 & 84.97 & 14.38 & 48.78 & 61.98 & 49.97 & 71.30 \\
\bottomrule
{\footnotesize Colors: \hlc{orange!20}{1st}, \hlc{yellow!20}{2nd}, \hlc{gray!15}{3rd}.}
\end{tabular}%
}

\end{table*}

\noindent\textbf{Baseline.}
For zero-shot baselines, we use the SDV-Gen models of VideoBooth~\cite{videobooth}, Phantom~\cite{phantom} and VACE~\cite{vace} with Wan 2.1-1.3B, and MAGREF~\cite{deng2025magref}. Also we additionally compare with state-of-the-art SDI-Gen models OminiControl~\cite{ominicontrol}, BLIP-Diffusion~\cite{blipdiffusion}, and IP-Adapter~\cite{ipadapter}, combined with the CogVideoX-5B I2V model: each first performs subject-driven image generation with its original setup, and we then generate videos via CogVideoX-5B. 
For per-subject tuning, we additionally compare with Still-Moving~\cite{stillmoving} and CustomCrafter~\cite{customcrafter} qualitatively, using the results reported in their paper due to unavailable code or costly per-subject tuning.

\noindent\textbf{Evaluation Details.}
We randomly select 30 reference images from recent SDI-Gen benchmark sets~\cite{chosenone,stillmoving} and the standard DreamBooth dataset~\cite{dreambooth}.
For each image, we use GPT to generate 8 prompts, resulting in 240 reference image--prompt pairs for video generation.
Using the generated 240 videos, we report VBench metrics~\cite{vbench}, including \textit{Motion Smoothness} for temporal consistency, \textit{Dynamic Degree} for the amount of motion, \textit{CLIP-T} for text--video alignment, \textit{CLIP-I} for CLIP-based image similarity, and \textit{DINO-I} for DINO-based subject fidelity, and compare them with baseline methods.
For ablation studies, we use a 120-video subset, corresponding to 4 prompts per image, as it shows similar trends.
In addition, we report evaluation results on the OpenS2V benchmark version 1.1~\cite{s2vnexus}.

\begin{figure*}[t]
    \centering
    \captionsetup{type=figure}
    \includegraphics[width=0.98\textwidth]{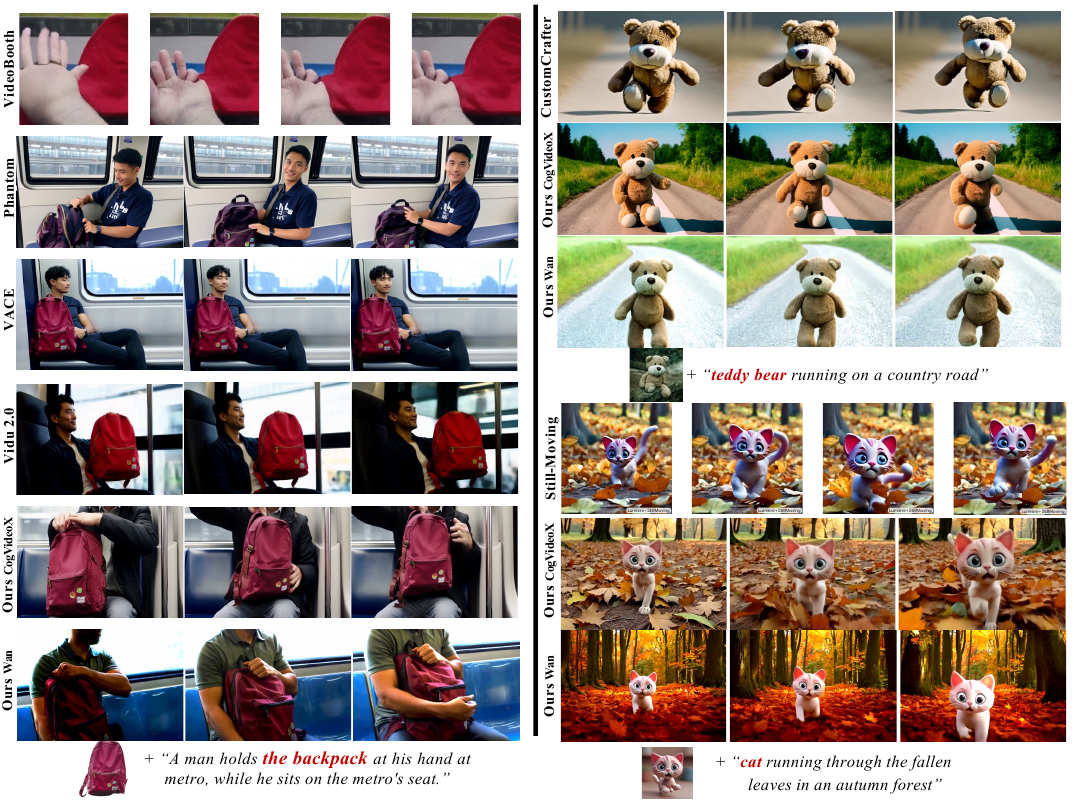}
    \caption{\textbf{Qualitative comparison with zero-shot methods (left) and per-subject tuning methods (right).} 
    Ours CogvideoX denotes our model fine-tuned on CogVideoX-5B, and Ours Wan denotes our model fine-tuned on Wan 2.2-5B. Note that ours is zero-shot, requiring no per-subject tuning at inference time.}
    \label{fig:quali}
\end{figure*} 

\subsection{Comparison with Baselines}
\label{sec:tunefree}

\noindent\textbf{Quantitative Analysis.}
Table~\ref{tab:quanti} reports VBench metrics for zero-shot SDV-Gen, where Phantom and VACE use the publicly available Wan-2.1 1.3B versions.
Our CogVideoX-5B and Wan 2.2-5B variants achieve strong motion dynamics, obtaining the second-highest and highest Dynamic Degree scores, respectively, while maintaining high Motion Smoothness.
Our Wan 2.1-1.3B variant further achieves the best Motion Smoothness.
Across CLIP-T, CLIP-I, and DINO-I, our variants remain competitive with Phantom and VACE, with Cog VideoX-5B achieving the best DINO-I score.
Compared with SDI-Gen$\rightarrow$I2V pipelines, which often obtain high CLIP-I but weaker motion dynamics, our method provides a stronger balance between subject fidelity and video motion.

Table~\ref{tab:single_domain_stv} further evaluates OpenS2V-Eval v1.1 on the Single-Domain setting.
Our Wan 2.2-5B variant achieves a competitive Total score of $56.84$ with only $288$ training hours, closely matching Phantom-14B ($57.02$) and outperforming commercial models such as Kling, Pika, and Vidu, as well as open-source models such as MAGREF and SkyReels.
Although VACE-14B obtains the highest Total score, it requires substantially larger training cost.
Our Wan 2.1-1.3B variant also remains competitive with 1.3B-scale baselines, surpassing VACE-1.3B in Total score despite using much less training compute.
The relatively lower Amplification score of our Wan 2.2-5B variant is likely related to its TI2V backbone~\cite{wan2025,choi2026improving}; in contrast, our Wan 2.1-1.3B model shows comparable motion scores to other 1.3B-family models such as VACE-1.3B and Phantom-1.3B.
We additionally provide a \emph{human preference study} in Supplement Sec.~\textbf{E.1} for all CogVideoX and Wan-family variants.

\noindent\textbf{Qualitative Analysis.}
As shown on the left of Figure~\ref{fig:quali}, our method produces sharper subject details and more consistent identity than tuning-free baselines such as Vidu~2.0 and VideoBooth, while remaining competitive with recent SDV-Gen methods (Phantom and VACE).  
In the backpack example, VideoBooth largely fails to follow the prompt, collapsing to unrelated close-up frames.  
Vidu~2.0 preserves the backpack appearance but shows less stable motion across frames.  
In contrast, our CogVideoX and Wan variants better retain fine backpack textures and maintain coherent foreground motion, consistent with the VBench gains in Table~\ref{tab:quanti}.

\setlength{\intextsep}{-12pt}
\begin{wraptable}[8]{r}{0.60\linewidth}
  \centering
    \caption{\textbf{Ablation study on training strategy.} All variants are evaluated on the same 120-sample subset using VBench.}
  \resizebox{\linewidth}{!}{%
    \begin{tabular}{l ccccc}
      \toprule
      \multirow{2}{*}{\textbf{Method}} & \textbf{Motion} & \textbf{Dynamic} & \multirow{2}{*}{\textbf{CLIP-T}} & \multirow{2}{*}{\textbf{CLIP-I}} & \multirow{2}{*}{\textbf{DINO-I}} \\
      & \textbf{Smoothness} & \textbf{Degree} & & & \\
      \midrule
      Image-only & 99.60 & 0.84  & 32.67 & 71.15 & 43.19 \\
      Two-stage  & 96.04 & 81.51 & 28.96 & 84.73 & 76.13 \\
      \midrule
      Ours & 98.45 & 69.64 & 32.69 & 77.14 & 62.88 \\
      \bottomrule
    \end{tabular}%
  }
  \label{tab:abl_train_st}
\end{wraptable}
We also compare against per-subject tuning methods CustomCrafter~\cite{customcrafter} and Still-Moving~\cite{stillmoving} using their official samples (Figure~\ref{fig:quali}, right).  
For the teddy-bear sequence, CustomCrafter exhibits noticeable shape/texture drift across frames, whereas our results are more identity-faithful.  
For the cat sequence, Still-Moving tends to lose fine facial details (e.g., whiskers) and shows weaker detail preservation, while our results better maintain subject appearance under motion.  
We further find that \emph{ours-mini}, trained with only 4{,}000 SDI-Gen pairs, remains competitive against these baselines.
We discuss this further in the ablation below.

\subsection{Ablation Study}
\noindent\textbf{Training Strategy.} 

\setlength{\intextsep}{-10pt}
\begin{wrapfigure}[15]{r}{0.60\linewidth}
    \centering
    \includegraphics[
        width=\linewidth,
        trim={2pt 2pt 2pt 6pt},
        clip
    ]{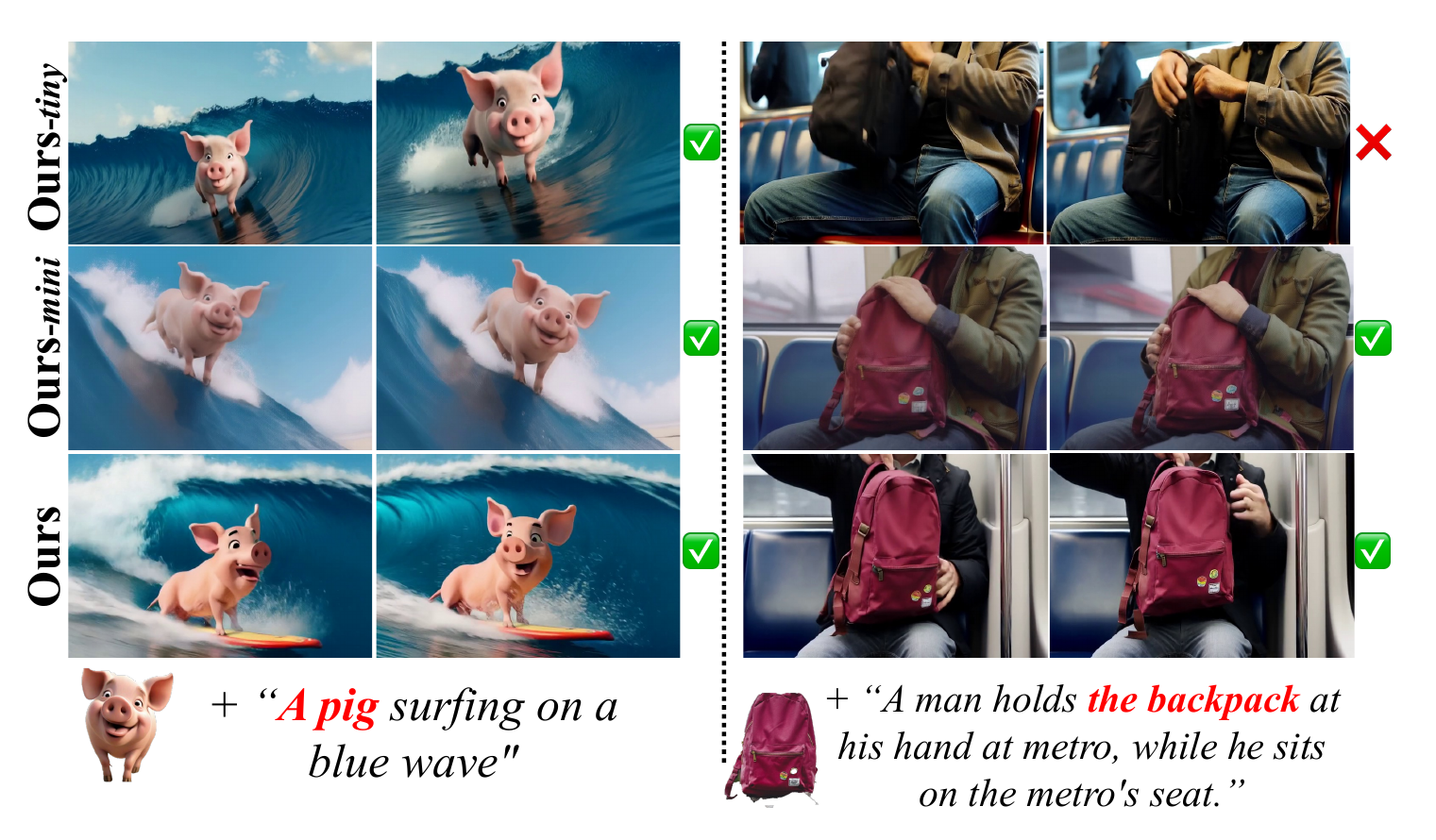}
    \caption{\textbf{Ablation on the number of subject-image pairs.}
    Using only 2,000 pairs (Ours-\textit{tiny}) leads to weaker motion dynamics and more frequent failure cases.}
    \label{fig:image_set_abl}
\end{wrapfigure}
\noindent We compare three training strategies: (1) alternating optimization (Ours), (2) image-only, and (3) sequential two-stage (\ie, image-only $\rightarrow$ image-to-video) fine-tuning approaches in the Table~\ref{tab:abl_train_st}.
Image-only fine-tuning produced high motion smoothness (99.60) because static videos (0.84 dynamic degree) lead to smooth motion, indicating a need for improved dynamics.
Two-stage fine-tuning improved dynamics (81.51) and similarity (CLIP-I: 84.73) but introduced severe artifacts and identity forgetting. (See supplement Sec.~\textbf{D.8}.)
Our \emph{dual-task learning} excelled, achieving a balance between motion smoothness (98.45) and dynamic degree (69.64).
For additional training strategies comparison with gradient surgery~\cite{gradientsurgery} or with penalty, see supplement Sec.~\textbf{B.4} and Sec.~\textbf{D.9}.

\setlength{\intextsep}{-12pt}
\begin{wraptable}[8]{r}{0.60\linewidth}
    \centering
    \caption{\textbf{Ablation study on subject-image training-pair scale.} 
    All variants are evaluated on the same 120-sample subset. 
    We report VBench metrics, Flow Entropy (FE), and copy score (CS).}
    \label{tab:abl_train_size}
    \scriptsize
    \setlength{\tabcolsep}{1.2pt}
    \renewcommand{\arraystretch}{1.05}
    \resizebox{\linewidth}{!}{%
    \begin{tabular}{l c cccccc cc}
        \toprule
        \multirow{2}{*}{\textbf{Method}}
        & \multirow{2}{*}{\textbf{\# Pairs}}
        & \textbf{Motion} & \textbf{Dynamic}
        & \multirow{2}{*}{\textbf{CLIP-T}}
        & \multirow{2}{*}{\textbf{CLIP-I}}
        & \multirow{2}{*}{\textbf{DINO-I}}
        & \multirow{2}{*}{\textbf{FE}$\uparrow$}
        & \multirow{2}{*}{\textbf{CS}$\downarrow$} \\
        &
        & \textbf{Smooth.} & \textbf{Degree}
        & & & & & \\
        \midrule
        Ours-\emph{tiny} & 2{,}000 & 98.72 & 60.71 & 33.38 & 74.72 & 57.03 & 54.73 & 45.27 \\
        Ours-\emph{mini} & 4{,}000 & 97.99 & \textbf{78.33} & \textbf{33.22} & 75.87 & 58.86 & 58.91 & 41.09 \\
        \midrule
        Ours             & 200K    & 98.45 & 69.64 & 32.69 & \textbf{77.14} & \textbf{62.88} & \textbf{63.77} & \textbf{36.32} \\
        \bottomrule
    \end{tabular}%
    }
    
\end{wraptable}
\noindent\textbf{Effect of Scaling the Image Dataset.}
While our method remains robust even with only 4,000 subject-image pairs (Ours~\emph{mini}), scaling the dataset to the full 200K pairs consistently improves fine-grained subject fidelity and temporal realism.
With fewer subject-image pairs, we occasionally observe a \emph{copy-paste} artifact: the generated subject preserves its identity, but appears overly rigid, often repeating a near-identical appearance or layout across frames.

For example, in the \emph{pig surfing} case in Figure~\ref{fig:image_set_abl}, the full-set model better captures viewpoint and pose variations, including head turns and body articulation, rather than reusing a static subject template.
This trend is also reflected in Table~\ref{tab:abl_train_size}.
As the number of subject-image pairs increases, CLIP-I and DINO-I improve, indicating stronger identity preservation.
At the same time, flow entropy (FE), which measures the diversity of subject motion, increases, while copy score (CS), which captures near-rigid repetition of the reference subject, decreases.
These results suggest that a small number of subject-image pairs is sufficient to learn the basic identity mapping, whereas broader subject-image coverage helps preserve high-frequency details and reduces copy-paste artifacts under complex motion.
See the supplement Sec.~\textbf{D.10} for implementation and measurement details on the FE and CS metrics.

\begin{figure}[t]
    \centering
    \includegraphics[width=1.0\columnwidth]{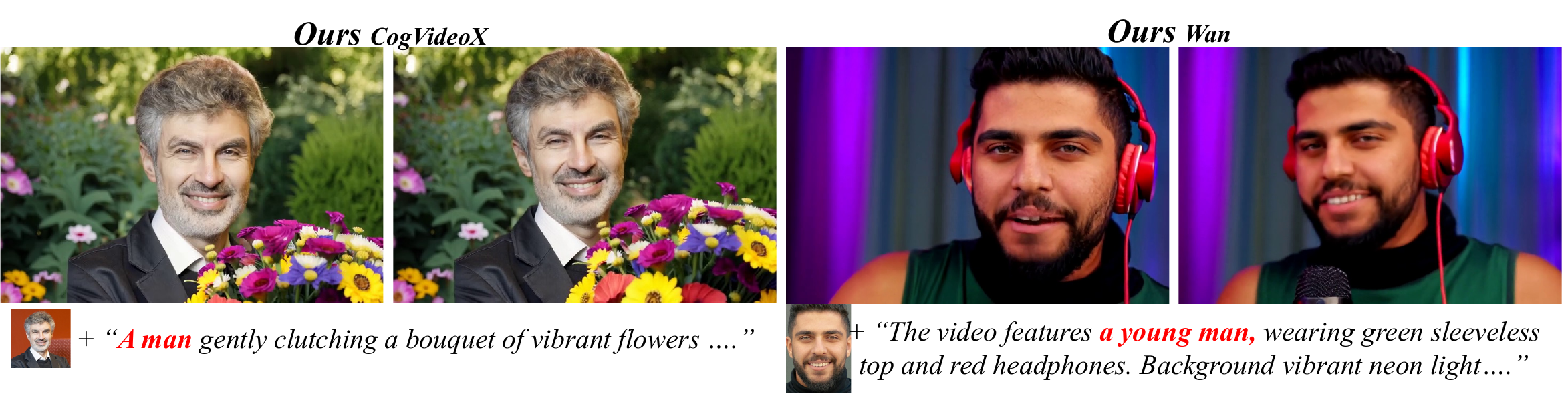}
    \caption{\textbf{Human-driven Video Generation.} \textit{Left:} CogVideoX. \textit{Right:} Wan. Without training on human-centric datasets, our model still preserves facial identity well.}
    \label{fig:humanface}
\end{figure}

\subsection{Limitations}
Our primary fine-tuning dataset, Subject-200K~\cite{ominicontrol}, is largely object-centric and includes very few human faces.
Despite this limited face coverage, our model generalizes reasonably well to facial subjects, preserving recognizable identity (Figure~\ref{fig:humanface}). Especially with Ours with Wan 2.2-5B, FaceSim is comparable to \emph{commercial} models on OpenS2V (Table~\ref{tab:single_domain_stv}).
However, failures can still occur under strong stylization/extreme appearance changes.

Regarding the Motion score in Table~\ref{tab:single_domain_stv}, we verify that our extended Cog VideoX model preserves the motion magnitude of the vanilla model, indicating that motion variation is not degraded even after our recipe is applied (details in the supplement). However, it is lower than state-of-the-art subject-driven methods, suggesting room for improvement.

\section{Conclusion}
We recast subject-driven video generation as a dual-task problem that combines identity injection from subject-image pairs with motion-awareness preservation from a small set of arbitrary videos.
Using a simple stochastic task-switching scheme, we train with 200K subject-image pairs and 4,000 arbitrary videos in 288 A100-hours, achieving competitive subject fidelity and motion quality while generalizing to unseen subjects in a zero-shot manner.
Across backbones, gradient measurements show that the image and video objectives rapidly become nearly orthogonal, enabling stable task switching without explicit gradient surgery and delivering strong identity fidelity, motion quality, and text alignment compared to both per-subject-tuned and zero-shot SDV-Gen baselines.

\paragraph{Acknowledgement.}
This work was supported by IITP grant (RS-2021-II211343: AI Grad. School Program (SNU) (5\%), RS-2019-II191906: AI Grad. School Program (POSTECH) (5\%), RS-2026-25511821: Development of Personalized Video Automatic Generation and Insertion Technology (20\%), RS-2024-00436680: Global Research Support Program in the Digital Field (40\%)) and NRF (RS-2024-00405857 (30\%)) funded by the Korea government (MSIT). We appreciate generous support from Microsoft Research Asia.

%
%
\bibliographystyle{splncs04}
\bibliography{main}

@String(CVPR  = {IEEE Conf. Comput. Vis. Pattern Recog.})

@String(ICLR  = {Int. Conf. Learn. Represent.})

@String(AAAI  = {AAAI})

@String(TOG   = {ACM Trans. Graph.})

@String(CVPR  = {CVPR})

@String(ICLR  = {ICLR})

@String(TOG   = {ACM TOG})

@inproceedings{controlnet,
  title={Adding conditional control to text-to-image diffusion models},
  author={Zhang, Lvmin and Rao, Anyi and Agrawala, Maneesh},
  booktitle={Proceedings of the IEEE/CVF international conference on computer vision},
  pages={3836--3847},
  year={2023}
}

@inproceedings{t2iadapter,
  title={T2i-adapter: Learning adapters to dig out more controllable ability for text-to-image diffusion models},
  author={Mou, Chong and Wang, Xintao and Xie, Liangbin and Wu, Yanze and Zhang, Jian and Qi, Zhongang and Shan, Ying},
  booktitle={Proceedings of the AAAI conference on artificial intelligence},
  volume={38},
  number={5},
  pages={4296--4304},
  year={2024}
}

@inproceedings{dreambooth,
  title = {DreamBooth: Fine-Tuning Text-to-Image Diffusion Models for Subject-Driven Generation},
  author = {Ruiz, Nataniel and Li, Yuanzhen and Jampani, Varun and Pritch, Yael and Rubinstein, Michael and Aberman, Kfir},
  booktitle = {CVPR},
  year = {2023}
}

@article{textualinversion,
  title={An Image is Worth One Word: Personalizing Text-to-Image Generation using Textual Inversion},
  author={Gal, Rinon and Alaluf, Yuval and Atzmon, Yuval and Patashnik, Or and Bermano, Amit H. and Chechik, Gal and Cohen-Or, Daniel},
  journal={arXiv preprint arXiv:2208.01618},
  year={2022}
}

@inproceedings{customdiffusion,
  title = {Multi-Concept Customization of Text-to-Image Diffusion},
  author = {Kumari, Nupur and Zhang, Bingliang and Zhang, Richard and Shechtman, Eli and Zhu, Jun-Yan},
  booktitle = {CVPR},
  year = {2023}
}

@inproceedings{disenbooth,
  title={Disenbooth: Identity-preserving disentangled tuning for subject-driven text-to-image generation},
  author={Chen, Hong and Zhang, Yipeng and Wu, Simin and Wang, Xin and Duan, Xuguang and Zhou, Yuwei and Zhu, Wenwu},
  booktitle={International conference on learning representations},
  volume={2024},
  pages={23104--23126},
  year={2024}
}

@inproceedings{lumiere,
  title={Lumiere: A space-time diffusion model for video generation},
  author={Bar-Tal, Omer and Chefer, Hila and Tov, Omer and Herrmann, Charles and Paiss, Roni and Zada, Shiran and Ephrat, Ariel and Hur, Junhwa and Liu, Guanghui and Raj, Amit and others},
  booktitle={SIGGRAPH Asia 2024 Conference Papers},
  pages={1--11},
  year={2024}
}

@article{stillmoving,
  title={Still-moving: Customized video generation without customized video data},
  author={Chefer, Hila and Zada, Shiran and Paiss, Roni and Ephrat, Ariel and Tov, Omer and Rubinstein, Michael and Wolf, Lior and Dekel, Tali and Michaeli, Tomer and Mosseri, Inbar},
  journal={ACM Transactions on Graphics (TOG)},
  volume={43},
  number={6},
  pages={1--11},
  year={2024},
  publisher={ACM New York, NY, USA}
}

@article{conceptmaster,
  title={Conceptmaster: Multi-concept video customization on diffusion transformer models without test-time tuning},
  author={Huang, Yuzhou and Yuan, Ziyang and Liu, Quande and Wang, Qiulin and Wang, Xintao and Zhang, Ruimao and Wan, Pengfei and Zhang, Di and Gai, Kun},
  journal={arXiv preprint arXiv:2501.04698},
  year={2025}
}

@inproceedings{videobooth,
  title={Videobooth: Diffusion-based video generation with image prompts},
  author={Jiang, Yuming and Wu, Tianxing and Yang, Shuai and Si, Chenyang and Lin, Dahua and Qiao, Yu and Loy, Chen Change and Liu, Ziwei},
  booktitle={Proceedings of the IEEE/CVF Conference on Computer Vision and Pattern Recognition},
  pages={6689--6700},
  year={2024}
}

@inproceedings{magicmirror,
  title={Magicmirror: Id-preserved video generation in video diffusion transformers},
  author={Zhang, Yuechen and Liu, Yaoyang and Xia, Bin and Peng, Bohao and Yan, Zexin and Lo, Eric and Jia, Jiaya},
  booktitle={Proceedings of the IEEE/CVF International Conference on Computer Vision},
  pages={14464--14474},
  year={2025}
}

@inproceedings{cogvideox,
  title={Cogvideox: Text-to-video diffusion models with an expert transformer},
  author={Yang, Zhuoyi and Teng, Jiayan and Zheng, Wendi and Ding, Ming and Huang, Shiyu and Xu, Jiazheng and Yang, Yuanming and Hong, Wenyi and Zhang, Xiaohan and Feng, Guanyu and others},
  booktitle={International Conference on Learning Representations},
  volume={2025},
  pages={83048--83077},
  year={2025}
}

@article{cogvideo,
  title={Cogvideo: Large-scale pretraining for text-to-video generation via transformers},
  author={Hong, Wenyi and Ding, Ming and Zheng, Wendi and Liu, Xinghan and Tang, Jie},
  journal={arXiv preprint arXiv:2205.15868},
  year={2022}
}

@article{videocontrolnet,
  title={Videocontrolnet: A motion-guided video-to-video translation framework by using diffusion model with controlnet},
  author={Hu, Zhihao and Xu, Dong},
  journal={arXiv preprint arXiv:2307.14073},
  year={2023}
}

@inproceedings{rombach2021highresolution,
  title={High-resolution image synthesis with latent diffusion models},
  author={Rombach, Robin and Blattmann, Andreas and Lorenz, Dominik and Esser, Patrick and Ommer, Bj{\"o}rn},
  booktitle={Proceedings of the IEEE/CVF conference on computer vision and pattern recognition},
  pages={10684--10695},
  year={2022}
}

@article{svd,
  title={Stable video diffusion: Scaling latent video diffusion models to large datasets},
  author={Blattmann, Andreas and Dockhorn, Tim and Kulal, Sumith and Mendelevitch, Daniel and Kilian, Maciej and Lorenz, Dominik and Levi, Yam and English, Zion and Voleti, Vikram and Letts, Adam and others},
  journal={arXiv preprint arXiv:2311.15127},
  year={2023}
}

@inproceedings{align,
  title={Align your latents: High-resolution video synthesis with latent diffusion models},
  author={Blattmann, Andreas and Rombach, Robin and Ling, Huan and Dockhorn, Tim and Kim, Seung Wook and Fidler, Sanja and Kreis, Karsten},
  booktitle={Proceedings of the IEEE/CVF conference on computer vision and pattern recognition},
  pages={22563--22575},
  year={2023}
}

@inproceedings{sd3,
  title={Scaling rectified flow transformers for high-resolution image synthesis},
  author={Esser, Patrick and Kulal, Sumith and Blattmann, Andreas and Entezari, Rahim and M{\"u}ller, Jonas and Saini, Harry and Levi, Yam and Lorenz, Dominik and Sauer, Axel and Boesel, Frederic and others},
  booktitle={Forty-first international conference on machine learning},
  year={2024}
}

@article{tewel2024training,
  title={Training-free consistent text-to-image generation},
  author={Tewel, Yoad and Kaduri, Omri and Gal, Rinon and Kasten, Yoni and Wolf, Lior and Chechik, Gal and Atzmon, Yuval},
  journal={ACM Transactions on Graphics (TOG)},
  volume={43},
  number={4},
  pages={1--18},
  year={2024},
  publisher={ACM New York, NY, USA}
}

@article{atzmon2024multi,
  title={Multi-shot character consistency for text-to-video generation},
  author={Atzmon, Yuval and Gal, Rinon and Tewel, Yoad and Kasten, Yoni and Chechik, Gal},
  year={2024}
}

@article{lora,
  title={Lora: Low-rank adaptation of large language models.},
  author={Hu, Edward J and Shen, Yelong and Wallis, Phillip and Allen-Zhu, Zeyuan and Li, Yuanzhi and Wang, Shean and Wang, Liang and Chen, Weizhu and others},
  journal={Iclr},
  volume={1},
  number={2},
  pages={3},
  year={2022}
}

@article{motionbooth,
  title={Motionbooth: Motion-aware customized text-to-video generation},
  author={Wu, Jianzong and Li, Xiangtai and Zeng, Yanhong and Zhang, Jiangning and Zhou, Qianyu and Li, Yining and Tong, Yunhai and Chen, Kai},
  journal={Advances in Neural Information Processing Systems},
  volume={37},
  pages={34322--34348},
  year={2024}
}

@inproceedings{dreamvideo,
  title={Dreamvideo: Composing your dream videos with customized subject and motion},
  author={Wei, Yujie and Zhang, Shiwei and Qing, Zhiwu and Yuan, Hangjie and Liu, Zhiheng and Liu, Yu and Zhang, Yingya and Zhou, Jingren and Shan, Hongming},
  booktitle={Proceedings of the IEEE/CVF Conference on Computer Vision and Pattern Recognition},
  pages={6537--6549},
  year={2024}
}

@inproceedings{customcrafter,
  title={Customcrafter: Customized video generation with preserving motion and concept composition abilities},
  author={Wu, Tao and Zhang, Yong and Wang, Xintao and Zhou, Xianpan and Zheng, Guangcong and Qi, Zhongang and Shan, Ying and Li, Xi},
  booktitle={Proceedings of the AAAI Conference on Artificial Intelligence},
  volume={39},
  number={8},
  pages={8469--8477},
  year={2025}
}

@article{ipadapter,
  title={Ip-adapter: Text compatible image prompt adapter for text-to-image diffusion models},
  author={Ye, Hu and Zhang, Jun and Liu, Sibo and Han, Xiao and Yang, Wei},
  journal={arXiv preprint arXiv:2308.06721},
  year={2023}
}

@article{idanimator,
  title={Id-animator: Zero-shot identity-preserving human video generation},
  author={He, Xuanhua and Liu, Quande and Qian, Shengju and Wang, Xin and Hu, Tao and Cao, Ke and Yan, Keyu and Zhang, Jie},
  journal={arXiv preprint arXiv:2404.15275},
  year={2024}
}

@inproceedings{consisid,
  title={Identity-preserving text-to-video generation by frequency decomposition},
  author={Yuan, Shenghai and Huang, Jinfa and He, Xianyi and Ge, Yunyang and Shi, Yujun and Chen, Liuhan and Luo, Jiebo and Yuan, Li},
  booktitle={Proceedings of the Computer Vision and Pattern Recognition Conference},
  pages={12978--12988},
  year={2025}
}

@article{dreamvideo2,
  title={Dreamvideo-2: Zero-shot subject-driven video customization with precise motion control},
  author={Wei, Yujie and Zhang, Shiwei and Yuan, Hangjie and Wang, Xiang and Qiu, Haonan and Zhao, Rui and Feng, Yutong and Liu, Feng and Huang, Zhizhong and Ye, Jiaxin and others},
  journal={arXiv preprint arXiv:2410.13830},
  year={2024}
}

@article{msdiffusion,
  title={Ms-diffusion: Multi-subject zero-shot image personalization with layout guidance},
  author={Wang, Xierui and Fu, Siming and Huang, Qihan and He, Wanggui and Jiang, Hao},
  journal={arXiv preprint arXiv:2406.07209},
  year={2024}
}

@article{blipdiffusion,
  title={Blip-diffusion: Pre-trained subject representation for controllable text-to-image generation and editing},
  author={Li, Dongxu and Li, Junnan and Hoi, Steven},
  journal={Advances in Neural Information Processing Systems},
  volume={36},
  pages={30146--30166},
  year={2023}
}

@inproceedings{ominicontrol,
  title={Ominicontrol: Minimal and universal control for diffusion transformer},
  author={Tan, Zhenxiong and Liu, Songhua and Yang, Xingyi and Xue, Qiaochu and Wang, Xinchao},
  booktitle={Proceedings of the IEEE/CVF International Conference on Computer Vision},
  pages={14940--14950},
  year={2025}
}

@inproceedings{ssrencoder,
  title={Ssr-encoder: Encoding selective subject representation for subject-driven generation},
  author={Zhang, Yuxuan and Song, Yiren and Liu, Jiaming and Wang, Rui and Yu, Jinpeng and Tang, Hao and Li, Huaxia and Tang, Xu and Hu, Yao and Pan, Han and others},
  booktitle={Proceedings of the IEEE/CVF Conference on Computer Vision and Pattern Recognition},
  pages={8069--8078},
  year={2024}
}

@article{t5,
  title={Exploring the limits of transfer learning with a unified text-to-text transformer},
  author={Raffel, Colin and Shazeer, Noam and Roberts, Adam and Lee, Katherine and Narang, Sharan and Matena, Michael and Zhou, Yanqi and Li, Wei and Liu, Peter J},
  journal={Journal of machine learning research},
  volume={21},
  number={140},
  pages={1--67},
  year={2020}
}

@inproceedings{vbench,
  title={Vbench: Comprehensive benchmark suite for video generative models},
  author={Huang, Ziqi and He, Yinan and Yu, Jiashuo and Zhang, Fan and Si, Chenyang and Jiang, Yuming and Zhang, Yuanhan and Wu, Tianxing and Jin, Qingyang and Chanpaisit, Nattapol and others},
  booktitle={Proceedings of the IEEE/CVF Conference on Computer Vision and Pattern Recognition},
  pages={21807--21818},
  year={2024}
}

@inproceedings{videoalchemist,
  title={Multi-subject open-set personalization in video generation},
  author={Chen, Tsai-Shien and Siarohin, Aliaksandr and Menapace, Willi and Fang, Yuwei and Lee, Kwot Sin and Skorokhodov, Ivan and Aberman, Kfir and Zhu, Jun-Yan and Yang, Ming-Hsuan and Tulyakov, Sergey},
  booktitle={Proceedings of the Computer Vision and Pattern Recognition Conference},
  pages={6099--6110},
  year={2025}
}

@misc{pexels,
  author       = {jovianzm},
  title        = {{Pexels-400k}},
  year         = {2025},
  publisher    = {Hugging Face Datasets},
  howpublished = {\url{https://huggingface.co/datasets/jovianzm/Pexels-400k}},
  type         = {dataset},
  note         = {Accessed: 2026-03-05}
}

@misc{pixartalpha,
      title={PixArt-$\alpha$: Fast Training of Diffusion Transformer for Photorealistic Text-to-Image Synthesis}, 
      author={Junsong Chen and Jincheng Yu and Chongjian Ge and Lewei Yao and Enze Xie and Yue Wu and Zhongdao Wang and James Kwok and Ping Luo and Huchuan Lu and Zhenguo Li},
      year={2023},
      eprint={2310.00426},
      archivePrefix={arXiv},
      primaryClass={cs.CV}
}

@article{rope,
  title={Roformer: Enhanced transformer with rotary position embedding},
  author={Su, Jianlin and Ahmed, Murtadha and Lu, Yu and Pan, Shengfeng and Bo, Wen and Liu, Yunfeng},
  journal={Neurocomputing},
  volume={568},
  pages={127063},
  year={2024},
  publisher={Elsevier}
}

@inproceedings{flovd,
  title={Flovd: Optical flow meets video diffusion model for enhanced camera-controlled video synthesis},
  author={Jin, Wonjoon and Dai, Qi and Luo, Chong and Baek, Seung-Hwan and Cho, Sunghyun},
  booktitle={Proceedings of the Computer Vision and Pattern Recognition Conference},
  pages={2040--2049},
  year={2025}
}

@inproceedings{raft,
  title={Raft: Recurrent all-pairs field transforms for optical flow},
  author={Teed, Zachary and Deng, Jia},
  booktitle={European conference on computer vision},
  pages={402--419},
  year={2020},
  organization={Springer}
}

@inproceedings{chosenone,
  title={The chosen one: Consistent characters in text-to-image diffusion models},
  author={Avrahami, Omri and Hertz, Amir and Vinker, Yael and Arar, Moab and Fruchter, Shlomi and Fried, Ohad and Cohen-Or, Daniel and Lischinski, Dani},
  booktitle={ACM SIGGRAPH 2024 conference papers},
  pages={1--12},
  year={2024}
}

@article{wan2025,
  title={Wan: Open and advanced large-scale video generative models},
  author={Wan, Team and Wang, Ang and Ai, Baole and Wen, Bin and Mao, Chaojie and Xie, Chen-Wei and Chen, Di and Yu, Feiwu and Zhao, Haiming and Yang, Jianxiao and others},
  journal={arXiv preprint arXiv:2503.20314},
  year={2025}
}

@inproceedings{dit2023,
  title={Scalable diffusion models with transformers},
  author={Peebles, William and Xie, Saining},
  booktitle={Proceedings of the IEEE/CVF international conference on computer vision},
  pages={4195--4205},
  year={2023}
}

@inproceedings{subjectagnostic2024,
  title={Improving subject-driven image synthesis with subject-agnostic guidance},
  author={Chan, Kelvin CK and Zhao, Yang and Jia, Xuhui and Yang, Ming-Hsuan and Wang, Huisheng},
  booktitle={Proceedings of the IEEE/CVF Conference on Computer Vision and Pattern Recognition},
  pages={6733--6742},
  year={2024}
}

@inproceedings{jdeituningfree2024,
  title={Jedi: Joint-image diffusion models for finetuning-free personalized text-to-image generation},
  author={Zeng, Yu and Patel, Vishal M and Wang, Haochen and Huang, Xun and Wang, Ting-Chun and Liu, Ming-Yu and Balaji, Yogesh},
  booktitle={Proceedings of the IEEE/CVF Conference on Computer Vision and Pattern Recognition},
  pages={6786--6795},
  year={2024}
}

@inproceedings{Ding_2024,
  title={Freecustom: Tuning-free customized image generation for multi-concept composition},
  author={Ding, Ganggui and Zhao, Canyu and Wang, Wen and Yang, Zhen and Liu, Zide and Chen, Hao and Shen, Chunhua},
  booktitle={Proceedings of the IEEE/CVF Conference on Computer Vision and Pattern Recognition},
  pages={9089--9098},
  year={2024}
}

@article{liu2023customizable,
  title={Customizable image synthesis with multiple subjects},
  author={Liu, Zhiheng and Zhang, Yifei and Shen, Yujun and Zheng, Kecheng and Zhu, Kai and Feng, Ruili and Liu, Yu and Zhao, Deli and Zhou, Jingren and Cao, Yang},
  journal={Advances in neural information processing systems},
  volume={36},
  pages={57500--57519},
  year={2023}
}

@inproceedings{phantom,
  title={Phantom: Subject-consistent video generation via cross-modal alignment},
  author={Liu, Lijie and Ma, Tianxiang and Li, Bingchuan and Chen, Zhuowei and Liu, Jiawei and Li, Gen and Zhou, Siyu and He, Qian and Wu, Xinglong},
  booktitle={Proceedings of the IEEE/CVF International Conference on Computer Vision},
  pages={14951--14961},
  year={2025}
}

@inproceedings{vace,
  title={Vace: All-in-one video creation and editing},
  author={Jiang, Zeyinzi and Han, Zhen and Mao, Chaojie and Zhang, Jingfeng and Pan, Yulin and Liu, Yu},
  booktitle={Proceedings of the IEEE/CVF International Conference on Computer Vision},
  pages={17191--17202},
  year={2025}
}

@article{hu2025hunyuancustom,
  title={Hunyuancustom: A multimodal-driven architecture for customized video generation},
  author={Hu, Teng and Yu, Zhentao and Zhou, Zhengguang and Liang, Sen and Zhou, Yuan and Lin, Qin and Lu, Qinglin},
  journal={arXiv preprint arXiv:2505.04512},
  year={2025}
}

@article{catastrophic_forgetting,
  title={Catastrophic forgetting in connectionist networks},
  author={French, Robert M},
  journal={Trends in cognitive sciences},
  volume={3},
  number={4},
  pages={128--135},
  year={1999},
  publisher={Elsevier}
}

@article{experience_replay,
  title={Human-level control through deep reinforcement learning},
  author={Mnih, Volodymyr and Kavukcuoglu, Koray and Silver, David and Rusu, Andrei A and Veness, Joel and Bellemare, Marc G and Graves, Alex and Riedmiller, Martin and Fidjeland, Andreas K and Ostrovski, Georg and others},
  journal={nature},
  volume={518},
  number={7540},
  pages={529--533},
  year={2015},
  publisher={Nature Publishing Group}
}

@article{GEM,
  title={Gradient episodic memory for continual learning},
  author={Lopez-Paz, David and Ranzato, Marc'Aurelio},
  journal={Advances in neural information processing systems},
  volume={30},
  year={2017}
}

@article{AGEM,
  title={Efficient lifelong learning with a-gem},
  author={Chaudhry, Arslan and Ranzato, Marc'Aurelio and Rohrbach, Marcus and Elhoseiny, Mohamed},
  journal={arXiv preprint arXiv:1812.00420},
  year={2018}
}

@article{generative_replay,
  title={Continual learning with deep generative replay},
  author={Shin, Hanul and Lee, Jung Kwon and Kim, Jaehong and Kim, Jiwon},
  journal={Advances in neural information processing systems},
  volume={30},
  year={2017}
}

@article{pseudo_rehearsal,
  title={Catastrophic forgetting, rehearsal and pseudorehearsal},
  author={Robins, Anthony},
  journal={Connection Science},
  volume={7},
  number={2},
  pages={123--146},
  year={1995},
  publisher={Taylor \& Francis}
}

@article{phantomdata,
  title={Phantom-data: Towards a general subject-consistent video generation dataset},
  author={Chen, Zhuowei and Li, Bingchuan and Ma, Tianxiang and Liu, Lijie and Liu, Mingcong and Zhang, Yi and Li, Gen and Li, Xinghui and Zhou, Siyu and He, Qian and others},
  journal={arXiv preprint arXiv:2506.18851},
  year={2025}
}

@article{s2vnexus,
  title={Opens2v-nexus: A detailed benchmark and million-scale dataset for subject-to-video generation},
  author={Yuan, Shenghai and He, Xianyi and Deng, Yufan and Ye, Yang and Huang, Jinfa and Lin, Bin and Luo, Jiebo and Yuan, Li},
  journal={arXiv preprint arXiv:2505.20292},
  year={2025}
}

@article{gradientsurgery,
  title={Gradient surgery for multi-task learning},
  author={Yu, Tianhe and Kumar, Saurabh and Gupta, Abhishek and Levine, Sergey and Hausman, Karol and Finn, Chelsea},
  journal={Advances in neural information processing systems},
  volume={33},
  pages={5824--5836},
  year={2020}
}

@inproceedings{chen2024videocrafter2,
  title={Videocrafter2: Overcoming data limitations for high-quality video diffusion models},
  author={Chen, Haoxin and Zhang, Yong and Cun, Xiaodong and Xia, Menghan and Wang, Xintao and Weng, Chao and Shan, Ying},
  booktitle={Proceedings of the IEEE/CVF conference on computer vision and pattern recognition},
  pages={7310--7320},
  year={2024}
}

@article{lvdm,
  title={Latent video diffusion models for high-fidelity long video generation},
  author={He, Yingqing and Yang, Tianyu and Zhang, Yong and Shan, Ying and Chen, Qifeng},
  journal={arXiv preprint arXiv:2211.13221},
  year={2022}
}

@inproceedings{seed,
  title={Seedvr: Seeding infinity in diffusion transformer towards generic video restoration},
  author={Wang, Jianyi and Lin, Zhijie and Wei, Meng and Zhao, Yang and Yang, Ceyuan and Loy, Chen Change and Jiang, Lu},
  booktitle={Proceedings of the IEEE/CVF Conference on Computer Vision and Pattern Recognition},
  pages={2161--2172},
  year={2025}
}

@article{ltx,
  title={Ltx-video: Realtime video latent diffusion},
  author={HaCohen, Yoav and Chiprut, Nisan and Brazowski, Benny and Shalem, Daniel and Moshe, Dudu and Richardson, Eitan and Levin, Eran and Shiran, Guy and Zabari, Nir and Gordon, Ori and others},
  journal={arXiv preprint arXiv:2501.00103},
  year={2024}
}

@incollection{mccloskey1989catastrophic,
  title={Catastrophic interference in connectionist networks: The sequential learning problem},
  author={McCloskey, Michael and Cohen, Neal J},
  booktitle={Psychology of learning and motivation},
  volume={24},
  pages={109--165},
  year={1989},
  publisher={Elsevier}
}

@article{chen2025skyreels,
  title={Skyreels-v2: Infinite-length film generative model},
  author={Chen, Guibin and Lin, Dixuan and Yang, Jiangping and Lin, Chunze and Zhu, Junchen and Fan, Mingyuan and Zhang, Hao and Chen, Sheng and Chen, Zheng and Ma, Chengcheng and others},
  journal={arXiv preprint arXiv:2504.13074},
  year={2025}
}

@inproceedings{choi2026improving,
  title={Improving Motion in Image-to-Video Models via Adaptive Low-Pass Guidance},
  author={Choi, June Suk and Lee, Kyungmin and Yu, Sihyun and Choi, Yisol and Shin, Jinwoo and Lee, Kimin},
  booktitle={Proceedings of the IEEE/CVF Conference on Computer Vision and Pattern Recognition},
  pages={40364--40375},
  year={2026}
}

@inproceedings{wang2025dualreal,
  title={Dualreal: Adaptive joint training for lossless identity-motion fusion in video customization},
  author={Wang, Wenchuan and Huang, Mengqi and Tu, Yijing and Mao, Zhendong},
  booktitle={Proceedings of the IEEE/CVF International Conference on Computer Vision},
  pages={16565--16575},
  year={2025}
}

@article{deng2025magref,
  title={MAGREF: Masked Guidance for Any-Reference Video Generation with Subject Disentanglement},
  author={Deng, Yufan and Yin, Yuanyang and Guo, Xun and Wang, Yizhi and Fang, Jacob Zhiyuan and Yuan, Shenghai and Yang, Yiding and Wang, Angtian and Liu, Bo and Huang, Haibin and others},
  journal={arXiv preprint arXiv:2505.23742},
  year={2025}
}

@article{fvd,
  title={Towards accurate generative models of video: A new metric \& challenges},
  author={Unterthiner, Thomas and Van Steenkiste, Sjoerd and Kurach, Karol and Marinier, Raphael and Michalski, Marcin and Gelly, Sylvain},
  journal={arXiv preprint arXiv:1812.01717},
  year={2018}
}

\clearpage
\appendix

\addtocontents{toc}{\protect\setcounter{tocdepth}{2}}

\begin{center}
{\Large\bfseries Supplementary Material}\\[0.4em]
{\large\itshape Learning Zero-Shot Subject-Driven Video Generation\\ Using 1\% Compute}
\end{center}

\vspace{1em}

{\small\bfseries Supplementary Contents\par}
\vspace{2mm}

\begingroup
\makeatletter
\def\authcount#1{}        
\let\l@title\@gobbletwo   
\let\l@author\@gobbletwo  

\setcounter{tocdepth}{-10}

\@starttoc{toc}
\makeatother
\endgroup

\vspace{3mm}

\section*{Notes}\addcontentsline{toc}{section}{Notes} We use \gmp{green} color to refer to figures, tables in the main manuscript (\eg, Figure~\gmp{2}) and use \textbf{black} color to refer to figures, tables in this supplementary material (\eg, Figure~\textbf{1}).

\appendix

\begin{algorithm}[t]
\footnotesize
\caption{\textbf{Dual-task Learning for Subject-Driven Video Generation}}
\label{alg:ours}
\SetKwInOut{Input}{Input}
\SetKwInOut{Output}{Output}
\DontPrintSemicolon
\Input{Pretrained CogVideoX-5B $f_\theta$ with LoRA; paired image set $\mathcal{D}_{\text{img}}$; proxy video set $\mathcal{D}_{\text{vid}}$; replay probability $p=0.2$; \texttt{RandomFrameSelect} and \texttt{TokenDrop}($p_{\text{drop}}=0.5$).}
\Output{Fine-tuned $f_\theta$ balancing identity fidelity and motion coherence}

\textbf{Problem Decomposition:}
\begin{itemize}
  \item \emph{Identity learning}: inject subject features from $\mathcal{D}_{\text{img}}$ with random crop/jitter and view drop
  \item \emph{motion-awareness preservation}: maintain motion priors via proxy replay from $\mathcal{D}_{\text{vid}}$
\end{itemize}

\ForEach{iteration $t \leftarrow 1,\dots,T_{\text{max}}$}{
    Sample $u \sim \mathcal{U}(0,1)$ for stochastic task selection\;
    
    \eIf{$u < p$}{
        \tcc{\textbf{Motion-awareness preservation}: video path}
        Decode batch $\{(T_i, V_i)\}$ from $\mathcal{D}_{\text{vid}}$ to $720 \times 480$, $8$ fps (max $49$ frames)\;
        $\mathbf{f}_{\text{ref}} \leftarrow \texttt{RandomFrameSelect}(V_i)$ \tcp*{mitigate frame bias}
        $\mathbf{X}_{\text{ref}} \leftarrow \texttt{TokenDrop}(\text{VAE}(\mathbf{f}_{\text{ref}}), p_{\text{drop}})$ \tcp*{drop ref tokens}
        Compute $\mathcal{L}_{\text{vid}}(T_i, V_i, \mathbf{X}_{\text{ref}})$ via v-prediction\;
        $\mathbf{g}_{\text{vid}} \leftarrow \nabla_\theta \mathcal{L}_{\text{vid}}$; $\theta \leftarrow \theta - \eta \mathbf{g}_{\text{vid}}$\;
    }{
        \tcc{\textbf{Identity learning}: image path}
        Decode paired views $\{(I^{\text{src}}_i, I^{\text{tgt}}_i)\}$ from $\mathcal{D}_{\text{img}}$ with crop jitter and optional view drop\;
        Apply LoRA-only updates on subject-token pathways\;
        Compute $\mathcal{L}_{\text{img}}(I^{\text{src}}_i, I^{\text{tgt}}_i)$ via v-prediction\;
        $\mathbf{g}_{\text{img}} \leftarrow \nabla_\theta \mathcal{L}_{\text{img}}$; $\theta \leftarrow \theta - \eta \mathbf{g}_{\text{img}}$\;
    }
    
    \tcc{Optional: monitor gradient interactions (off by default)}
    \If{$t \mod 100 = 0$}{
        $\phi_t \leftarrow \cos(\mathbf{g}_{\text{img}}, \mathbf{g}_{\text{vid}})$ \tcp*{conflict tracking / orthogonality}
    }
}
\KwRet{$f_\theta$}
\end{algorithm}

\section{Implementation Details}

\subsection{Training Protocol}
\label{sec:algo_details}
Our subject-driven video generation (SDV-Gen) framework follows Algorithm~\ref{alg:ours}. We initialize a pre-trained multi-modal diffusion transformer $f_\theta$ (e.g., CogVideo X-5B~\cite{cogvideox}), which already provides strong text--visual alignment and motion priors.

Training uses two data sources: a subject-image pairs $\mathcal{D}_{\text{img}}$, consisting of image pairs $(I^{(1)}, I^{(2)})$ of the same subject under different poses or contexts, and an arbitrary video dataset $\mathcal{D}_{\text{vid}}$, consisting of text--video pairs $(T, V)$. Our key idea is to decompose SDV-Gen into two complementary objectives, identity injection and motion-aware preservation, and optimize them by stochastic interleaving during training.

At each iteration, we sample $u \sim \mathcal{U}(0,1)$. If $u < p$, we enter the \emph{motion-aware phase} and sample a mini-batch of arbitrary videos $(T_i, V_i)$. We then optimize the video reconstruction loss $\mathcal{L}_{\text{vid}}(T_i, V_i)$ to preserve or recover realistic motion dynamics. In this phase, \texttt{RandomFrameSelect} and \texttt{TokenDrop} are applied to prevent the model from overfitting to a single reference frame.

Otherwise, if $u \ge p$, we enter the \emph{identity phase} and sample a mini-batch of paired subject images $(I^{(1)}_i, I^{(2)}_i)$ from $\mathcal{D}_{\text{img}}$. We optimize the identity injection loss $\mathcal{L}_{\text{img}}(I^{(1)}_i, I^{(2)}_i)$ while updating only the LoRA parameters associated with the subject-conditioning pathway $\mathbf{X}_{\text{in}}$. This preserves the model's original capacity for text conditioning $\mathbf{C}_T$ and output generation $\mathbf{X}_{\text{out}}$.

Unless otherwise noted, we use $p = 0.2$. This stochastic alternation balances subject fidelity and temporal consistency throughout training, avoiding the drawbacks of purely sequential fine-tuning or single-objective optimization. After $T_{\max}$ iterations, the resulting model acts as a zero-shot SDV-Gen generator without requiring a large-scale annotated SDV-Gen dataset.

\subsection{Backbone-Specific Configurations}
We validate the generality of our method on three backbones, CogVideoX-5B, Wan 2.1-1.3B and Wan 2.2-5B. The overall training recipe is shared across backbones, while several design choices are adapted to match each architecture. CogVideoX is based on MM-DiT, whereas Wan uses a cross-attention-based DiT variant.

\subsubsection{Shared Configuration}
We insert rank-$128$ LoRA adapters with scaling $\alpha = 64$ into the transformer. The adapters are applied to the attention projections (\texttt{to\_q}, \texttt{to\_k}, \texttt{to\_v}, \texttt{to\_out.0}) and selected MLP and normalization-related linear projections (e.g., \texttt{proj}, \texttt{text\_proj}, \texttt{norm1.linear}, \texttt{norm2.linear}, \texttt{ff.net.2}). All base weights remain frozen, and only the LoRA parameters are updated. Normalization parameters themselves are not directly optimized; any normalization-related adaptation is induced through the attached LoRA modules. We denote the union of all trainable parameters by $\theta_{\text{train}}$, which is also the set used for gradient logging in Sec.~\ref{sec:grad_measure}. Gradient checkpointing is enabled for memory efficiency.

\subsubsection{CogVideoX-5B}
We train CogVideoX-5B in bfloat16 with TF32 enabled, and use tiling/slicing to improve memory efficiency. Training sequences contain up to $49$ frames at $720\times480$ resolution and $8$ or $16$ FPS. In practice, videos are first decoded at a higher frame rate and then sub-sampled to $8$ or $16$ FPS for training.

We optimize the model using AdamW with learning rate $5\times10^{-5}$, $\beta_1 = 0.9$, $\beta_2 = 0.95$, cosine-with-restarts scheduling with $200$ warmup steps, and gradient clipping at $1.0$. We train for $30$ epochs with a global batch size of $32$, and save checkpoints every $50$ steps. In the identity phase, we apply per-view random crop jitter and randomly drop one reference view with probability $0.5$ to reduce pose overfitting.

\subsubsection{Wan 2.1-1.3B}
We evaluate our method on Wan 2.1-T2V-1.3B using the same overall training protocol. Since Wan 2.1 is a text-to-video backbone, we adapt our reference conditioning to the original Wan 2.1 denoising formulation without introducing an explicit image-mask interface. In particular, the reference image is injected through our subject-conditioning pathway, while the video latents follow the native Wan 2.1 generation process. We also use decoupled classifier-free guidance for reference and text conditioning, with guidance scales of $1.0$ and $4.0$, respectively. Following the standard Wan 2.1-1.3B setting, we generate videos with $81$ frames at $832\times480$ resolution and $16$ FPS, using aspect-ratio preserving padding for the reference image.

\subsubsection{Wan 2.2-5B}
We further evaluate our method on Wan 2.2-5B using the same overall training protocol. 
Compared with CogVideoX, several backbone-specific adjustments are required. 
First, because Wan 2.2-5B is a TI2V model, timestep conditioning must be factorized between the reference frame and the latent video frames through masking. 
Second, the positional encoding treatment differs across backbones. 
While CogVideoX works well with a continuous extension of RoPE~\cite{rope}, we found that directly extending RoPE in Wan does not perform well. 
Third, to better match the native capacity of Wan 2.2-5B, we use sequences of up to $81$ frames at $1280\times704$ resolution and $16$ FPS. 

Unless otherwise noted, we keep the same AdamW optimizer, learning-rate schedule, warmup and gradient clipping as in CogVideoX, for both Wan-based variants.

\section{Gradient Measurement and Analysis}
\label{sec:grad_measure}
\subsection{Measurement Protocol}

Measuring full gradients of a multi-billion-parameter model is memory-intensive.
We therefore follow multi-task learning practice and compute on the LoRA parameters actually updated during fine-tuning, denoted $\theta_{\text{train}}$.

\paragraph{Trained Parameters.}
We compute statistics over the same trainable set $\theta_{\text{train}}$ defined in Sec.~\ref{sec:algo_details} (LoRA adapters on attention and selected projection layers; base weights and normalization scale/bias remain frozen).

At a checkpoint $\theta_t$, we freeze parameters and assemble two validation mini-batches, using the same preprocessing as training unless noted:
\begin{itemize}
  \item \textbf{Subject-image pairs}: $((I^{(1)}, P), I^{(2)})$ sampled from the validation split, disjoint from training subjects. Images are decoded to $720\!\times\!480$ with the same preprocessing as training (random crop jitter enabled; no view drop during measurement). Prompts $P$ are tokenized exactly as in training.
  \item \textbf{Arbitrary videos}: $(T, V)$ sampled from the validation split of the unpaired video corpus. Videos are decoded on-the-fly, resized to $720\!\times\!480$, and truncated to at most $49$ frames rendered at $8$ or $16$ fps. A reference frame is selected uniformly at random (\texttt{RandomFrameSelect}).
\end{itemize}

\paragraph{Timestep Sampling and Objective.}
We use the same diffusion scheduler as in Sec.~\ref{sec:algo_details}. For each sample, draw a timestep $\tau \sim \text{Unif}\{0,\dots,N{-}1\}$ and add Gaussian noise according to the scheduler; the target is formed via velocity parameterization. We then compute gradients with respect to $\theta_{\text{train}}$:
\begin{align*}
g_{\text{img}}(t) \,=\, \nabla_{\theta_{\text{train}}} \, \mathcal{L}_{\text{img}}\big(I^{(1)}, I^{(2)}, P\big), \\
g_{\text{vid}}(t) \,=\, \nabla_{\theta_{\text{train}}} \, \mathcal{L}_{\text{vid}}(T, V). \qquad
\end{align*}

\paragraph{DDP Pre-sync Capture and Aggregation.}
In distributed data parallel training, we record \emph{pre-allreduce} gradients to avoid distortion from cross-rank synchronization.
For each dataset (Subject-image pairs/Arbitrary videos) and before any optimizer step:
\begin{enumerate}
  \item Register backward hooks on modules containing $\theta_{\text{train}}$ to copy per-parameter \texttt{.grad} into a preallocated flat buffer in a fixed parameter order.
  \item Capture the flat gradient on each rank before the framework performs all-reduce.
  \item Reconstruct the global gradient by reducing the sum of the flat buffers across ranks and dividing by the world size (equivalent to an average pre-sync gradient at the current step).
  \item Move the aggregated flat gradient to the CPU for logging to minimize device memory pressure. When gradient accumulation is used, we first accumulate local micro-batches, then capture the pre-sync aggregate.
\end{enumerate}

\begin{figure}[t]
    \centering
    \includegraphics[width=0.85\columnwidth]{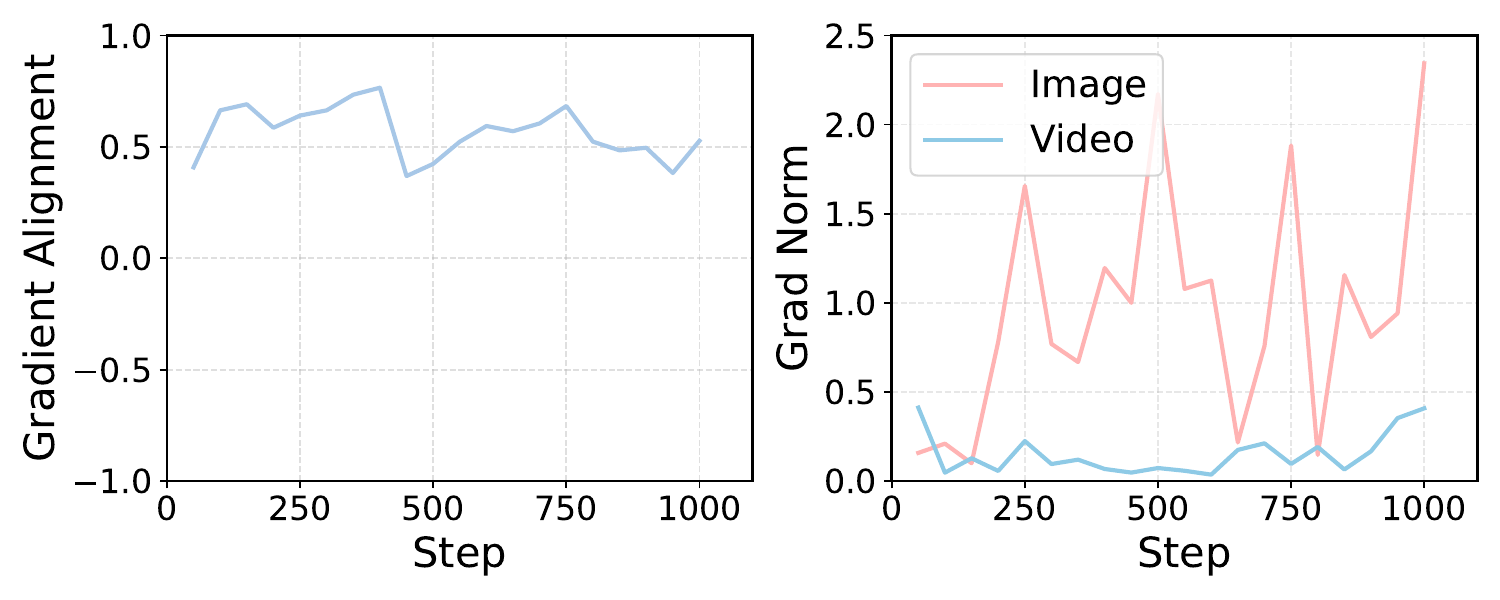}
    \caption{\textbf{Gradient alignment in CogVideoX-5B between image and video batches during image-only finetuning and gradient norms of two objectives.} Recorded gradient alignment and norm every 50 steps.}
    \label{fig:grad_align_img_only}
\end{figure}

\paragraph{Flattening and Layer-wise Grouping.}
Let $\tilde g_{\text{img}}(t)$ and $\tilde g_{\text{vid}}(t)$ be the aggregated flat vectors formed by concatenating per-parameter gradients from $\theta_{\text{train}}$ in a deterministic module/parameter order, skipping \texttt{None} entries. We also maintain indices to recover layer-wise blocks (e.g., attention vs. MLP-associated projections) for per-group statistics.

\paragraph{Cosine Similarity and Norms.}
We report the cosine similarity and L2 norms:
\begin{align*}
\phi(t) \,=\, \frac{\langle \tilde g_{\text{img}}(t), \tilde g_{\text{vid}}(t) \rangle}
{\lVert \tilde g_{\text{img}}(t) \rVert_2\, \lVert \tilde g_{\text{vid}}(t) \rVert_2 + \varepsilon},\\
\lVert \tilde g_{\text{img}}(t) \rVert_2,
\lVert \tilde g_{\text{vid}}(t) \rVert_2, \qquad \quad
\end{align*}
with a small $\varepsilon$ to avoid division by zero. When desired, we compute the same statistics per layer group by slicing the flat vectors using the stored indices.

\paragraph{Cadence and Compared Strategies.}
We repeat the above every $K$ training steps (we use $K{=}5$ to $50$ depending on the experiment.), holding the evaluation seed fixed and reusing the same validation batches over time for comparability. We compare: (a) SDI-Gen-only fine-tuning ($p{=}0$), (b) our stochastic dual-task learning ($p{=}0.2$).

\subsection{Image-only Finetuning and Comparison}
\label{sec:img_only_grad}

We now probe the gradient behaviour of a pure SDI-Gen set fine-tuning regime ($p{=}0$), where all optimizer steps are taken on $\mathcal{L}_{\text{img}}$ and the video loss $\mathcal{L}_{\text{vid}}$ is never used for training.
At checkpoints between steps $50$ and $1000$, we freeze $\theta$ and measure $g_{\text{img}}(t)$ and $g_{\text{vid}}(t)$ on fixed held-out image and video mini-batches following the protocol in Sec.~\ref{sec:grad_measure}; the resulting norms and cosine similarities are shown in Figure~\ref{fig:grad_align_img_only}.

\paragraph{Norms vs.\ the Dual-task Schedule.}
In the main paper, Figure~\gmp{5} shows that under our stochastic dual-task schedule the norms of $g_{\text{img}}(t)$ and $g_{\text{vid}}(t)$ rapidly settle to a similar and relatively small scale after the first $\sim$100 steps and remain non-negligible but stable thereafter.
In contrast, the image-only run in Figure~\ref{fig:grad_align_img_only} exhibits persistently large and highly fluctuating image gradients: $\lVert g_{\text{img}}(t)\rVert_2$ grows well beyond the range observed in Figure~\gmp{5} and does not decay to a steady plateau.
At the same time, the video gradient norm stays clearly non-zero and is substantially larger than in the dual-task case, even though $\mathcal{L}_{\text{vid}}$ is never optimized.
This indicates that image gradient updates continue to inject strong gradients into the shared LoRA subspace along directions that the video objective remains sensitive to, rather than gradually stabilizing that subspace, as in the dual-task setting.

\paragraph{Alignment vs.\ Emergent Orthogonality.} Figure~\gmp{5} in the main paper shows that, under dual-task training, the cosine similarity $\phi(t) = \cos\angle(g_{\text{img}}(t), g_{\text{vid}}(t))$ rapidly collapses into a narrow band around zero and stays there, revealing an emergent near-orthogonality between identity and temporal gradients while their norms remain non-vanishing.
By contrast, in the image-only run, the cosine in Figure~\ref{fig:grad_align_img_only} never approaches this orthogonal regime; after an initial transient, it stays strongly positive, typically in the $0.4$–$0.75$ range across checkpoints.
Pure image fine-tuning thus does not lead to the decoupling observed in Figure~\gmp{5} and instead keeps $g_{\text{img}}$ and $g_{\text{vid}}$ acting through highly overlapping directions in the LoRA parameter subspace.

\paragraph{Implications for Temporal Behavior.}

At first glance, a large positive cosine might seem desirable, since it suggests that the two gradients would cooperate if both were optimized simultaneously.
In the image-only setting, however, the optimizer follows only $-g_{\text{img}}(t)$.
The strong and persistent alignment in Figure~\ref{fig:grad_align_img_only}, combined with the large and highly variable gradient norms, implies that each identity step makes a substantial update in directions to which the video loss is also sensitive, yet no video gradient is applied to regularize or correct this drift.
Empirically, this matches the behavior of the ``Image-only'' baseline in Table~\gmp{4} of the main manuscript, which achieves high motion smoothness but almost zero dynamic degree, indicating near zero-motion (``frozen'') videos despite strong identity scores.
Comparing Figure\gmp{5} with Figure~\ref{fig:grad_align_img_only} thus suggests that our stochastic dual-task schedule is crucial for decoupling identity and temporal directions in parameter space, allowing identity injection to proceed without continuously eroding the pretrained motion prior.

\subsection{Gradient Alignment on Wan 2.1}

\begin{figure}[t]
    \centering
    \includegraphics[width=0.90\columnwidth]{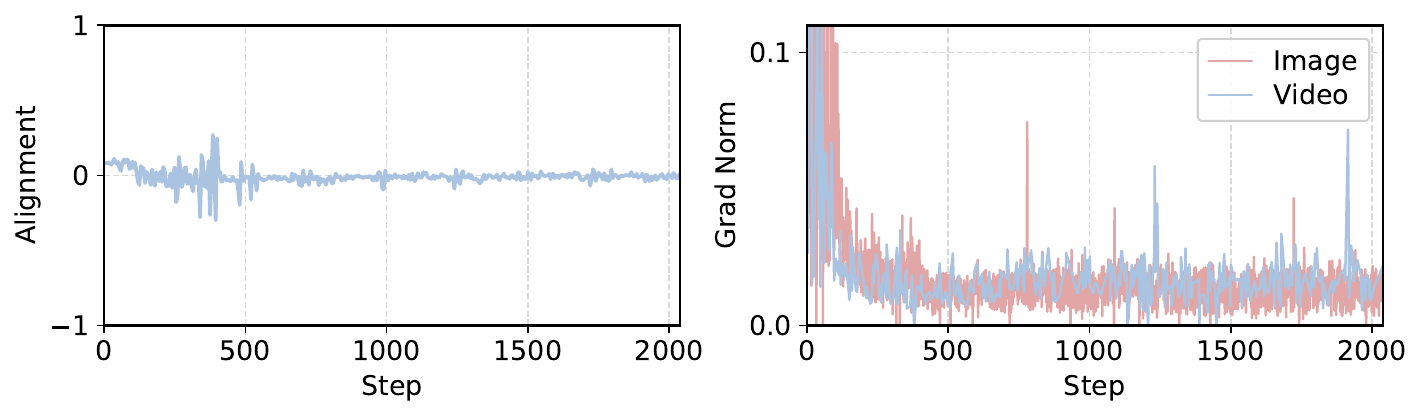}
    \caption{\textbf{Gradient alignment and gradient norms for the two training objectives using the Wan 2.1-1.3B backbone.}}
    \label{fig:wan21_grad}
\end{figure}

We further analyze the gradient alignment between the two objectives and their corresponding gradient norms on Wan 2.1-1.3B~\cite{wan2025} in Figure~\ref{fig:wan21_grad}. 
Complementing the results on CogVideoX-5B~\cite{cogvideox} and Wan 2.2-5B~\cite{wan2025} shown in Figure~\gmp{5}, Wan 2.1-1.3B exhibits the similar trend under the same training recipe. 
In particular, the two objectives become nearly orthogonal during training, suggesting that the emergent gradient orthogonalization is consistent across different video diffusion backbones.

\subsection{Gradient Surgery Details and Comparison}
\label{sec:pcgrad}

\paragraph{Gradient Surgery in Our Dual-task Setting.} We consider the two objectives used in the main paper (identity injection and motion-awareness preservation) and apply Projected Conflicting Gradient (PCGrad, Gradient Surgery) at the level of task gradients computed on the trainable set $\theta_{\text{train}}$ (Sec.~\ref{sec:algo_details}). At a given iteration we first obtain pre-synchronization gradients $g_{\text{img}}$ and $g_{\text{vid}}$ as in Sec.~\ref{sec:grad_measure}.
Gradient Surgery then removes components responsible for pairwise conflicts:
\begin{align*}
\tilde g_{\text{img}} &= g_{\text{img}} - \mathbf{1}[\langle g_{\text{img}}, g_{\text{vid}}\rangle < 0] \, \frac{\langle g_{\text{img}}, g_{\text{vid}}\rangle}{\lVert g_{\text{vid}}\rVert_2^2 + \varepsilon} \, g_{\text{vid}}, \\
\tilde g_{\text{vid}} &= g_{\text{vid}} - \mathbf{1}[\langle g_{\text{vid}}, g_{\text{img}}\rangle < 0] \, \frac{\langle g_{\text{vid}}, g_{\text{img}}\rangle}{\lVert g_{\text{img}}\rVert_2^2 + \varepsilon} \, g_{\text{img}},
\end{align*}
where $\varepsilon$ stabilizes divisions for small norms. The final update direction is $g^{\star} = \tilde (1-p)g_{\text{img}} + p\tilde g_{\text{vid}}$ with changing probability of \textit{p}.
In practice, we implement Gradient Surgery in a \emph{buffered} fashion: as training proceeds, we maintain a short FIFO buffer of recent pre-sync gradients annotated by domain and apply the projection once both domains have appeared within the buffer. This avoids doubling the per-iteration cost and lets us trade off temporal smoothing of gradients against recency by changing the buffer size.

\setlength{\intextsep}{-2pt}
\begin{wrapfigure}[20]{r}{0.55\columnwidth}
    \centering
    \includegraphics[width=0.98\linewidth]{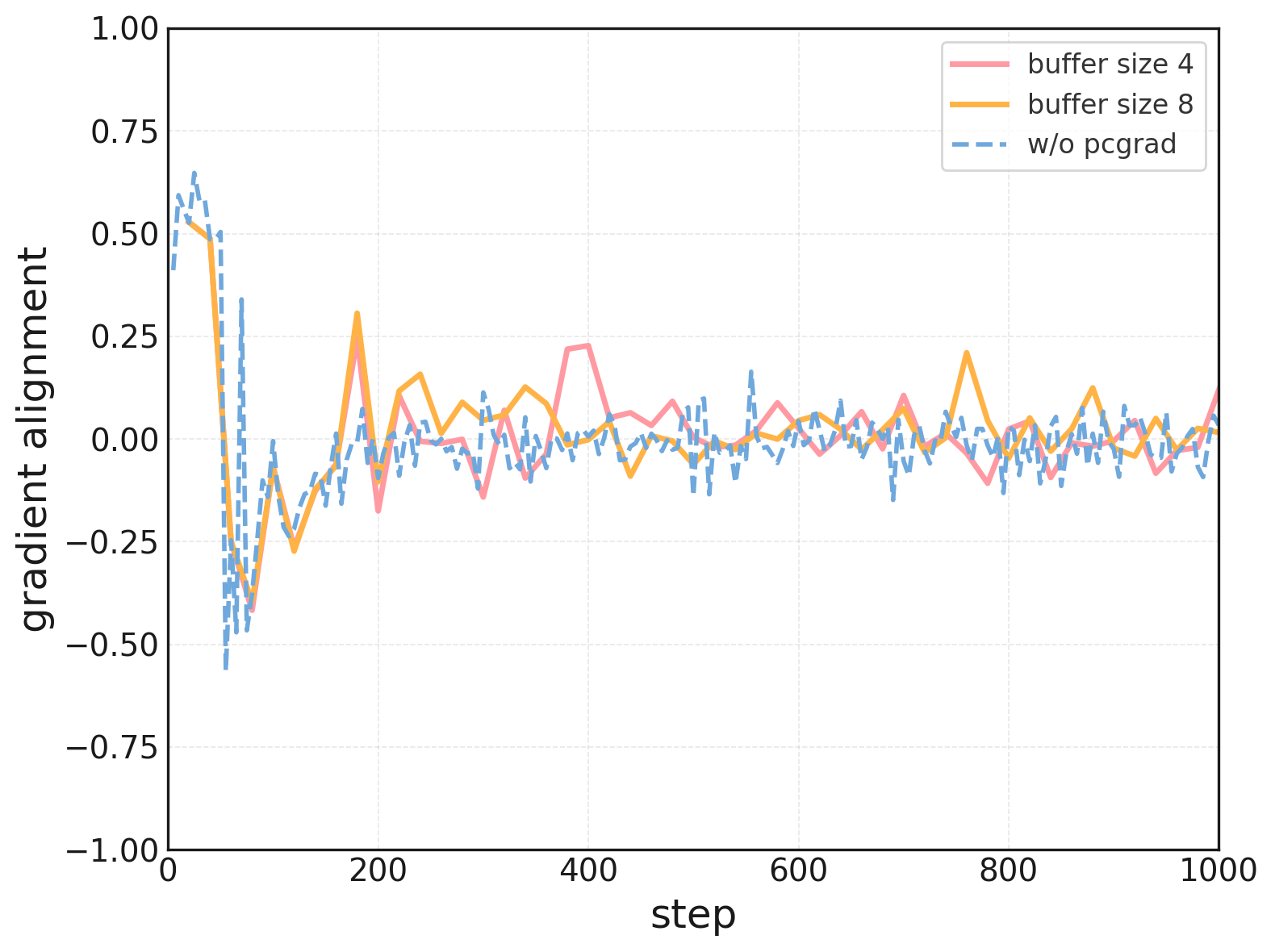}
    \caption{\textbf{Gradient alignment of the two objectives with and without Gradient Surgery.} The \textit{red} and \textit{orange} curves show the alignment obtained with Gradient Surgery using different buffer sizes, while \textit{blue dashed} curve shows the result without Gradient Surgery, corresponding to Figure~\gmp{5} in the manuscript.}
    \label{fig:pcgrad_alignment}
\end{wrapfigure}
\paragraph{Gradient Alignment under Gradient Surgery.}
Figure~\ref{fig:pcgrad_alignment} reports the cosine similarity between $g_{\text{img}}$ and $g_{\text{vid}}$ when Gradient Surgery is used during training with two buffer sizes (4 and 8). Both variants quickly push the gradient cosine into a narrow band around zero and keep it there, with only small differences between buffer sizes.
The resulting curves are very close to those obtained with our default stochastic dual-task schedule in blue dashed line in Figure~\ref{fig:pcgrad_alignment} (Equivalent to that of Figure~\gmp{5} in the main manuscript), indicating that stochastic interleaving plus proxy replay already induces near-orthogonality between image and video-batch gradients. In particular, increasing the buffer length from 4 to 8 yields only minor changes in the alignment trajectory, suggesting that most of the benefit comes from the dual-task geometry itself rather than from explicit projections over a longer history.

\begin{figure}[t]
  \centering
  \includegraphics[width=\linewidth]{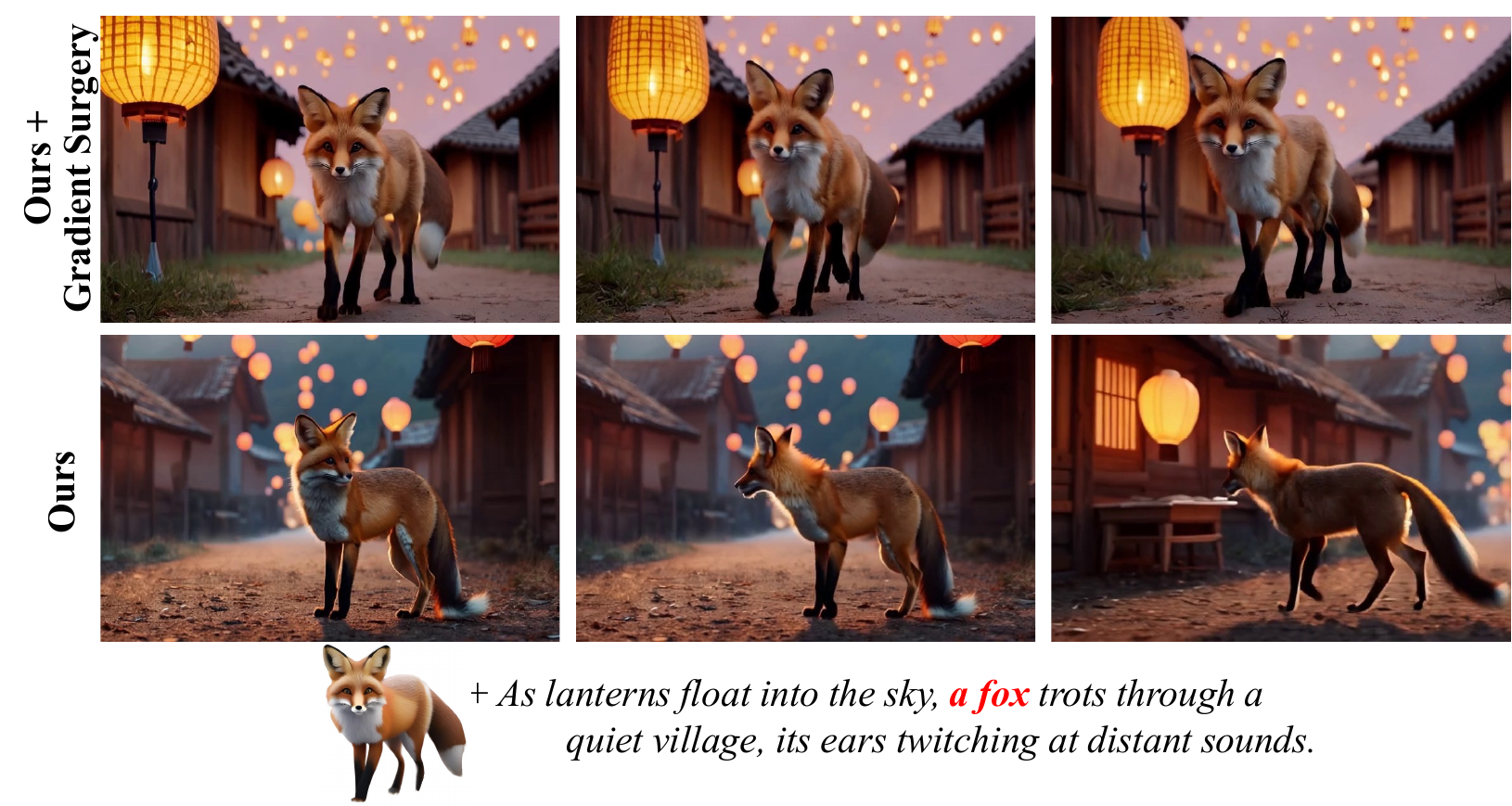}
  \caption{\textbf{Qualitative Comparison with Gradient Surgery.}}
  \label{fig:pcgrad}
\end{figure}

\paragraph{Why We Keep the Simpler Schedule.}
Although Gradient Surgery slightly accelerates the convergence in early iterations, we observe only marginal differences in downstream metrics and qualitative samples (Figure~\ref{fig:pcgrad}) compared to our default method.
This is consistent with the observation that (i) the stochastic task-switching with mixing probability $p$ already reduces gradient conflict, (ii) updates are confined to low-rank LoRA adapters where the representational overlap between tasks is limited, and (iii) the replay ratio biases most steps toward image gradient updates, so genuine conflicts are relatively rare.
Given this, the additional complexity of maintaining gradient buffers and performing projections offers limited practical gains in our setting.
We therefore adopt the stochastic dual-task schedule without Gradient Surgery as our main training recipe and regard Gradient Surgery as an optional refinement that produces similar gradient-alignment behavior, as illustrated in Figure~\ref{fig:pcgrad_alignment}.

\paragraph{Quantitative Comparison on VBench.}
We evaluate default stochastic dual-task training (without Gradient Surgery) and its Gradient Surgery-augmented counterpart on VBench. We report text alignment, subject consistency, motion smoothness, and overall quality metrics using the same evaluation protocol and sampling setup as in Sec.~\ref{sec:algo_details}. Across settings, we observe that headline VBench scores are similar between the two, with differences within experimental variance; see Table~\ref{tab:pcgrad_abl} for full results.

\begin{table}[t]
    \centering
    \caption{\textbf{Quantitative comparison on applying Gradient Surgery.}}
    \resizebox{0.95\columnwidth}{!}{%
    \begin{tabular}{l ccccc}
        \toprule
        \multirow{2}{*}{\textbf{Training Method}}  & \textbf{Motion} & \textbf{Dynamic} & \multirow{2}{*}{\textbf{CLIP-T}} & \multirow{2}{*}{\textbf{CLIP-I}} & \multirow{2}{*}{\textbf{DINO-I}} \\
        & \textbf{Smoothness} & \textbf{Degree} & & & \\
        \midrule
            w/ Gradient Surgery + buffer 8 & 98.86 & 50.83 & 31.28 & 82.17 & 74.32  \\
            w/ Gradient Surgery + buffer 4 & 97.74 & 70.83 & 33.44 & 74.52 & 56.15 \\
        \midrule
            w/o Gradient Surgery & 98.45 & 69.64 & 32.69 & 77.14 & 62.88\\
        \bottomrule
    \end{tabular}%
    }
    \label{tab:pcgrad_abl}
\end{table}

\section{Theoretical Analysis}
\label{app:theory}

In this section, we provide a stylized analysis that complements Proposition~4.2.1 in the main paper and the empirical gradient measurements in Figs.~\gmp{5} and Figs.~\ref{fig:grad_align_img_only}--\ref{fig:pcgrad_alignment}.  
We work with a local second-order model for both objectives and show that, under dual-task learning, the gradient inner product decays to zero. In contrast, under image-only fine-tuning, it typically converges to a non-zero constant.  
We then use the same framework to explain why Gradient Surgery yields only marginal gains once gradients are already nearly orthogonal.  
Throughout this section, the analysis is intended purely as intuition for the observed gradient dynamics, not as a rigorous description of large diffusion transformers.

\subsection{Setup and Notation}

Let $L_1:\mathbb{R}^n \to \mathbb{R}$ (identity loss) and $L_2:\mathbb{R}^n \to \mathbb{R}$ (temporal loss) be two objectives with gradients
\[
g_1(\theta) = \nabla_\theta L_1(\theta), \qquad g_2(\theta) = \nabla_\theta L_2(\theta).
\]

\paragraph{Gradient Conflict.}
We measure the interaction between the two tasks by the cosine of the angle between their gradients:
\begin{equation}
\label{eq:grad_conflict_app}
\phi(\theta)
= \cos \angle\bigl(g_1(\theta), g_2(\theta)\bigr)
= \frac{\langle g_1(\theta), g_2(\theta) \rangle}{\|g_1(\theta)\| \,\|g_2(\theta)\|}.
\end{equation}

\paragraph{Stochastic Task Switching.}
The training rule in the main paper alternates between updates on $L_1$ and $L_2$:
\begin{equation}
\label{eq:stochastic_update_app}
\theta_{t+1} = \theta_t - \eta g_t,
\qquad
g_t =
\begin{cases}
g_1(\theta_t), & \text{w.p. } 1-p,\\[2pt]
g_2(\theta_t), & \text{w.p. } p,
\end{cases}
\end{equation}
so that the expected update direction equals the gradient of the mixture loss
\begin{equation}
\label{eq:mixture_loss_app}
\bar L(\theta)
= (1-p)L_1(\theta) + pL_2(\theta),
\qquad
\mathbb{E}[g_t \mid \theta_t] = \nabla_\theta \bar L(\theta_t).
\end{equation}
The dual-task schedule in Sec.~4 of the main paper can therefore be viewed as stochastic gradient descent on $\bar L$, with $L_1$ and $L_2$ sampled according to the mixing probability $p$.

\subsection{Local Second-Order Mixture Model}

We focus on the behavior in a neighborhood of a point $\theta^\star$ where both tasks are approximately stationary.

\begin{assumption}[Local second-order behavior]
\label{ass:quadratic_app}
There exists a point $\theta^\star$ and symmetric matrices $H_1,H_2 \in \mathbb{R}^{n \times n}$ such that, in a neighborhood of $\theta^\star$, the losses are twice differentiable with Hessians $H_i(\theta^\star)$ and
\begin{equation}
L_i(\theta)
\approx
L_i(\theta^\star)
+
\frac{1}{2}(\theta-\theta^\star)^\top H_i (\theta-\theta^\star),
\qquad i \in \{1,2\},
\end{equation}
so that the corresponding gradients satisfy the local linearization
\begin{equation}
\label{eq:quadratic_gradients_app}
g_i(\theta)
\approx
H_i(\theta-\theta^\star),
\qquad i \in \{1,2\}.
\end{equation}
We assume $H_1$ and $H_2$ are positive semi-definite and approximately commute, i.e., they are (up to a small perturbation) simultaneously diagonalizable in this local region.
\end{assumption}

This second-order approximation is standard around critical points.  
Commutativity means that near $\theta^\star$ the two tasks share a common set of principal directions, which lets us express the dynamics in a joint eigenbasis.

\begin{assumption}[Positive curvature of the mixture along the trajectory]
\label{ass:mixture_pd_app}
Let
\begin{equation}
\label{eq:M_def_app}
M \;:=\; (1-p)H_1 + p H_2.
\end{equation}
We assume that, restricted to the subspace explored by training (the span of $\theta_0-\theta^\star$), $M$ is positive definite, i.e., all eigenvalues $\mu_k$ of $M$ on this subspace satisfy
\begin{equation}
\label{eq:mu_bounds_app}
0 < \mu_{\min} \;\le\; \mu_k \;\le\; \mu_{\max} < \infty.
\end{equation}
\end{assumption}

Note that Assumption~\ref{ass:mixture_pd_app} only requires positive curvature for the mixture $\bar L$ along the trajectory we care about, rather than global strong convexity of each task.

\subsection{Dual-Task Learning and Gradient Interaction}

Under Assumptions~\ref{ass:quadratic_app}--\ref{ass:mixture_pd_app}, we can make the statement of Proposition~4.2.1 concrete in this local model.

\begin{proposition}[Local model behind Proposition~4.2.1]
\label{prop:quadratic_alignment_app}
Assume the local second-order model (Assumption~\ref{ass:quadratic_app}) and the positive-curvature condition on the mixture (Assumption~\ref{ass:mixture_pd_app}).  
Consider deterministic gradient descent on $\bar L$ with step size $0 < \eta < 1/\mu_{\max}$,
\begin{equation}
\label{eq:gd_mixture_app}
\theta_{t+1}
=
\theta_t - \eta \nabla \bar L(\theta_t)
=
\theta_t - \eta\bigl[(1-p)g_1(\theta_t) + p g_2(\theta_t)\bigr].
\end{equation}
In this local model, the gradients evolve linearly, and the inner product
\[
A_t \;\equiv\; \bigl\langle g_1(\theta_t), g_2(\theta_t)\bigr\rangle
\]
decays exponentially:
\begin{equation}
\label{eq:At_decay_app}
|A_t|
\;\le\;
C \,(1-\eta\mu_{\min})^{2t},
\qquad
\text{for some constant } C>0,
\end{equation}
and in particular $\displaystyle \lim_{t\to\infty} A_t = 0$.
\end{proposition}

\begin{proof}
Let $z_t := \theta_t - \theta^\star$.  
Under the local linearization,
\[
g_1(\theta_t) \approx H_1 z_t, \qquad g_2(\theta_t) \approx H_2 z_t,
\]
and the gradient of the mixture loss is
\[
\nabla \bar L(\theta_t)
\approx (1-p)H_1 z_t + p H_2 z_t
= M z_t.
\]
The gradient descent update~\eqref{eq:gd_mixture_app} therefore becomes
\[
z_{t+1}
= \theta_{t+1} - \theta^\star
\approx \theta_t - \theta^\star - \eta M z_t
= (I - \eta M) z_t.
\]

By Assumption~\ref{ass:quadratic_app}, $H_1$ and $H_2$ commute and hence are simultaneously diagonalizable with $M$ (ignoring the small perturbation in this toy model).  
Let $U$ be the orthonormal matrix of shared eigenvectors and write
\[
H_1 = U \Lambda_1 U^\top,
\qquad
H_2 = U \Lambda_2 U^\top,
\qquad
M   = U \Mu U^\top,
\]
where $\Lambda_1 = \mathrm{diag}(\lambda^{(1)}_k)$, $\Lambda_2 = \mathrm{diag}(\lambda^{(2)}_k)$ and $\Mu = \mathrm{diag}(\mu_k)$.  
In these coordinates $z_t = U c_t$ for some coefficients $c_t \in \mathbb{R}^n$ and the update simplifies to
\[
c_{t+1} = (I - \eta \Mu)\, c_t,
\qquad
c_t = (I - \eta \Mu)^t c_0,
\]
so that each component evolves as
\[
c_{t,k} = (1 - \eta \mu_k)^t c_{0,k}.
\]

The gradients in this basis are
\[
g_1(\theta_t) \approx U \Lambda_1 c_t,
\qquad
g_2(\theta_t) \approx U \Lambda_2 c_t,
\]
and their inner product is
\begin{align*}
A_t
&\approx \langle g_1(\theta_t), g_2(\theta_t)\rangle
  = c_t^\top \Lambda_1 \Lambda_2 c_t \\
&= \sum_k \lambda^{(1)}_k \lambda^{(2)}_k c_{t,k}^2
  = \sum_k \lambda^{(1)}_k \lambda^{(2)}_k
    \bigl(1 - \eta \mu_k\bigr)^{2t} c_{0,k}^2.
\end{align*}
Using Assumption~\ref{ass:mixture_pd_app}, $0 < \mu_{\min} \le \mu_k \le \mu_{\max}$ on the subspace reached by training, and since $\eta < 1/\mu_{\max}$ we have $0 \le 1 - \eta \mu_k \le 1 - \eta \mu_{\min} < 1$.  
Therefore
\[
|A_t|
\le
\Bigl(\max_k |\lambda^{(1)}_k \lambda^{(2)}_k|\Bigr)
\|c_0\|^2
(1 - \eta \mu_{\min})^{2t}
=
C (1-\eta\mu_{\min})^{2t},
\]
for some constant $C>0$ depending only on the initial point and the local Hessians.  
This proves the exponential decay~\eqref{eq:At_decay_app} and hence $\lim_{t\to\infty} A_t = 0$ in this local model.
\end{proof}

\paragraph{Connection to Proposition~4.2.1 and Empirical Trends.}
In the main manuscript we measure the cosine similarity between the image and video gradients, $\phi(t) = \cos\angle\bigl(g_{\text{img}}(t), g_{\text{vid}}(t)\bigr)$, and empirically observe that it rapidly converges to a narrow band around zero while the gradient norms remain non-negligible (Figure~\gmp{5}).  
Proposition~\ref{prop:quadratic_alignment_app} shows that, in a local second-order model where the two tasks share a common eigenbasis and the mixture has positive curvature along the trajectory, gradient descent on the mixture loss naturally drives the inner product $\langle g_1,g_2\rangle$ (and hence its cosine) toward zero.  
This toy result should be read as one concrete instance consistent with Proposition~4.2.1, not as a formal guarantee for deep diffusion transformers.

\subsection{Image-Only Fine-Tuning in the Same Model}

We now show that in the same framework, pure image-only fine-tuning ($p=0$) is not expected to drive the cosine toward zero.  
This matches the behavior observed in Figure~\ref{fig:grad_align_img_only}, where the cosine remains strongly positive throughout training.

\begin{proposition}[Image-only fine-tuning]
\label{prop:img_only_app}
Assume the setting of Assumption~\ref{ass:quadratic_app} and suppose that $H_1$ is positive definite on the subspace explored by training.  
Consider gradient descent on $L_1$ alone,
\begin{equation}
\theta_{t+1} = \theta_t - \eta g_1(\theta_t),
\qquad
0 < \eta < 2/\lambda_{\max}(H_1),
\end{equation}
and define $A_t = \langle g_1(\theta_t), g_2(\theta_t)\rangle$ and $\phi(t)$ as in~\eqref{eq:grad_conflict_app}.  
Let $k^\star$ be an index of an eigenvalue of $H_1$ with maximal magnitude on the trajectory, i.e., $|\lambda^{(1)}_{k^\star}| = \max_k |\lambda^{(1)}_k|$ and $c_{0,k^\star} \neq 0$.  
Then in the local model,
\begin{equation}
\lim_{t\to\infty} \phi(t)
=
\mathrm{sign}\bigl(\lambda^{(1)}_{k^\star}\lambda^{(2)}_{k^\star}\bigr),
\end{equation}
and in particular the cosine converges to a non-zero constant whenever $\lambda^{(1)}_{k^\star}\lambda^{(2)}_{k^\star} \neq 0$.
\end{proposition}

\begin{proof}
Under the image-only update $z_{t+1} = (I-\eta H_1)z_t$ with $z_t = \theta_t - \theta^\star$, the same eigenbasis $U$ as in the previous proof diagonalizes $H_1$ and $H_2$.  
Writing $z_t = U c_t$ gives $c_{t+1,k} = (1-\eta\lambda^{(1)}_k)c_{t,k}$ and hence
\[
c_{t,k} = (1-\eta\lambda^{(1)}_k)^t c_{0,k}.
\]
Because $H_1$ is positive definite on the relevant subspace and $0<\eta<2/\lambda_{\max}(H_1)$, we have $|1-\eta\lambda^{(1)}_k|<1$ for all $k$, and the component with largest $|\lambda^{(1)}_k|$ dominates as $t\to\infty$.  
More precisely, for $k\neq k^\star$,
\[
\frac{c_{t,k}}{c_{t,k^\star}}
=
\biggl(\frac{1-\eta\lambda^{(1)}_k}{1-\eta\lambda^{(1)}_{k^\star}}\biggr)^t
\frac{c_{0,k}}{c_{0,k^\star}}
\to 0.
\]
Thus $z_t$ aligns with $v_{k^\star}$, the $k^\star$th column of $U$, and the gradients satisfy
\[
g_1(\theta_t) \approx \lambda^{(1)}_{k^\star} c_{t,k^\star} v_{k^\star},
\qquad
g_2(\theta_t) \approx \lambda^{(2)}_{k^\star} c_{t,k^\star} v_{k^\star},
\]
up to asymptotically negligible contributions from other modes.  
Therefore
\[
\lim_{t\to\infty} \phi(t)
=
\lim_{t\to\infty}
\frac{\langle g_1(\theta_t), g_2(\theta_t)\rangle}{\|g_1(\theta_t)\|\|g_2(\theta_t)\|}
=
\mathrm{sign}\bigl(\lambda^{(1)}_{k^\star}\lambda^{(2)}_{k^\star}\bigr),
\]
which is non-zero whenever $\lambda^{(1)}_{k^\star}\lambda^{(2)}_{k^\star} \neq 0$.
\end{proof}

This proposition explains why, in Figure~\ref{fig:grad_align_img_only}, the cosine between image and video gradients remains strongly positive under image-only fine-tuning.  
In the local model, the optimization path aligns with a dominant eigen-direction of $H_1$, and if $H_2$ has positive curvature along the same direction, then $g_1$ and $g_2$ stay positively aligned instead of becoming orthogonal.  
Because the optimizer follows only $-g_1$ in this regime, every identity update moves parameters along directions to which the video objective is also highly sensitive, without any counter-balancing video gradient, which is consistent with the near-static yet identity-faithful behavior of the image-only baseline in Table~\gmp{4}.

\subsection{Effect of Gradient Surgery in the Local Model}

We finally use the same framework to explain why Gradient Surgery yields gradient alignment curves similar to our default dual-task schedule (Figure~\ref{fig:pcgrad_alignment}) and only marginal gains in downstream metrics (Table~\ref{tab:pcgrad_abl}).

\paragraph{Gradient Surgery Update for Dual-tasks.}
For two tasks with gradients $g_1$ and $g_2$, Gradient Surgery projects away conflicting components when $\langle g_1,g_2\rangle<0$.  
In the two-task case, the projected gradients are
\begin{align}
\tilde g_1 &= g_1 - \mathbf{1}[\langle g_1,g_2\rangle<0]\frac{\langle g_1,g_2\rangle}{\|g_2\|^2}g_2, \\
\tilde g_2 &= g_2 - \mathbf{1}[\langle g_2,g_1\rangle<0]\frac{\langle g_2,g_1\rangle}{\|g_1\|^2}g_1.
\end{align}
Given mixing probability $p$, the Gradient Surgery update direction is $g^{\text{PC}} = (1-p)\tilde g_1 + p\tilde g_2$, whereas the dual-task schedule without projection uses the mixture gradient $g^{\text{mix}} = (1-p)g_1 + p g_2$.

\begin{lemma}[Difference between Gradient Surgery and mixture update]
\label{lem:pcgrad_diff_app}
Let $A = \langle g_1,g_2\rangle$ and suppose $A<0$ so that Gradient Surgery modifies both gradients.  
Then the difference between the Gradient Surgery and mixture updates satisfies
\begin{equation}
\label{eq:pcgrad_diff_bound}
\|g^{\text{PC}} - g^{\text{mix}}\|
\le
|A|\Bigl(\frac{1-p}{\|g_2\|} + \frac{p}{\|g_1\|}\Bigr)
\le
2|A|\max\Bigl\{\frac{1}{\|g_1\|},\frac{1}{\|g_2\|}\Bigr\},
\end{equation}
and in particular
\begin{equation}
\label{eq:pcgrad_cos_bound}
\frac{\|g^{\text{PC}} - g^{\text{mix}}\|}{\|g^{\text{mix}}\|}
\le
C_\text{rel}\,|\cos\angle(g_1,g_2)|,
\end{equation}
for some constant $C_\text{rel}$ depending only on the ratio of $\|g_1\|$ and $\|g_2\|$.
\end{lemma}

\begin{proof}
When $A<0$ we have
\[
\tilde g_1 - g_1 = -\frac{A}{\|g_2\|^2}g_2,
\qquad
\tilde g_2 - g_2 = -\frac{A}{\|g_1\|^2}g_1,
\]
and therefore
\[
g^{\text{PC}} - g^{\text{mix}}
=
(1-p)(\tilde g_1-g_1) + p(\tilde g_2-g_2)
=
-A\Bigl(\frac{1-p}{\|g_2\|^2}g_2 + \frac{p}{\|g_1\|^2}g_1\Bigr).
\]
Taking norms gives
\[
\|g^{\text{PC}} - g^{\text{mix}}\|
\le
|A|\Bigl(\frac{1-p}{\|g_2\|} + \frac{p}{\|g_1\|}\Bigr),
\]
and the second inequality in~\eqref{eq:pcgrad_diff_bound} follows by bounding each coefficient by the maximum of $1/\|g_1\|$ and $1/\|g_2\|$.  
Dividing by $\|g^{\text{mix}}\|$ and using $|A|$ where $|A| = \|g_1\|\|g_2\||\cos\angle(g_1,g_2)|$ yields~\eqref{eq:pcgrad_cos_bound} for a suitable constant $C_\text{rel}$.
\end{proof}

Lemma~\ref{lem:pcgrad_diff_app} shows that the Gradient Surgery update is a small perturbation of the mixture gradient whenever the cosine is already close to zero.  
Combined with Proposition~\ref{prop:quadratic_alignment_app}, which predicts that $A_t$ (and hence the cosine) decays exponentially under dual-task mixing, this indicates that Gradient Surgery can only have a significant effect during an early transient phase when gradients are strongly conflicting.  
Once the dynamics enter the near-orthogonal regime, Gradient Surgery and the unmodified dual-task schedule follow almost identical update directions.

\paragraph{Connection to Empirical Gradient Surgery Results.}
In Sec.~\ref{sec:pcgrad} and Figure~\ref{fig:pcgrad_alignment}, we empirically compare gradient alignment with and without Gradient Surgery.  
Both buffered variants of Gradient Surgery quickly drive the cosine into a narrow band near zero and then track the unprojected dual-task curve closely, which is consistent with Lemma~\ref{lem:pcgrad_diff_app} and Proposition~\ref{prop:quadratic_alignment_app}.  
Because our stochastic dual-task schedule already reduces gradient conflict rapidly and updates are confined to a low-rank LoRA subspace, the additional projections performed by Gradient Surgery mainly act as a small early-time correction and yield only marginal gains in downstream metrics (Table~\ref{tab:pcgrad_abl}).  
This theoretical picture supports our choice to adopt the simpler stochastic schedule without Gradient Surgery as the default training recipe.

\begin{table}[t]
    \centering
    \caption{\textbf{Ablation on the reference token.} Adding \texttt{<CLS>} yields improved subject identity scores (CLIP-I, DINO-I) and a higher dynamic degree.}
    \resizebox{0.75\columnwidth}{!}{%
    \begin{tabular}{l ccccc}
        \toprule
        \multirow{2}{*}{\textbf{Training Method}}  & \textbf{Motion} & \textbf{Dynamic} & \multirow{2}{*}{\textbf{CLIP-T}} & \multirow{2}{*}{\textbf{CLIP-I}} & \multirow{2}{*}{\textbf{DINO-I}} \\
        & \textbf{Smoothness} & \textbf{Degree} & & & \\
        \midrule
            w/o Ref. token &  \textbf{98.84} & 54.55 & 32.87 & 73.36 & 57.51 \\
        \midrule
            w/ Ref. token &  97.99 & \textbf{78.33} & \textbf{33.22} & \textbf{75.87} & \textbf{58.86} \\
        \bottomrule
    \end{tabular}%
    }
    \label{tab:abl_ref}
\end{table}

\section{Training Strategies \& Ablation Studies}
In this section, we present extensive ablation studies on training strategies and various techniques used. 
Due to the cost of training, we use the model trained with a subset of 4,000 of subject-image pairs~\cite{ominicontrol} unless otherwise stated (\ie, Ours-\textit{mini}), since they tend to converge a little faster for the comparison in this section, with less degraded performance according to the Table~\gmp{5} in the main manuscript.

\subsection{Effect of the Reference Token}
\label{sec:ref_token_ablation}
Table~\ref{tab:abl_ref} demonstrates how adding a dedicated \texttt{<CLS>} token to the prompt affects our model’s performance.
Without this reference token, the model attains slightly higher motion smoothness and marginally better CLIP-T scores, but it under-performs on dynamic degree and identity-focused metrics.
Introducing \texttt{<CLS>} evidently improves subject fidelity and fosters more diverse motion.
We attribute these gains to the reference token guiding the alignment of subject tokens ($\mathbf{X}_{\text{in}}$) with the textual prompt more explicitly, leading to stronger identity preservation and more coherent motion variations. 

\begin{table}[t]
    \centering
    \caption{\textbf{Comparison with T2V-only stochastically-switched finetuning, with image drop probability of 1.} Switching to text-only input (T2V) moderately boosts CLIP-T but hurts subject fidelity and dynamic degree.}
    \resizebox{0.85\columnwidth}{!}{%
    \begin{tabular}{l ccccc}
        \toprule
        \multirow{2}{*}{\textbf{Training Method}}   & \textbf{Motion} & \textbf{Dynamic} & \multirow{2}{*}{\textbf{CLIP-T}} & \multirow{2}{*}{\textbf{CLIP-I}} & \multirow{2}{*}{\textbf{DINO-I}} \\
      & \textbf{Smoothness}  & \textbf{Degree}& & & \\
        \midrule
        T2I + T2V (joint) &  98.85 & 44.14 & {33.41} & 72.71 & 48.68 \\
        \midrule
           Ours&  97.99 & {78.33} & 33.22 & {75.87} & {58.86} \\
        \bottomrule
    \end{tabular}%
    }
    \label{tab:t2v_i2v}
\end{table}

\subsection{T2V vs. I2V Fine-tuning}
We also ablate replacing our \emph{image-to-video} (I2V) fine-tuning with a text-to-video (T2V) setup. 
In Table~\ref{tab:t2v_i2v}, the T2I+T2V (joint) attains a slightly better CLIP-T but lower dynamic degree and subject alignment (CLIP-I, DINO-I). 
This suggests that T2V fine-tuning struggles when introducing a novel subject identity solely through text, resulting in weaker overall identity preservation. 
By contrast, our I2V approach strikes a better balance, preserving subject details (CLIP-I: 75.87, DINO-I: 58.86) and maintaining sufficient motion (dynamic degree: 78.33).

\subsection{Effect of Using Subject-Video Pairs}

We further ablate the role of subject-video pairs by replacing arbitrary videos with subject-video pairs in the motion-training branch. Specifically, we compare our default setting with two variants: one that replaces 20\% of arbitrary videos with subject-video pairs, and another that uses only subject-video pairs without subject-image pairs or arbitrary videos. This experiment evaluates whether explicit subject-motion correspondence provides additional benefit for subject-driven video generation. Note that, in this ablation study, we use full 200K subject-image pairs.

As shown in Table~\ref{tab:s2v_ablation}, we observe that the use of subject-image pairs in our default setting shows compatible performance with alternatives that use subject-video pairs: partially replacing arbitrary videos with subject-video pairs, or using only subject-video pairs in place of both subject-image pairs and arbitrary videos. This result suggests that subject-image pairs, when combined with arbitrary videos, can provide a practical training signal for subject-driven video generation without requiring subject-video pairs.

\begin{table}[t]
    \centering
    \caption{\textbf{Ablation on subject-video pairs (S2V).} We compare the default setting with variants that replace arbitrary videos with S2V pairs. S2V-20\% replaces 20\% of arbitrary videos with S2V pairs, while S2V-100\% uses only S2V pairs.}
    \resizebox{0.75\columnwidth}{!}{%
    \begin{tabular}{l ccc ccc}
        \toprule
        \multirow{2}{*}{\textbf{Method}}
        & \multicolumn{3}{c}{\textbf{CogVideoX}}
        & \multicolumn{3}{c}{\textbf{Wan}} \\
        \cmidrule(lr){2-4} \cmidrule(lr){5-7}
        & \textbf{MS} & \textbf{DD} & \textbf{CLIP-I}
        & \textbf{MS} & \textbf{DD} & \textbf{CLIP-I} \\
        \midrule
        Ours        & 98.45 & \textbf{69.64} & 77.14 & 98.53 & \textbf{66.67} & 78.28 \\
        + S2V-20\%  & \textbf{98.95} & 51.79 & 73.71 & 98.72 & 56.67 & 77.68 \\
        + S2V-100\% & 98.39 & 64.17 & 71.77 & \textbf{98.88} & 46.67 & \textbf{79.89} \\
        \bottomrule
    \end{tabular}%
    }
    \label{tab:s2v_ablation}
\end{table}

\begin{table}[t]
    \centering
    \caption{\textbf{Ablation study on using a different number of videos.}}
    \resizebox{0.75\columnwidth}{!}{%
    \begin{tabular}{l ccccc}
        \toprule
        \textbf{Video}  & \textbf{Motion} & \textbf{Dynamic} & \multirow{2}{*}{\textbf{CLIP-T}} & \multirow{2}{*}{\textbf{CLIP-I}} & \multirow{2}{*}{\textbf{DINO-I}} \\
        \textbf{Count}& \textbf{Smoothness} & \textbf{Degree} & & & \\
        \midrule
        1K &  99.03 & 59.66 & 32.16 & 72.69 & 56.98 \\
        2K &  98.96 & 52.25 & 32.49 & 72.13 & 54.42 \\
        3K &  98.79 & 55.46 & 32.04 & 72.79 & 55.57 \\
        4K &  97.99 & 78.33 & 33.22 & 75.87 & 58.86 \\
        \bottomrule
    \end{tabular}%
    }
    \label{tab:abl_vid_cnt}
\end{table}

\subsection{Effect of Varying the Video Dataset Size}
\label{sec:video_dataset_size}

Table~\ref{tab:abl_vid_cnt} reports how our method’s performance changes when using different amounts of arbitrary videos for I2V fine-tuning (1K, 2K, 3K, and 4K videos).
Notably, with only 1K videos, we already obtain relatively strong results, suggesting that even a small unlabeled corpus can restore temporal consistency to some extent. However, increasing the video count to 4K (our default setting) steadily improves dynamic degree from 59.66 to 78.33 and also boosts identity fidelity (DINO-I) from 56.98 up to 58.86, indicating more consistent subject representation across frames.

We also observe a modest variation in CLIP-T and CLIP-I scores when moving from 1K to 4K videos, suggesting that a larger video dataset helps balance the preservation of subject detail and temporal motion, without overfitting to specific frames or motion patterns. In short, while our method is fairly robust to smaller unlabeled datasets, using around 4K (\ie, 1\% of Pexels~\cite{pexels} dataset) videos offers the best trade-off between data efficiency and stable motion/appearance results.

\subsection{Effect of Varying the Switching Probability}
\label{sec:p_ablation}

In Table~\ref{tab:abl_p}, we examine how different values of $p$—the probability of sampling unlabeled video data (I2V fine-tuning)—affect overall performance. When $p=0.0$, the model relies solely on subject-driven image generation (SDI-Gen) fine-tuning, leading to static video but with moderate identity scores.
Increasing $p$ to 0.2 or 0.4 yields balanced improvements across most metrics, reflecting better coordination between identity and motion. At $p=0.6$, the model dedicates a greater share of updates to I2V fine-tuning, strengthening identity alignment (CLIP-I: 76.71, DINO-I: 62.84) while keeping the dynamic degree stable (60.87). Although different $p$ values trade off between motion smoothness and identity fidelity to varying degrees, our chosen $p=0.2$ demonstrates a strong overall balance, as highlighted in the main paper.

\begin{table}[t]
    \centering
    \caption{\textbf{Ablation study on using different 'p' to choose I2V finetuning.}}
    \resizebox{0.75\columnwidth}{!}{%
    \begin{tabular}{l ccccc}
        \toprule
        \textbf{Video}   & \textbf{Motion}& \textbf{Dynamic} & \multirow{2}{*}{\textbf{CLIP-T}} & \multirow{2}{*}{\textbf{CLIP-I}} & \multirow{2}{*}{\textbf{DINO-I}} \\
        \textbf{Ratio} & \textbf{Smoothness}& \textbf{Degree} & & & \\
        \midrule
        0.0 & 99.60 & 0.84 & 32.67 & 71.15 & 43.19 \\
        0.2 (Ours) &  97.99 & 78.33 & 33.22 & 75.7 & 58.86 \\
        0.4 & 98.31 & 59.12 & 31.94 & 73.53 & 56.60 \\
        0.6 &  98.33 & 60.87 & 31.31 & 76.71 & 62.84 \\
        \bottomrule
    \end{tabular}%
    }
    \label{tab:abl_p}
\end{table}

\subsection{Effect of Random Initial Frame Selection \& Dropping.}

We further investigate the effects of \emph{random frame selection} and \emph{image-token dropping} during I2V fine-tuning through qualitative analysis. 
As shown in Figure~\ref{fig:quali_random}, when neither technique is applied, the first reference image exhibits excessive dominance over subsequent frames, resulting in limited temporal variation.
Enabling \emph{random frame selection} alone partially restores temporal dynamics, yet introduces visual artifacts in certain regions (e.g., shirts, oranges).
In contrast, the joint application of both techniques achieves a more favorable balance: artifacts are substantially reduced while the reference image blends more naturally into the generated sequence.
These qualitative observations underscore the importance of mitigating over-reliance on the first frame and preventing excessive conditioning on a single reference image to achieve smoother, more temporally coherent video generation.

\begin{figure}[t]
    \centering
    \includegraphics[width=0.76\columnwidth]{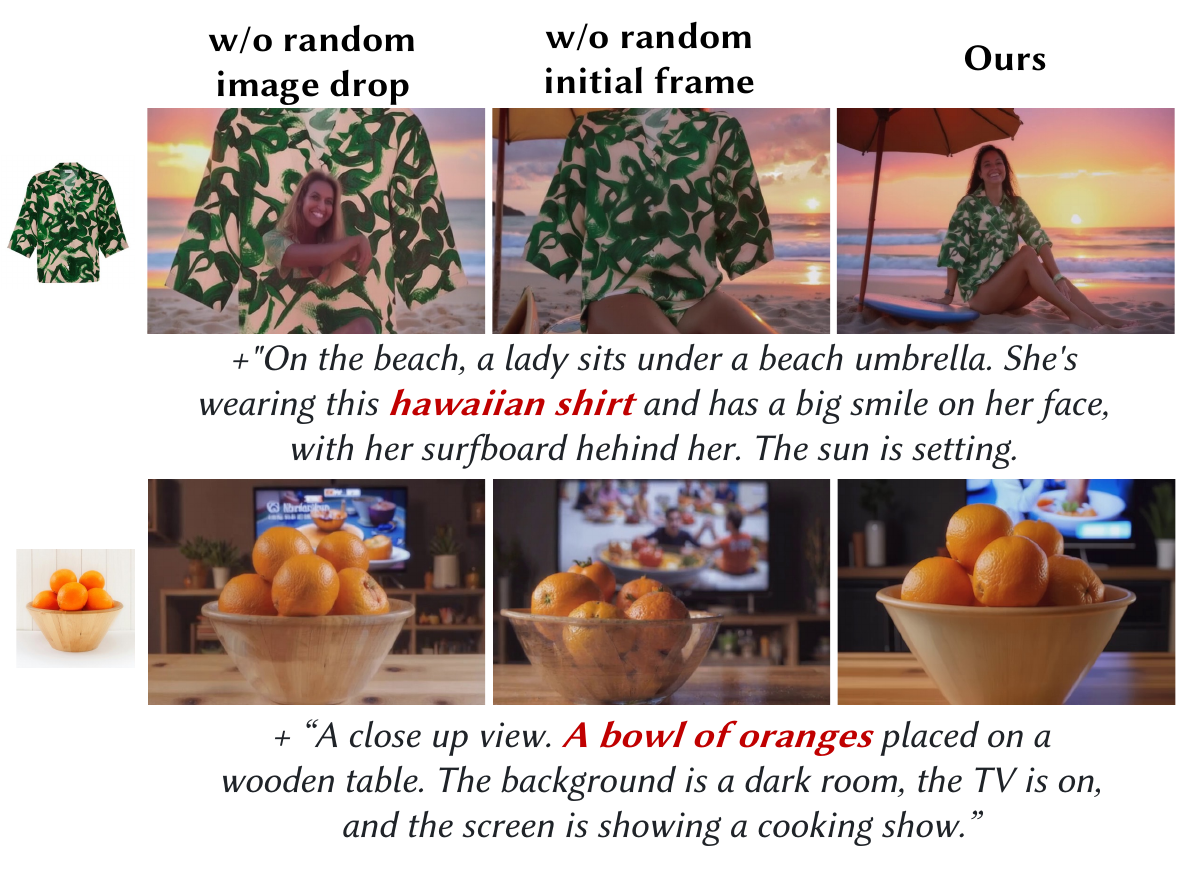}
    \caption{\textbf{Effect of random initial frame and image token dropping.}}
    \label{fig:quali_random}
\end{figure}

\subsection{Effect of LoRA Rank on Training}
We further study the effect of LoRA rank on training with CogVideoX-5B in Table~\ref{tab:lorarank}.
In qualitative comparison in Figure~\ref{fig:lorarank}, the differences across LoRA ranks are relatively small, and the standard evaluation metrics do not show a large gap, as the identity injection capability is not affected a lot with LoRA rank. However, the fidelity of the injected identity is affected a lot, showing clear qualitative trend.

\setlength{\intextsep}{-10pt}
\begin{wrapfigure}[10]{r}{0.36\columnwidth}
    \centering
    \includegraphics[width=0.98\linewidth, trim={2 23 2 12}, clip]{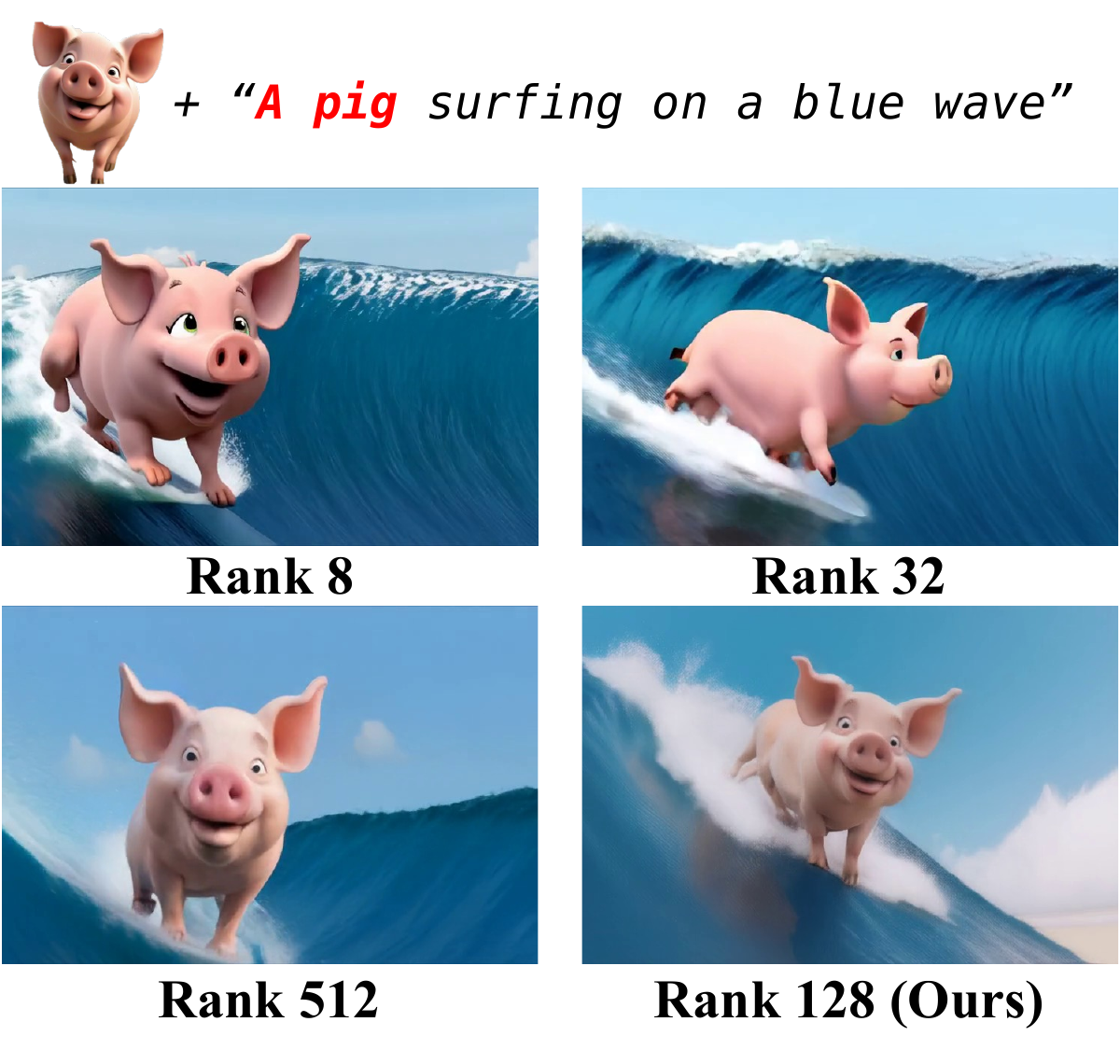}
    \caption{\textbf{LoRA Comparison.}}
    \label{fig:lorarank}
\end{wrapfigure}
When the LoRA rank is reduced below 32, the generated videos exhibit noticeable degradation in visual quality, with weaker textures, blurrier structures, and the loss of fine-grained details.
In particular, subtle appearance cues that are important for subject fidelity tend to disappear as the rank becomes too small.
When we increase the LoRA rank to 512, the fidelity gets improved in qualitative manner, we chose rank 128 for the memory and compute balance.

\begin{table}[t]
    \centering
    \caption{\textbf{Abalation study on LoRA Rank on Training.}}
    \resizebox{0.75\columnwidth}{!}{%
    \begin{tabular}{l ccccc}
        \toprule
        \textbf{LoRA}   & \textbf{Motion}& \textbf{Dynamic} & \multirow{2}{*}{\textbf{CLIP-T}} & \multirow{2}{*}{\textbf{CLIP-I}} & \multirow{2}{*}{\textbf{DINO-I}} \\
        \textbf{Rank} & \textbf{Smoothness}& \textbf{Degree} & & & \\
        \midrule
        8 & 97.80 & 73.68 & 32.89 & 74.65 & 57.17 \\
        32 & 98.20 & 71.43 & 33.31 & 75.97 &  59.27 \\
        128 (Ours) & 97.99 & 78.33 & 33.22 & 75.7 & 58.86  \\
        512 & 98.44 & 69.64 & 33.75 & 73.97 & 55.67 \\
        \bottomrule
    \end{tabular}%
    }
    \label{tab:lorarank}
\end{table}

\subsection{Effect of Training Strategies}

We conduct an ablation on three training strategies: \emph{image-only}, \emph{two-stage}, and our \emph{alternating} (stochastically switched) approach. 
Qualitative comparisons in Figure~\ref{fig:quali_train_st} reveal distinct failure modes for each strategy.
The \emph{image-only} approach produces videos with minimal subject motion, effectively generating near-static sequences that artificially inflate motion-smoothness metrics while sacrificing temporal dynamics, as demonstrated in Table~\gmp{4} of the main manuscript.
\emph{Two-stage} training, which first fine-tunes on subject-image pairs before switching to arbitrary videos, introduces pronounced visual artifacts and exhibits catastrophic forgetting of the base model's text-to-video capability.
These artifacts manifest as inconsistent textures and structural distortions that compromise overall video quality, despite preserving certain fine-grained image details.

In contrast, our proposed \emph{alternating} strategy successfully navigates these failure modes by interleaving identity injection and motion-aware objectives throughout training.
As illustrated in Figure~\ref{fig:quali_train_st}, this approach generates videos with both natural motion dynamics and coherent visual appearance, avoiding the static collapse observed in image-only training and the artifact accumulation in two-stage training.
The qualitative results demonstrate that stochastic alternation between objectives effectively maintains the model's generative capabilities across both modalities.

\begin{figure}[t]
    \centering
    \includegraphics[width=0.96\columnwidth]{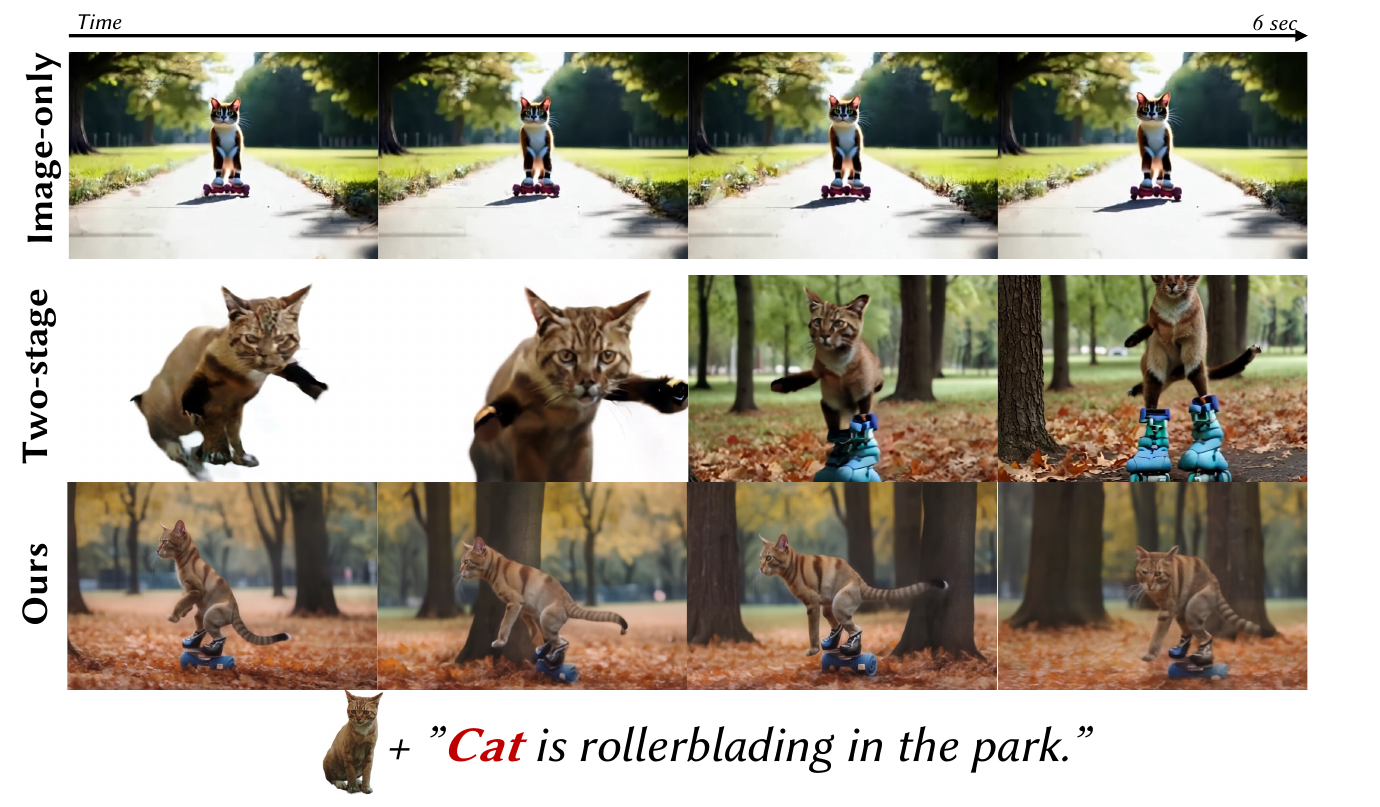}
    \caption{\textbf{Qualitative result on ablation study of our component in temporal awareness preservation.}}
    \label{fig:quali_train_st}
\end{figure}

\subsection{Exploration of Various Training Strategies with Penalties}

We further explore whether standard regularization strategies can mitigate the failure modes of the \emph{image-only} and \emph{two-stage} training schemes in Tables~\ref{tab:abl_image-only-l2}-~\ref{tab:abl_two-stage-l2} and Figures~\ref{fig:quali_image_only_penalty}-~\ref{fig:quali_i2v_penalty}
In particular, we apply commonly used penalties, including L2 regularization, and anchoring to the pretrained weights, during both image-only and two-stage fine-tuning.
However, these penalties do not meaningfully alleviate catastrophic forgetting in our setting.

Although such regularizers can slightly constrain parameter drift, they do not recover the pretrained model's temporal modeling capability once optimization becomes overly biased toward the image objective.
As a result, image-only training still tends to produce near-static videos, while two-stage training continues to suffer from degraded video quality and forgetting of the original motion prior.
Overall, we find that adding these practical penalties is not sufficient to resolve the core instability of these training strategies, which supports the need for our stochastic alternating schedule.

\begin{figure}[t]
    \centering
    \includegraphics[width=0.96\columnwidth]{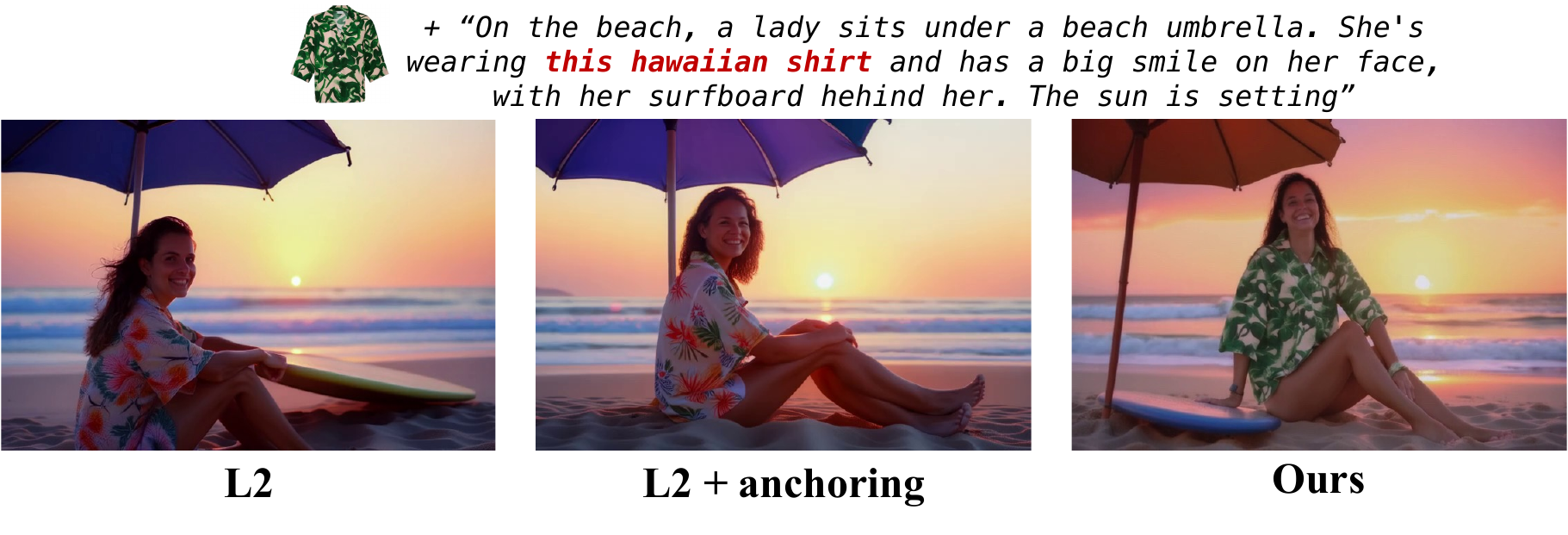}
    \caption{\textbf{Qualitative result on comparison with image-only approaches with penalty.}}
    \label{fig:quali_image_only_penalty}
\end{figure}

\begin{table}[t]
    \centering
    \caption{\textbf{Exploration on various training strategies with penalties on image-only.}}
    \resizebox{0.75\columnwidth}{!}{%
    \begin{tabular}{l ccccc}
        \toprule
        \textbf{Training}   & \textbf{Motion}& \textbf{Dynamic} & \multirow{2}{*}{\textbf{CLIP-T}} & \multirow{2}{*}{\textbf{CLIP-I}} & \multirow{2}{*}{\textbf{DINO-I}} \\
        \textbf{Strategy} & \textbf{Smoothness}& \textbf{Degree} & & & \\
        \midrule
        Ours & 97.99 & 78.33 & 33.22 & 75.7 & 58.86  \\
        \midrule
        image-only & 99.60 & 0.84 & 32.67 & 71.15 & 43.19 \\
        image-only + L2 & 99.36 & 31.67 & 32.91 & 73.12 & 48.23 \\
        image-only + L2 anchoring & 99.26 & 38.33 & 32.64 & 72.15 & 44.71\\
        \bottomrule
    \end{tabular}%
    }
    \label{tab:abl_image-only-l2}
\end{table}

\begin{figure}[t]
    \centering
    \includegraphics[width=0.96\columnwidth]{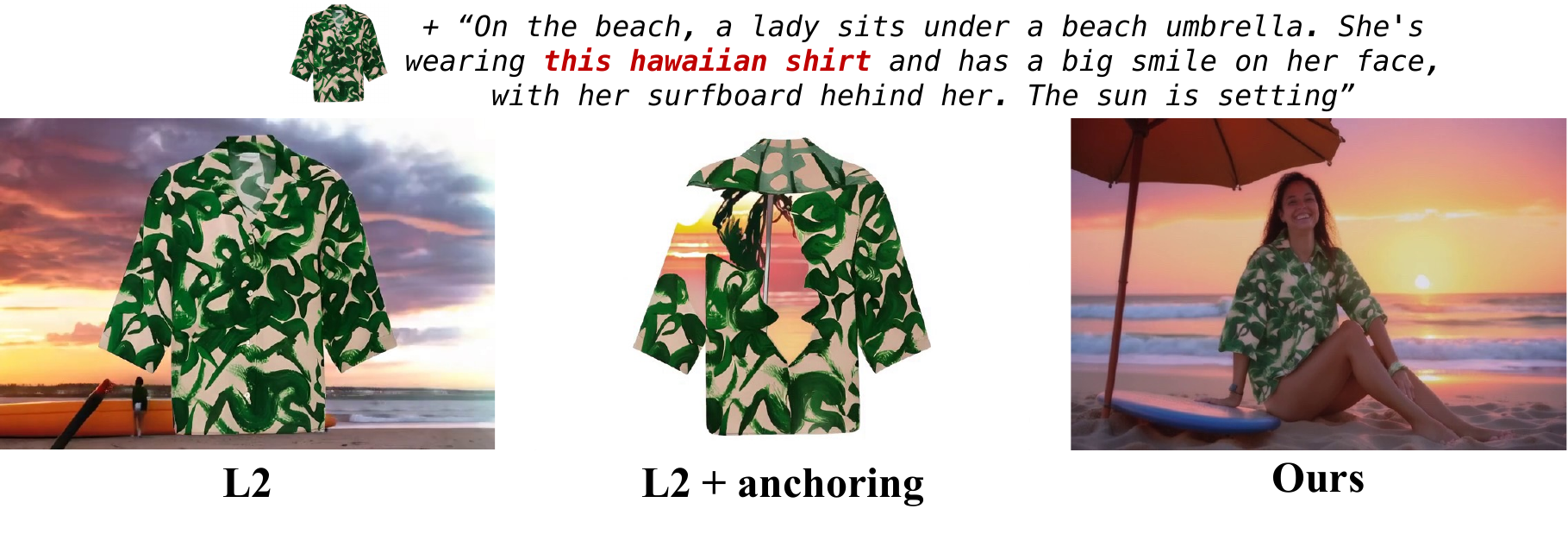}
    \caption{\textbf{Qualitative result on comparison with two-stage approaches with penalty.}}
    \label{fig:quali_i2v_penalty}
\end{figure}

\begin{table}[t]
    \centering
    \caption{\textbf{Exploration on various training strategies with penalties on two-stage.}}
    \resizebox{0.80\columnwidth}{!}{%
    \begin{tabular}{l ccccc}
        \toprule
        \textbf{Training}   & \textbf{Motion}& \textbf{Dynamic} & \multirow{2}{*}{\textbf{CLIP-T}} & \multirow{2}{*}{\textbf{CLIP-I}} & \multirow{2}{*}{\textbf{DINO-I}} \\
        \textbf{Strategy} & \textbf{Smoothness}& \textbf{Degree} & & & \\
        \midrule
        Ours & 97.99 & 78.33 & 33.22 & 75.7 & 58.86  \\
        \midrule
        two-stage & 96.04 & 81.51 & 28.96 & 84.73 & 76.13\\
        two-stage + L2 & 97.06 & 81.67 & 31.10 & 77.14 & 62.15 \\
        two-stage + L2 anchoring & 96.60 & 78.33 & 31.16 & 77.71 & 61.27\\
        \bottomrule
    \end{tabular}%
    }
    \label{tab:abl_two-stage-l2}
\end{table}

\subsection{Flow Entropy and Copy Score.}
We additionally report two lightweight temporal metrics, Flow Entropy (FE) and Copy Score (CS), to quantify whether the generated subject exhibits diverse local motion or behaves like a rigid copy-pasted layer. For each generated video, we uniformly sample 8 frames and resize them to a fixed analysis width of 288 pixels while preserving aspect ratio. We then compute dense optical flow between adjacent sampled frames using Farneback optical flow.

Since the copy-paste artifact of interest usually appears around the foreground subject, we restrict the flow analysis to a central elliptical region. Within this region, we further keep high-gradient pixels, using image gradients from the previous frame, so that the measurement is dominated by subject boundaries and textured subject regions rather than flat background. If the high-gradient mask is too small, we fall back to the full central ellipse.

Let $\mathbf{u}_i = (u_i, v_i)$ be the optical-flow vectors selected from this subject-centered region. We discard nearly static vectors by keeping only pixels whose flow magnitude is larger than
\[
\max(0.1, \operatorname{Quantile}_{0.30}(\|\mathbf{u}_i\|_2)).
\]
For the remaining moving pixels, we compute flow directions
\[
\theta_i = \operatorname{atan2}(v_i, u_i)
\]
and build a 12-bin histogram over $[-\pi, \pi]$. We convert this histogram into a probability distribution $p$ and compute the normalized directional entropy:
\[
H_{\mathrm{flow}}
=
-\frac{1}{\log 12}
\sum_{b=1}^{12} p_b \log(p_b + \epsilon).
\]
The final Flow Entropy score is reported on a 0--100 scale:
\[
\mathrm{FE} = 100 \cdot H_{\mathrm{flow}}.
\]
Higher FE indicates more diverse local subject motion.

We define Copy Score as the complementary flow-concentration score:
\[
\mathrm{CS} = 100 \cdot (1 - H_{\mathrm{flow}}).
\]
Lower CS is better. Intuitively, a rigidly copied or nearly static subject tends to concentrate optical-flow directions into only a few bins, yielding low entropy and high CS. In contrast, articulated or locally deforming subject motion spreads flow directions across bins, yielding high FE and low CS.
\section{Additional Evaluation}

\subsection{Human Preference Study}
\label{sec:human_study}

While benchmark metrics provide useful quantitative insights, they can sometimes be affected by degenerate behaviors, such as nearly static videos that obtain artificially high scores.
To complement our quantitative evaluation, we conduct two human preference studies, one for our CogVideoX-5B variant and another for our Wan-family variants and demonstrate the result in Table~\ref{tab:human_preference}.

For the CogVideoX-5B setting, we compare our method against VideoBooth, Vidu 2.0, and OminiControl with an I2V backbone.
We randomly select 20 generated videos from each method and ask 30 participants to rate the results on a five-point Likert scale across four criteria: identity consistency, prompt alignment, motion quality, and overall visual appeal.

For the Wan-family setting, we further evaluate our Wan 2.1-1.3B and Wan 2.2-5B variants against Phantom and VACE baselines.
In this study, 50 participants complete a 30-question evaluation using the same criteria. Our method is competitive across most metrics and achieves the best overall score. These results indicate that our approach provides a favorable balance between identity preservation and temporal dynamics, which aligns well with human judgments of video quality.

We observe that ours outperform in both cases: for CogVideoX-5B variant and Wan-variants.
Especially for Wan variants, ours with Wan 2.2-5B variant outperforms most of the metrics in all models including 14B variants of Phantom and VACE, demonstrating our method's effectiveness.

\begin{table}[t]
    \centering
    \caption{\textbf{Human preference study results in Likert scale of 1--5.}
    Results are grouped by the backbone/study setting.}
    \resizebox{0.80\columnwidth}{!}{%
    \begin{tabular}{@{}lcccc@{}}
        \toprule
        \textbf{Method} 
        & \shortstack{\textbf{ID}\\\textbf{Consistency}} 
        & \shortstack{\textbf{Prompt}\\\textbf{Alignment}} 
        & \shortstack{\textbf{Motion}\\\textbf{Quality}} 
        & \shortstack{\textbf{Overall}\\\textbf{Quality}} \\
        \midrule

        \multicolumn{5}{@{}l}{\textit{CogVideoX-5B study}} \\
        Omini+I2V        & 3.80 & 3.78 & 3.62 & 3.44 \\
        VideoBooth       & 3.25 & 3.20 & 3.08 & 2.91 \\
        Vidu 2.0         & 3.42 & 3.24 & 3.22 & 3.03 \\
        \midrule
        Ours (CogVideoX-5B) 
                         & \textbf{4.08} & \textbf{3.82} & \textbf{3.88} & \textbf{3.71} \\
        \midrule

        \multicolumn{5}{@{}l}{\textit{Wan study}} \\
        Phantom-1.3B     & 3.52 & 3.59 & 3.20 & 3.29 \\
        VACE-1.3B        & 3.20 & 3.56 & 3.28 & 3.28 \\
        Phantom-14B      & 3.55 & 3.87 & \textbf{3.89} & 3.72 \\
        VACE-14B         & \textbf{4.12} & 3.69 & 3.70 & 3.73 \\
        \midrule
        Ours (Wan-1.3B)  & 3.82 & 3.48 & 3.27 & 3.34 \\
        Ours (Wan-5B)    & 4.11 & \textbf{3.88} & 3.73 & \textbf{3.80} \\
        \bottomrule
    \end{tabular}%
    }
    \label{tab:human_preference}
\end{table}

\begin{table}[t]
    \caption{\textbf{Temporal Evaluation following FloVD~\cite{flovd}, assessing whether motion dynamics improve compared to image-only or two-stage training.} Small - Medium - Large - with each number representing FVD. $\dagger$ denotes Pexels~\cite{pexels}-finetuned version of CogVideoX~\cite{cogvideox}. }
    \centering
    \resizebox{0.55\columnwidth}{!}{%
    \begin{tabular}{lccc}
        \toprule
     \textbf{Method} & Small$\downarrow$ & Medium$\downarrow$ & Large$\downarrow$ \\
        \midrule
        CogVideoX-T2V$\dagger$ & 597.54 & 594.26 & 573.86 \\
        \midrule
        Image-only & 641.92 & 636.42 & 680.34 \\
        Two-stage & 801.97 & 872.30 &  824.03 \\
        \midrule
        Ours & \textbf{512.30} &  \textbf{511.66} & \textbf{550.14} \\
        \bottomrule
    \end{tabular}%
 }
    \label{tab:temporal_analysis}
\end{table}

\subsection{Temporal Modeling Evaluation}
\label{sec:temporal_eval_supp}

We evaluate temporal modeling performance using videos sampled from the Pexels dataset, ensuring no overlap with those used during training. Following the protocol of FloVD~\cite{flovd}, we categorize videos into three groups based on optical flow magnitude (small: $\leq25$, medium: $25\sim50$, large: $\geq50$) to enable detailed analysis. For object-motion evaluation in particular, we focus on videos with minimal or no camera motion and construct benchmark subsets by applying the same motion-based grouping to a separate collection of Pexels videos not used in our stochastically switched fine-tuning.

\subsubsection{Evaluation Protocol}
$\\$
\noindent\textbf{1) Foreground–Background Segmentation.}
For each video, we use an off-the-shelf segmentation model (e.g., Grounded-SAM2) on the \emph{first frame} to separate foreground and background regions. This allows us to measure object (foreground) motion independently from any camera-induced background shifts.

\noindent\textbf{2) Optical Flow Computation.}
We estimate optical flow between the \emph{first frame} and each subsequent frame using a standard flow estimator (e.g., RAFT~\cite{raft}). Let $\mathbf{u}_{f}(x)$ and $\mathbf{u}_{b}(x)$ denote the per-pixel flow vectors for the foreground and background pixels, respectively, at position $x$. We record:
\[
\begin{aligned}
\text{FlowMag}_{f} &= \frac{1}{N_{f}} \sum_{x \in \text{fg}} \|\mathbf{u}_{f}(x)\|,\\
\text{FlowMag}_{b} &= \frac{1}{N_{b}} \sum_{x \in \text{bg}} \|\mathbf{u}_{b}(x)\|,
\end{aligned}
\]
where $N_{f}$ and $N_{b}$ are the respective pixel counts in the foreground and background masks.

\noindent\textbf{3) Dataset Filtering.}
To ensure negligible camera motion, we \emph{discard} any video whose average magnitude of background flow $\text{FlowMag}_{b}$ exceeds 10~pixels. This filtering step excludes scenes with significant global shifts, retaining only those with primarily object-centric motion.

\noindent\textbf{4) Category Assignment.}
Based on the average magnitude of foreground flow $\text{FlowMag}_{f}$ (averaged over all frames), we categorize videos into:
\begin{itemize}
    \item \emph{Small}: $0 \leq \text{FlowMag}_{f} \leq 25$
    \item \emph{Medium}: $25 < \text{FlowMag}_{f} \leq 50$
    \item \emph{Large}: $\text{FlowMag}_{f} > 50$
\end{itemize}
Each category contains 300 videos, ensuring a balanced evaluation of low-, moderate-, and high-motion scenarios.

\noindent\textbf{5) Evaluation Protocol.}
Within each subset, we use only the \emph{first frame} (including any textual or reference cues, if required) to generate a video of the same length. We then compute FVD~\cite{fvd} between the generated outputs and the ground-truth videos. By comparing FVD across \emph{small}, \emph{medium}, and \emph{large} motion classes, we obtain a clearer picture of how each model (ours vs.\ baselines) adapts to varying object-motion intensities.

\subsubsection{Discussion.}

We compared \emph{image-only}, \emph{two-stage}, and our \emph{dual-task learning} strategies. Image-only training produced a video with poor FVD (Table~\ref{tab:temporal_analysis}). Two-stage fine-tuning leads to high FVD, which aligns with the result of artifacts as in Figure~\ref{fig:quali_train_st}. Our strategy (Ours) excelled, with lower FVD scores confirming superior motion realism and temporal consistency, on par with CogVideoX~\cite{cogvideox}.

One thing to note is that motion-focused split highlights each method’s strengths and weaknesses.
For example, a model might produce near-static outputs for low-motion data, \emph{cheating} on metrics such as smoothness in VBench metrics, yet fail to track fast-moving objects in high-motion videos. 
As observed in FloVD~\cite{flovd}, categorizing by foreground flow magnitude reveals these nuances more effectively than aggregated scores alone.

\noindent\textbf{Additional Details.}
To ensure a fair comparison with our method, we additionally fine-tune the original CogVideoX~\cite{cogvideox} on a subset of the Pexels~\cite{pexels} dataset that is equivalent to the one used in our training. 
Since our model is trained for 4K steps with a sampling ratio $p = 0.2$, we match this by finetuning CogVideoX for 800 steps ($0.2 \times 4000$). As a result, we achieved performance comparable to the original CogVideoX in motion dynamics evaluation.

\subsection{Motion Diversity and Sensitivity}

\begin{table}[t]
    \caption{\textbf{Pexels scaling.}}
    \centering
    \resizebox{0.4\linewidth}{!}{%
    \begin{tabular}{lccc}
        \toprule
        \textbf{Videos} & \textbf{MS} & \textbf{DD} & \textbf{CLIP-I} \\
        \midrule
        80K & 98.38 & 71.67 & 76.11 \\
        40K & 98.22 & 70.00 & 76.06 \\
         4K & 98.45 & 69.64 & 77.14 \\
        \bottomrule
    \end{tabular}%
    }
    \label{tab:scaling_pexels}
\end{table}

Since our training uses only a small subset of arbitrary videos from Pexels~\cite{pexels}, we conduct two additional analyses: (1) whether using 4,000 videos limits motion diversity, and (2) whether motion-modeling performance is sensitive to the particular Pexels-4K subset.

\noindent\textbf{Motion Diversity.}
To examine whether the limited number of replay videos reduces motion diversity, we scale the Pexels subset from 4K to 80K videos and compare VBench metrics in Table~\ref{tab:scaling_pexels}. Increasing the number of Pexels videos does not yield a clear gain, suggesting that excessive replay data is unnecessary when starting from a pre-trained video backbone.

\begin{table}[t]
    \caption{\textbf{Pexels-4K subset sensitivity.}}
    \centering
    \resizebox{0.4\linewidth}{!}{%
    \begin{tabular}{lccc}
        \toprule
        \textbf{Subset} & \textbf{MS} & \textbf{DD} & \textbf{CLIP-I} \\
        \midrule
        Seed 100 & 98.64 & 69.17 & 75.45 \\
        Seed 200 & 98.03 & 77.59 & 74.46 \\
        Ours     & 98.45 & 69.64 & 77.14 \\
        \bottomrule
    \end{tabular}%
    }
    \label{tab:seed_sensit}
\end{table}

\noindent\textbf{Subset Sensitivity.}
We further vary the random seed used to sample the Pexels-4K subset and observe stable results in Table~\ref{tab:seed_sensit}, indicating that the performance is not caused by a lucky subset. Pexels is used only for preserving temporal dynamics, not identity supervision; we also audit the evaluation references and prompts and find no overlap with the Pexels videos used for training.

\section{Estimation of Other Methods' Training Cost}
\label{sec:cost_estim}
We estimate the train time of CustomCrafter~\cite{customcrafter} and VACE~\cite{vace} based on the given implementation details in the manuscript.

\subsection{CustomCrafter.}
The estimated GPU hours for training a single subject are in the range of 100--300 A100-hours, with a median estimate of around 200 A100-hours.
$$\text{wall-clock time (hours)} = \frac{10\,000 \times t}{3\,600}$$
$$\text{total GPU-hours} = 4 \times \text{wall-clock time}$$
where $t \in [10, 30]$ is the estimated time per iteration in seconds.

\subsection{VACE.}
The estimated GPU hours for training VACE on LTX-Video are in the range of 9,000--27,000 A100-hours, with a median estimate of around 18,000 A100-hours.
$$\text{number of training steps} = 200\,000$$
$$\text{wall-clock time per training (hours)} = \frac{200\,000 \times t}{3\,600}$$
$$\text{GPU-hours} = 16 \times \frac{200\,000 \times t}{3\,600} = \frac{3\,200\,000 t}{3\,600} = \frac{8\,000 t}{9}$$
where $t \in [10, 30]$ is the estimated time per iteration in seconds.
The estimated GPU hours for training VACE on Wan-T2V are in the range of 70,000--210,000 A100-hours, with a median estimate of around 140,000 A100-hours.
$$\text{number of training steps} = 200\,000$$
$$\text{wall-clock time per training (hours)} = \frac{200\,000 \times t}{3\,600}$$
$$\text{GPU-hours} = 128 \times \frac{200\,000 \times t}{3\,600} = \frac{25\,600\,000 t}{3\,600} = \frac{64\,000 t}{9}$$
where $t \in [10, 30]$ is the estimated time per iteration in seconds.

\subsection{VideoBooth.}
VideoBooth trains two stages (coarse and fine), each for 218,000 steps, totaling 436,000 training steps.
We estimate the total compute assuming an 8$\times$A100 setup.
$$\text{number of training steps} = 218{,}000 + 218{,}000 = 436{,}000$$
$$\text{wall-clock time per training (hours)} = \frac{436{,}000 \times t}{3{,}600}$$
$$\text{GPU-hours} = 8 \times \frac{436{,}000 \times t}{3{,}600}
= \frac{3{,}488{,}000\, t}{3{,}600}
= \frac{8{,}720\, t}{9}$$
where $t \in [0.8, 2.0]$ is the estimated time per iteration in seconds.

Therefore, the estimated GPU hours for training VideoBooth are in the range of 775--1,938 A100-hours, with a median estimate of around 1,356 A100-hours (at $t=1.4$).

\section{Limitation Analysis \& Additional Result}

\noindent\textbf{Limitation \& Analysis.}
As discussed in the main manuscript, strong stylization can occasionally hinder faithful generation, leading to failure cases such as those shown in Figure~\ref{fig:human_limit}, and in some cases, to blurry outputs.
However, this phenomenon is less pronounced when using Wan~2.2-5B as the backbone.
We conjecture that this is because Wan~2.2-5B provides stronger identity preservation, as evidenced by its substantially higher FaceSim score in the OpenS2V~\cite{s2vnexus} results reported in Table~\gmp{3} of the main paper.

\begin{figure*}[t]
    \centering
    \captionsetup{type=figure}
    \includegraphics[width=0.98\textwidth]{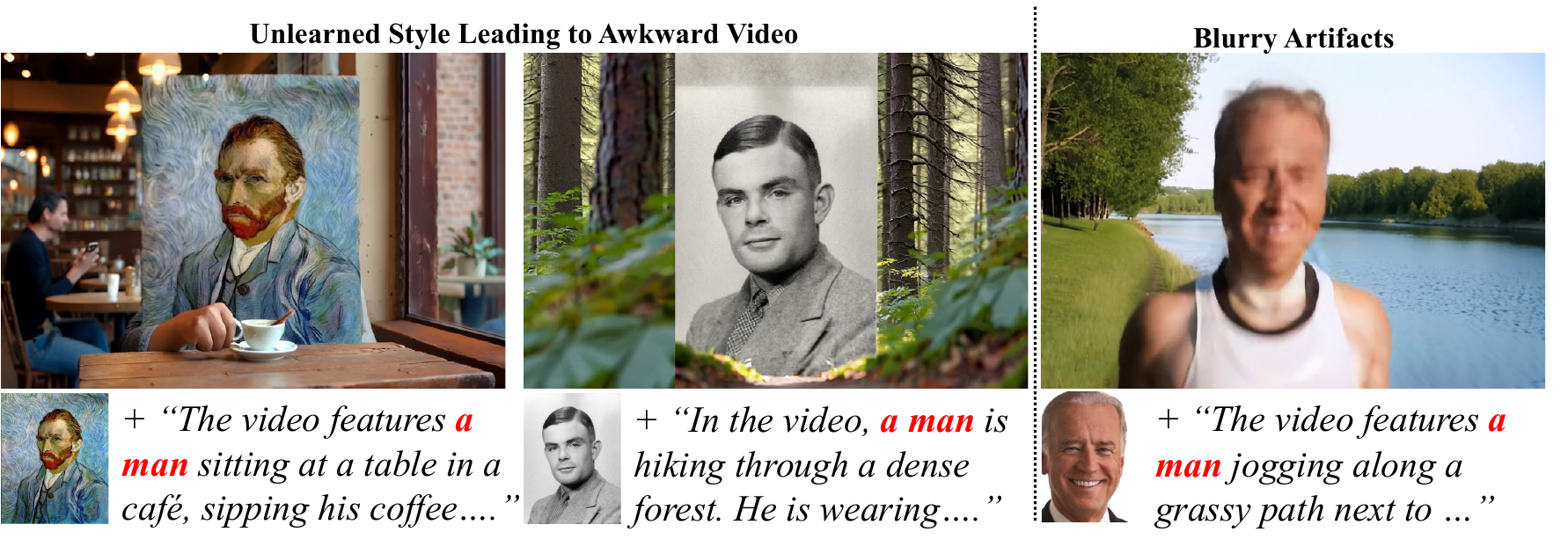}
    \caption{\textbf{Limitation in human-driven generation.}}
    \label{fig:human_limit}
\end{figure*}

\noindent\textbf{Additional Qualitative Results.}
We present additional qualitative comparisons against per-subject tuning baselines in Figs.~\ref{fig:sup1} and~\ref{fig:sup2}, as well as zero-shot baselines in Figs.~\ref{fig:sup3} and~\ref{fig:sup4}, including state-of-the-art methods~\cite{phantom,vace}.
For these comparisons, we use Phantom-Wan~2.1-1.3B and VACE-Wan~2.1-1.3B, which are denoted as Phantom-1.3B and VACE-1.3B in the supplemental video for brevity.
Note that \emph{mini} refers to our model trained on a 4,000-sample subset of subject-image pairs~\cite{ominicontrol}, and we utilize Ours with CogVideoX-5B variant and Wan 2.2-5B variant for comparison.

We further include additional qualitative video results in the supplemental video, along with benchmark results from Open-S2V-Evaluation~\cite{s2vnexus}.
In addition, to better evaluate performance on deformable subjects, highly dynamic scenes, and human-driven generation, we compare against Phantom-Wan~14B and VACE-Wan~14B, denoted as Phantom-14B and VACE-14B in the supplemental video.
Because 14B-scale models are computationally demanding and require four A100 GPUs for inference, we report the 14B results provided by Open-S2V-Evaluation.


\begin{figure*}[t]
    \centering
    \captionsetup{type=figure}
    \includegraphics[width=0.98\textwidth]{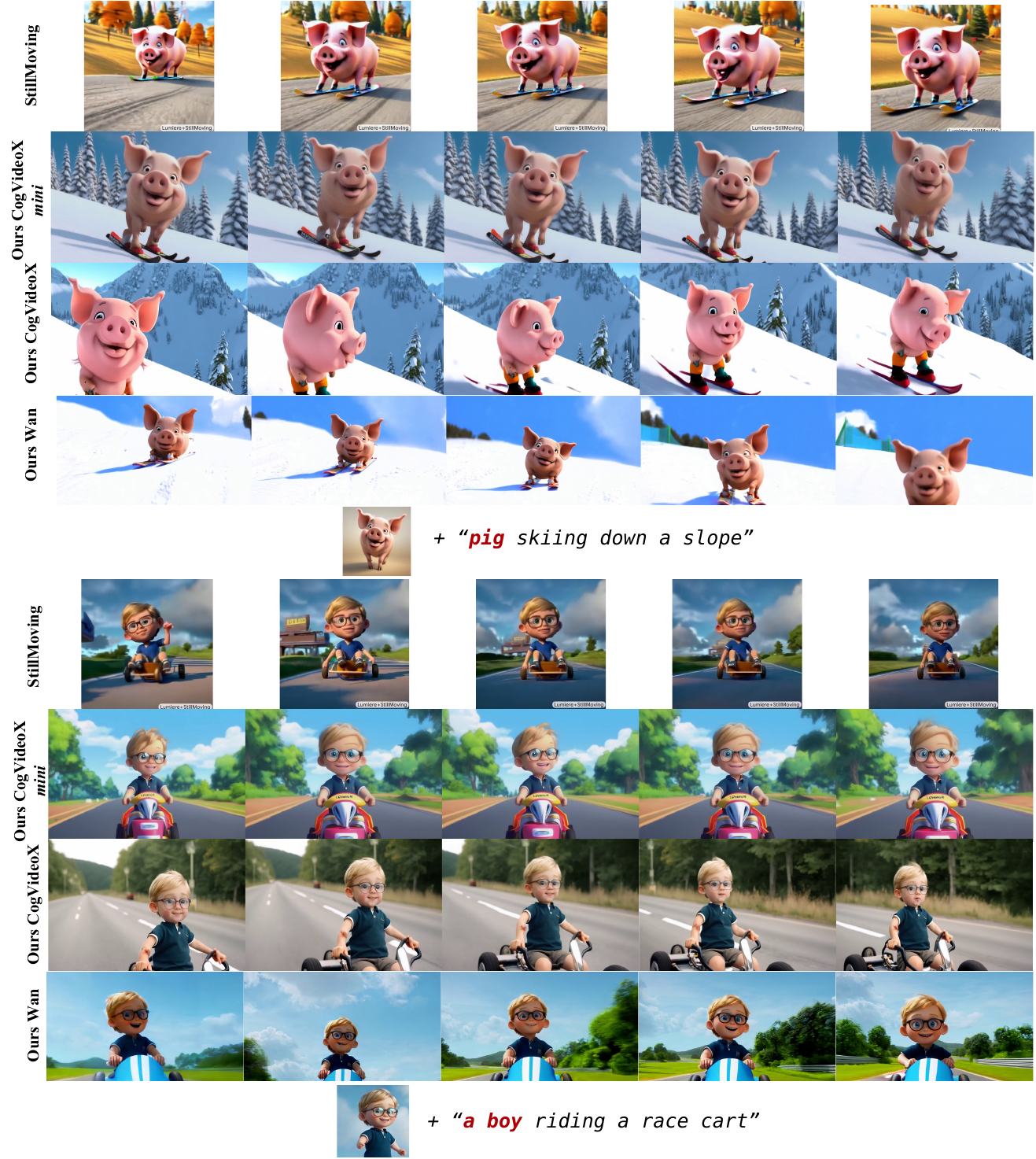}
    \caption{\textbf{Additional qualitative comparison with Still-Moving~\cite{stillmoving}.} Note that \emph{mini} denotes our method trained with a 4K subset of Subject 200K~\cite{ominicontrol}.}
    \label{fig:sup1}
\end{figure*}

\begin{figure*}[t]
    \centering
    \captionsetup{type=figure}
    \includegraphics[width=0.98\textwidth]{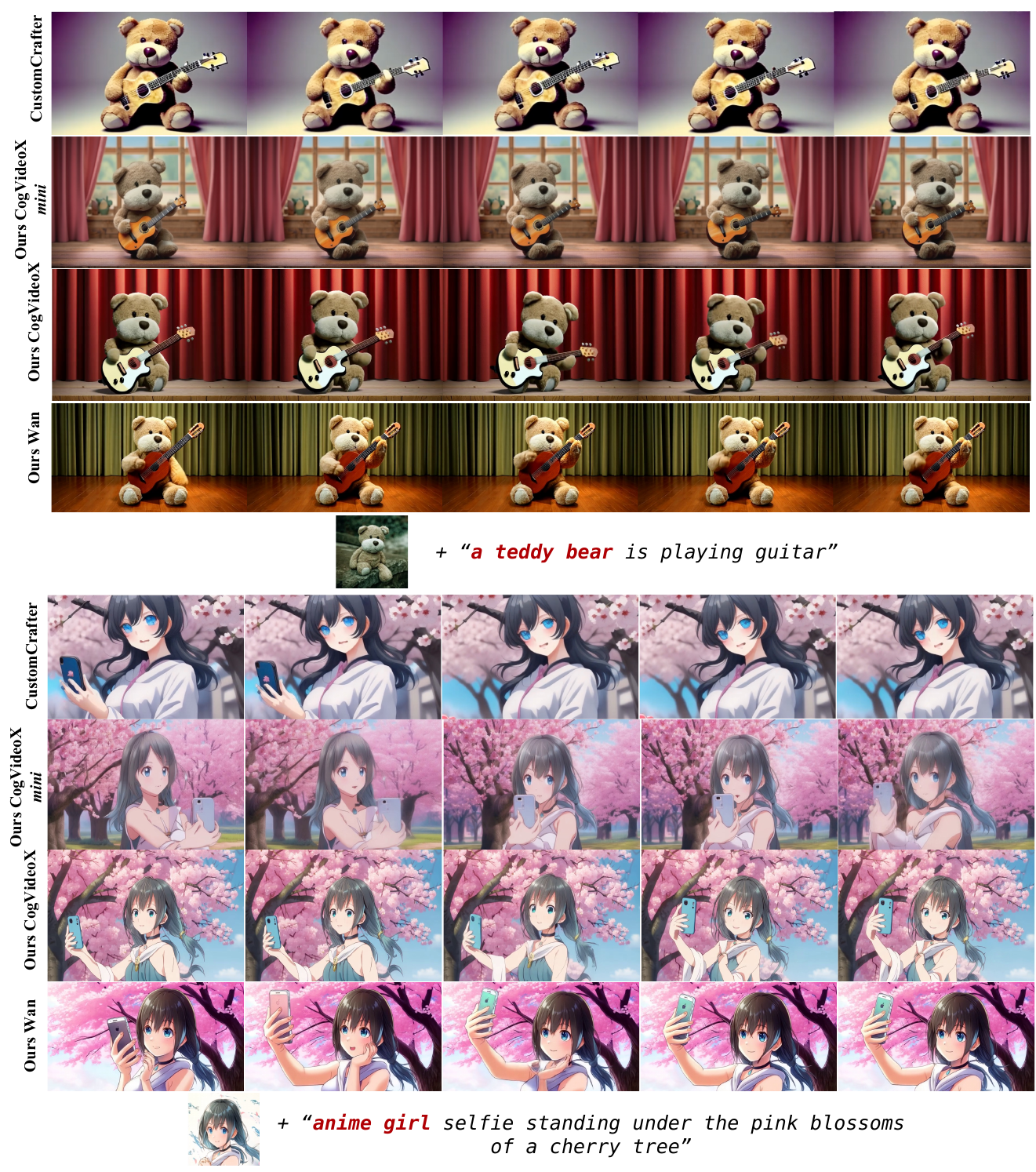}
    \caption{\textbf{Additional qualitative comparison with CustomCrafter~\cite{customcrafter}.} Note that \emph{mini} denotes our method trained with 4K subset of Subject 200K~\cite{ominicontrol}.}
    \label{fig:sup2}
\end{figure*}

\begin{figure*}[t]
    \centering
    \captionsetup{type=figure}
    \includegraphics[width=0.98\textwidth]{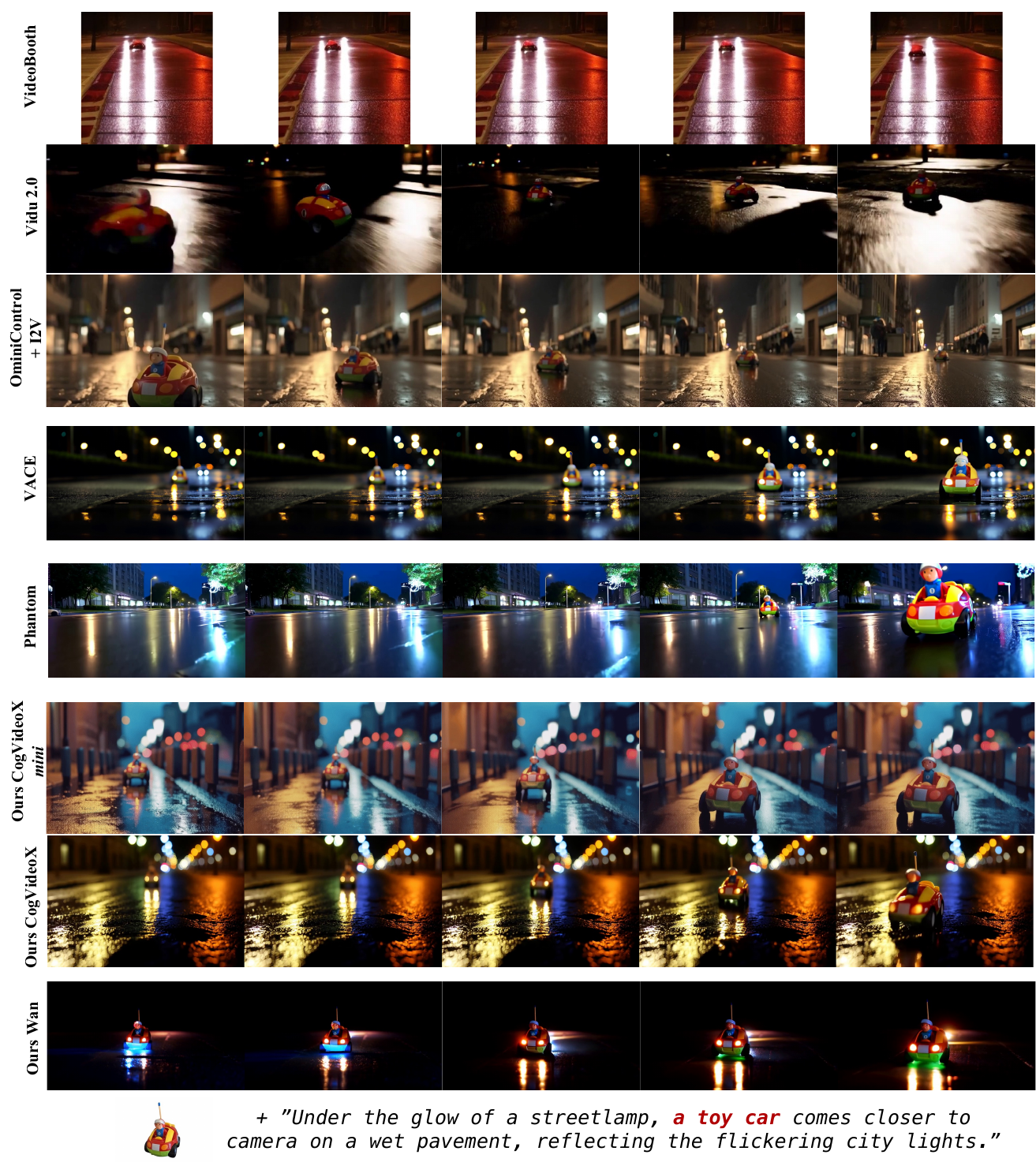}
    \caption{\textbf{Additional qualitative comparison with tuning-free baselines including VACE Wan-1.3B~\cite{vace}, Phantom Wan-1.3B~\cite{phantom}.} Note that \emph{mini} denotes our method trained with 4K subset of Subject 200K~\cite{ominicontrol}}
    \label{fig:sup3}
\end{figure*}

\begin{figure*}[t]
    \centering
    \captionsetup{type=figure}
    \includegraphics[width=0.98\textwidth]{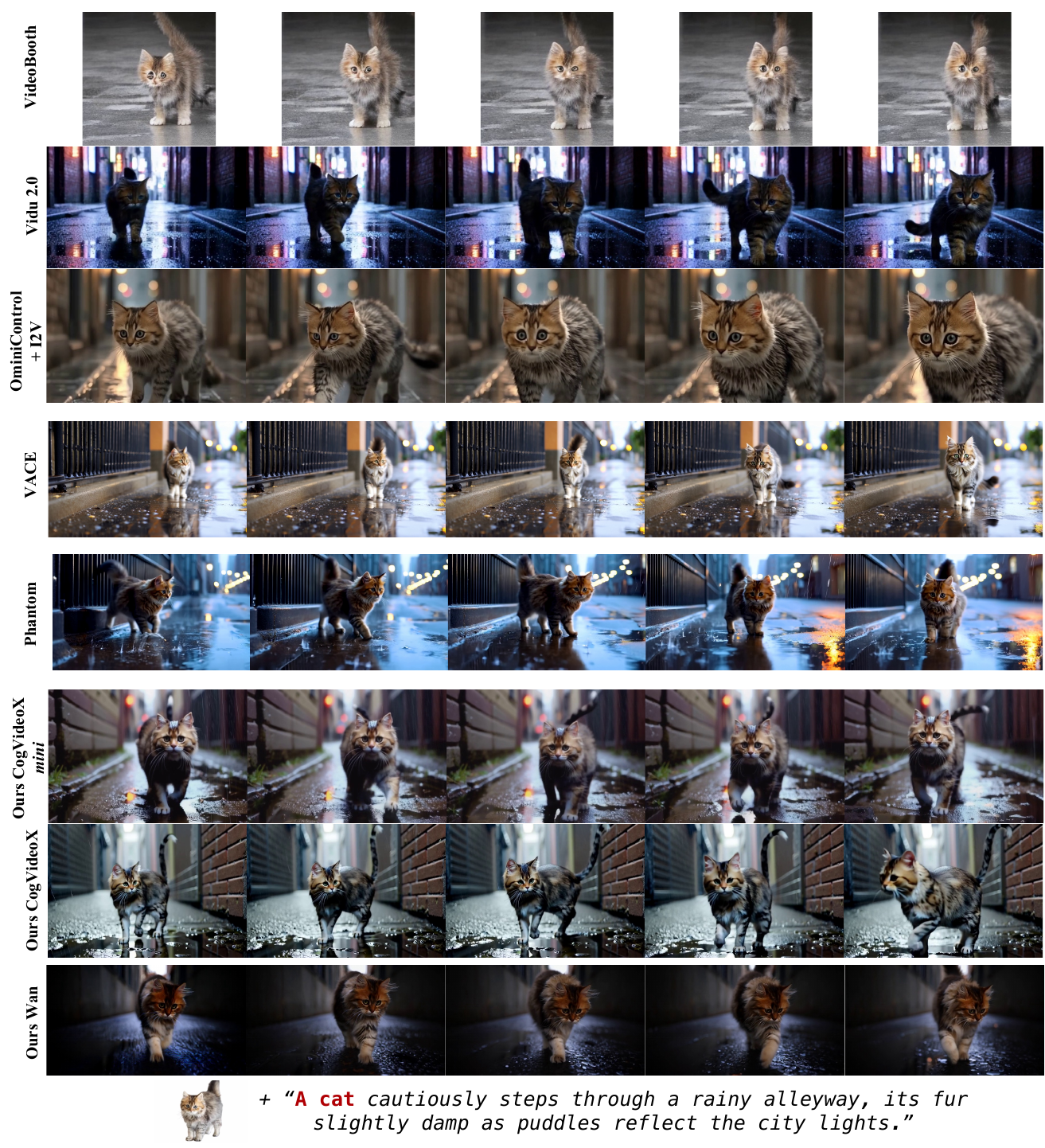}
    \caption{\textbf{Additional qualitative comparison with tuning-free baselines including VACE Wan-1.3B~\cite{vace}, Phantom Wan-1.3B~\cite{phantom}.} Note that \emph{mini} denotes our method trained with 4K subset of Subject 200K~\cite{ominicontrol}}
    \label{fig:sup4}
\end{figure*}

\clearpage

\section{Ethics Statement.}
This work is purely a research project. 
Currently, we have no plans to incorporate the developed method into a product or expand access to the public. 
Our research paper accounts for the ethical concerns associated with video generation research. 
To mitigate issues associated with training data, we have implemented a rigorous filtering process to remove inappropriate content, such as explicit imagery and offensive language, from our training data, thereby minimizing the likelihood of generating inappropriate content.

\end{document}